\documentclass{article}
\usepackage[accepted]{icml2026}
\usepackage{amsmath, amssymb, amsthm}
\usepackage{booktabs}          
\usepackage{graphicx}
\usepackage{hyperref}
\usepackage{microtype}
\usepackage{xcolor}
\usepackage{multicol}
\usepackage{algorithm}
\usepackage{algorithmic}
\usepackage{multirow}
\usepackage{enumitem}
\usepackage{float}             
\usepackage{dblfloatfix}       
\usepackage[section]{placeins} 

\newtheorem{theorem}{Theorem}
\newtheorem{proposition}{Proposition}

\usepackage{orcidlink}

\usepackage{fontawesome5}

\newcommand{\ua}{u_{\mathrm{a}}}       
\newcommand{\ue}{u_{\mathrm{e}}}       
\newcommand{\evikalmpg}{\textsc{PG-EviKal}}
\newcommand{\evikalm}{\textsc{GP-EviKal}}
\newcommand{\evikal}{\textsc{EviKal}}
\newcommand{\evid}{\textsc{Evidential}}
\newcommand{\mcd}{\textsc{MC-Dropout}}

\makeatletter
\icml@noticeprintedtrue
\makeatother

\begin{document}
\fancyhead{} 
\renewcommand{\headrulewidth}{0pt}

\twocolumn[
\icmltitle{Adapting Evidential Neural Networks to Test-Time Neighbor Fusion\\Improves Molecular
Property Prediction}

\renewcommand{\thefootnote}{*}
\begin{icmlauthorlist}
\icmlauthor{Cameron Gruich$^{\dagger}$\footnotemark}{}
\icmlauthor{Weichi Yao$^{\ddagger}{^,} ^{\S}$}{}
\icmlauthor{Yixin Wang$^{\ddagger}$}{}
\icmlauthor{Bryan Goldsmith$^{\dagger}$}{}
\end{icmlauthorlist}

\begin{center}
{\small
$\dagger$ Department of Chemical Engineering, University of Michigan, Ann
Arbor, Michigan 48109-2136, USA \\[3pt]
$\S$ Michigan Institute for Data \& AI in Society, University of
Michigan, Ann Arbor, Michigan 48109-1042, USA \\[3pt]
$\ddagger$ Department of Statistics, University of Michigan, Ann Arbor,
Michigan 48109-1107 USA
\par}
\end{center}

\icmlkeywords{uncertainty quantification, molecular property prediction,
        Kalman filter, evidential deep learning, Gaussian process}

\vskip 0.3in
]

\renewcommand{\thefootnote}{*}
\footnotetext{Corresponding author. E-mail: cameron.gruich@gmail.com}

\begin{abstract}
A trained molecular property model can be refined at test time by correcting each prediction with the measured labels of the most similar training molecules, a retraining-free procedure we call neighbor fusion; evidential neural networks make it principled by using their aleatoric and epistemic uncertainty to parameterize a Bayesian update. Our main contribution, \evikalmpg{}, learns a property-distance metric to re-rank structurally similar neighbors by their property relevance before fusion, building on \evikal{} (scalar Kalman filter) and \evikalm{} (Gaussian process variant handling correlated neighbors). Evaluated on 16 molecular datasets, \evikalmpg{} reduces RMSE relative to the evidential model baseline on 14 of them, with a median reduction of 19.4\%, and improves calibration; in sequential-assay scenarios it further incorporates newly measured molecules, refining predictions as they arrive without retraining. This work demonstrates that evidential uncertainty decomposition is not merely a calibration objective but an actionable inference resource that enables test-time refinement of molecular property predictions.
\end{abstract}



\section{Introduction}
\noindent Deep learning has emerged as a powerful tool for molecular property prediction \citep{wu2018moleculenet, yang2019analyzing}, enabling rapid in-silico screening of chemical libraries that would be prohibitively expensive to evaluate experimentally or from first principles modeling. However, point predictions alone are insufficient for many molecular discovery workflows. Knowing how uncertain a prediction is can be as valuable as the prediction itself \mbox{\citep{hirschfeld2020uncertainty}}, directing resources toward the most informative experiments and flagging unreliable estimates before they misdirect costly assays. This has motivated a line of work on Bayesian and uncertainty-aware molecular models that aim to augment predictions with calibrated uncertainty intervals \citep{scalia2020evaluating, soleimany2021evidential}.

Uncertainty quantification (UQ) methods typically provide prediction-centered intervals that encode the model's confidence in its own prediction \citep{kuleshov2018accurate, abdar2021review}. A natural question is whether this confidence signal can be used not just to rank predictions but also to actively improve them. If a model is uncertain about a query molecule of interest, structurally similar training molecules with known property labels could, in principle, correct or sharpen the query prediction. This line of thinking motivates test-time neighbor fusion: using the training set as a source of a corrective signal at inference time, without any model updates (Figure~\ref{fig:overview}).

\begin{figure*}[!ht]
  \centering
  \includegraphics[width=\linewidth]{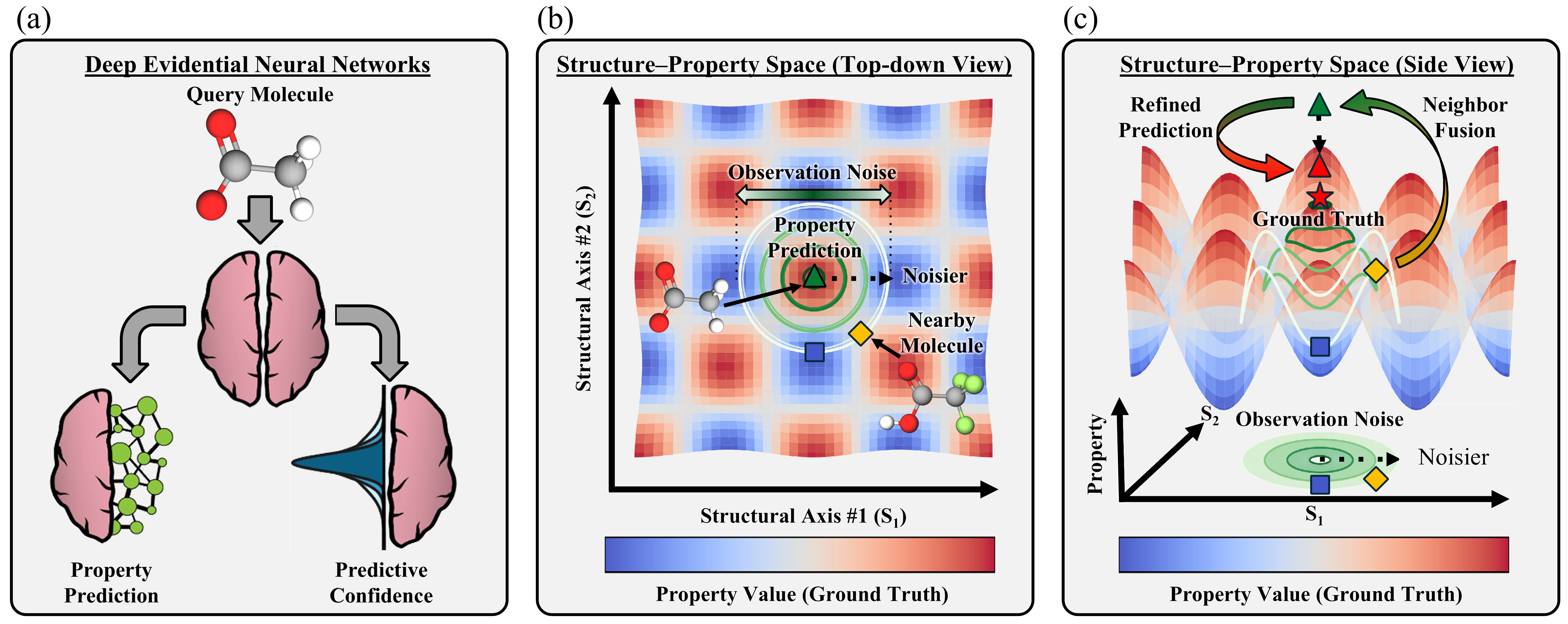}
  \caption{
    \textbf{Overview of the neighbor fusion scheme for improved molecular property prediction.} \textbf{(a)} For each query molecule, evidential neural networks provide on-demand both a molecular property prediction fitted to a choice of neural network backbone (left, green) and a Gaussian measure of the prediction's uncertainty (right, blue). \textbf{(b)} The query molecule (green triangle) is placed in Structure--Property space (red/blue, idealized). Structurally nearby molecules (yellow diamond) serve as corrective signals; more distant neighbors from the query are treated as noisier observations. \textbf{(c)} Neighbor labels are fused with the evidential property prediction, refining the prediction (red triangle) towards the ground truth (red star). Note that two neighbors may have identical measured dissimilarity to a query, but one may have a larger property distance (e.g., blue square vs. yellow diamond), showing that property relevance matters as well.
  }
  \label{fig:overview}
\end{figure*}

The natural alternative is active learning, which directs new measurements toward molecules where predictive uncertainty is high and then retrains the model on the expanded dataset \citep{reker2015active, graff2021accelerating}. Active learning is principled and well-studied, but requires retraining after each acquisition cycle \citep{settles2009active}. Such retraining is impractical when training takes hours, assigning new labels is prohibitively expensive, or the model is already in active deployment. Test-time neighbor fusion sidesteps this entirely. The training set is queried at inference, and neighbor labels are fused with the model's prediction to refine the estimate without gradient updates.

Effective neighbor fusion ideally utilizes uncertainty estimates of molecular property predictions that are both well-calibrated and available on demand at test time. By well-calibrated, we mean that the prediction-centered uncertainty interval accurately encloses where the ground truth could be relative to the property prediction \citep{gruich2023clarifying}. Monte Carlo (MC) Dropout \citep{gal2016dropout}, a commonly adopted UQ method for graph neural networks, provides only a narrow interpretation of predictive uncertainty from stochastic forward-pass variance and has, at times, exhibited poor calibration on molecular property benchmarks \citep{scalia2020evaluating}. Ensemble methods share a common, arguably more important limitation: both require sampling multiple forward passes of a neural network per molecule, and neither reliably separates the sources of uncertainty in a way that is useful for neighbor fusion \citep{kendall2017uncertainties, hullermeier2021aleatoric}.

Deep evidential regression \citep{amini2020deep} addresses both shortcomings. By placing a structured prior over the predictive distribution, it produces uncertainty estimates from a single forward pass without any sampling (i.e., on-demand uncertainty estimates, Figure~\ref{fig:overview}a). The loss function includes a regularization hyperparameter that can be tuned to improve the calibration of uncertainty estimates. Critically, the method decomposes uncertainty into aleatoric and epistemic components. Aleatoric uncertainty $\ua$ captures the irreducible noise in the structure--property relationship, such as label noise, conformational flexibility, experimental measurement noise, and any variability that would otherwise persist even with unlimited observations of the same structural type. Epistemic uncertainty $\ue$ reflects the knowledge gap of the model given the training data.

Epistemic uncertainty, by definition, decreases as more relevant observations are incorporated \citep{kendall2017uncertainties}. This suggests a natural use: initialize a prior or initial predictive uncertainty with the evidential model's epistemic uncertainty for a query molecule, then fuse structurally similar training neighbors as observations to reduce it (Figure~\ref{fig:overview}b). Neighbor selection is therefore critical. While Tanimoto similarity---measured by fingerprint overlap---is a standard approach \citep{rogers2010ecfp, bajusz2015tanimoto}, it captures only molecular overlap without guaranteeing that similar-looking neighbors will be informative for the target property (Figure~\ref{fig:overview}c).

This key distinction hinges on a property of the molecular structure--property landscape called QSAR smoothness (``Quantitative Structure--Activity Relationship smoothness''). A smooth landscape is one where structurally similar molecules tend to have similar property values---the absence of abrupt shifts or ``activity cliffs'' where small structural changes produce large property changes \citep{maggiora2006outliers, stumpfe2012exploring}. On smooth landscapes, Tanimoto similarity is a reliable proxy for property similarity. On rough landscapes with activity cliffs, structurally identical neighbors may have vastly different properties, making structural similarity alone an unreliable guide for neighbor selection. Moreover, we show that Tanimoto similarity alone is insufficient even on smooth landscapes when similar neighbors are not readily available. Our main contribution, \textbf{\evikalmpg{}} (``Property-guided Evidential Kalman''), learns a property-distance metric to select neighbors most informative for property refinement.

\evikalmpg{} selects neighbors based on a learned property-distance metric (PropDist) trained once on training set pairs using a fingerprint-based multi-layer perceptron (MLP). The MLP learns to predict property distances between query-neighbor pairs using both shared and unshared features of their fingerprints. At inference, candidates are first retrieved by Tanimoto similarity, then re-ranked by property relevance through PropDist before a batch Gaussian process posterior is applied. Using 50 neighbors ($K=50$), \evikalmpg{} reduces the evidential baseline RMSE by a median of 19.4\% across 14 of 16 benchmark datasets (Table~\ref{tab:main}). Beyond accuracy, property-guided neighbor selection improves calibration by prioritizing neighbors whose labels align with the evidential noise model, typically yielding tighter, better-calibrated posterior uncertainty intervals (Figure~\ref{fig:reliability}, Table~\ref{tab:main}).

To contextualize this contribution, we develop two simpler predecessor methods that establish the foundational principles. \evikal{} (``Evidential Kalman'') is a scalar Kalman filter that formalizes the connection between evidential uncertainty decomposition and test-time neighbor fusion on QSAR smooth datasets. \evikalm{} (``Gaussian Process Evidential Kalman'') generalizes \evikal{} to handle correlated neighbors via a Tanimoto covariance kernel, achieving a median 5.1\% RMSE reduction compared to the evidential baseline on smooth datasets. \evikalmpg{} uses the \evikalm{} Gaussian process but with property-guided selection that amplifies gains substantially over \evikalm{} on the same datasets: CDK2 ($-7.9\% \to -42.4\%$), BACE ($-6.9\% \to -38.9\%$), LD50 ($-5.1\% \to -38.2\%$) (Table~\ref{tab:main}). The progression from \evikal{} to \evikalm{} to \evikalmpg{} reveals that neighbor selection quality---not the fusion algorithm---is the bottleneck for improving test-time refinement of molecular property predictions.

Our benchmark compares and evaluates all three neighbor fusion methods (\evikal{}, \evikalm{}, \evikalmpg{}) across 16 molecular datasets spanning solubility, lipophilicity, kinase inhibition, receptor binding (serotonin 5-HT2A, dopamine D2), acute toxicity, cardiac safety, photovoltaics, aqueous pKa, cytotoxicity, and quantum-chemistry targets. Results are compared against \mcd{}, a plain GP with a Tanimoto kernel, and 5-model deep ensemble baselines. The substantially lower accuracy of the plain GP baseline demonstrates that the evidential neural network's learned representation is critical for achieving strong predictive performance. Calibration is achievable without it, but accuracy is not.

Together, these methods demonstrate that a trained evidential model is not the end of the inference pipeline for refining molecular property predictions; rather, it is a starting point. The training set, which is discarded after training in standard workflows, becomes a source of a corrective signal that can be rapidly queried on demand for molecules at test time. This requires no iterative retraining, no additional labels, and no evidential architecture changes. For practitioners working with limited data, expensive assays, or deployed models that cannot be retrained, test-time neighbor fusion with deep evidential models offers a practical mechanism for improving molecular property predictions and uncertainty calibration entirely at test time.

\section{Background and Related Work}
\label{sec:background}

\paragraph{Evidential regression.}
Placing a Normal-Inverse-Gamma prior on the likelihood parameters $(\mu, \sigma^2)$ of a Gaussian observation model yields a closed-form predictive distribution from a single forward pass (\citet{amini2020deep}). A four-output neural network head predicts the Normal-Inverse-Gamma parameters $(\gamma, \nu, \alpha, \beta)$, from which the predictive mean $\gamma$, aleatoric uncertainty $\ua = \beta/(\alpha - 1)$, and epistemic uncertainty $\ue = \beta/(\nu(\alpha-1))$ are obtained directly. The loss penalizes the model for being confidently wrong. High uncertainty estimates for well-predicted samples are discouraged, while overconfident predictions for poorly predicted samples are penalized more heavily \citep{amini2020deep}. Critically, $\ua$ and $\ue$ are available from a single forward pass---no sampling or ensembling is required at inference time.

\paragraph{Gaussian processes with molecular kernels.}
A Gaussian process (GP) defines a prior over functions and produces a closed-form posterior given training observations \citep{rasmussen2006gaussian}. For molecules, the natural kernel is the Tanimoto similarity over Extended Connectivity Fingerprints (ECFP4; \citealt{rogers2010ecfp})---binary bit vectors that encode the presence of circular substructures within a search radius of each heavy atom. The Tanimoto similarity $\mathrm{sim}(x, x') = |x \cap x'| / |x \cup x'|$ between two fingerprints $x$ and $x'$ is a valid kernel on binary vectors \citep{ralaivola2005graph} and often correlates with property similarity in organic molecules \citep{rogers2010ecfp}. \evikalmpg{} uses the Tanimoto kernel as the GP covariance function, so structurally similar training molecules contribute more to the posterior while correlations between neighbors are accounted for.

\paragraph{Uncertainty calibration.}
A calibrated uncertainty interval refers to a prediction-centered interval that is sized to accurately represent where the ground truth could exist \citep{gruich2023clarifying}. More specifically, the uncertainty interval is calibrated if its empirical coverage matches its nominal level. A 90\% interval should ideally contain the ground truth 90\% of the time (PICP@90\%, ``prediction interval coverage probability at 90\%''). A well-calibrated method achieves PICP@90\% $\approx 0.90$ over the test set; values below this indicate overconfident intervals that capture fewer true labels than promised. Separately, reliability curves measure miscalibration as deviations from parity, Appendix~\ref{app:reliability}. The Expected Calibration Error (ECE) as we define it measures the mean absolute deviation from parity on these curves, averaged over 10 confidence levels spaced from 0.1 to 0.9. Unlike PICP@90\%, ECE does not distinguish between overconfidence/underconfidence and simply measures absolute deviation from perfect calibration on the reliability curve.

\paragraph{Related work.}
Methods for uncertainty-aware prediction under distribution shift include ensemble methods \citep{lakshminarayanan2017simple} and conformal prediction \cite{angelopoulos2023conformal}. Test-time neighbor correction has precedent in $k$NN language model augmentation \citep{khandelwal2020nearest}.  Our contribution is the explicit noise model derived from the evidential decomposition, rather than a learned or heuristic correction. Sequential Bayesian updating in chemistry appears in Bayesian optimization for molecular design \citep{graff2021accelerating}, where a surrogate model is updated as new assay results arrive; \evikalmpg{} is structurally related but operates post-hoc, without any surrogate retraining. \citet{scalia2020evaluating} benchmark uncertainty methods on molecular property prediction and find \mcd{} is systematically overconfident, consistent with our results across a broader set of datasets and metrics. \citet{soleimany2021evidential} apply evidential deep learning to molecular graphs, using the aleatoric/epistemic decomposition as an active-learning acquisition signal; \evikalmpg{} uses the same decomposition as input to Kalman noise matrices for test-time neighbor fusion from a pre-trained model. The closest predecessor is deep kernel learning (DKL; \citealt{wilson2016deep}). \evikalmpg{} differs: the model head that predicts Normal-Inverse-Gamma parameters provides explicit $P_0 = \ue$ and heteroscedastic noise $R_k$ that DKL does not model, and it is post-hoc (i.e., applicable to any pre-trained evidential model without kernel co-training).

\section{Neighbor Fusion at Test Time: Methodology Overview}
\label{sec:method}

\noindent Neighbor fusion improves the predicted property of a given molecule by finding training set neighbors that are structurally similar to the molecule. After neighbors are selected, their true labels and associated uncertainties are fused with the current molecule prediction to refine the property estimate. The training dataset serves as a corrective signal for the evidential model's post-training predictions.

We refer to the molecule whose properties are being refined as the query molecule $q$ (Figure~\ref{fig:overview}a). Structurally similar training neighbors are selected using Tanimoto similarity of ECFP4 fingerprints, which encode molecular structure as bit vectors (2048-bit, capturing neighborhoods up to 2 bonds away). Tanimoto similarity measures bit overlap of these fingerprints, which in turn measures shared structural similarity. With an established similarity metric for neighbor selection, how should a neighbor's label influence the query's predicted property, and how much should it be trusted? The answer depends on the quality of each neighbor as an observation of the query's true property, which in turn depends on how structurally close the neighbor is and how uncertain the query's prediction already is (Figure~\ref{fig:method_overview}). We formalize this using a noise model that assigns each neighbor an observation uncertainty (Eq.~\eqref{eq:Rk}, Figure~\ref{fig:method_overview}b).

If a query's neighbors are treated as noisy observations of the true query property, then these observations should, at most, be only as certain as the query prediction. Our assumed noise model for neighbor observations formalizes this by using the query's aleatoric uncertainty $u_a^q$ as the minimum noise or minimum uncertainty for a given neighbor. Beyond the aleatoric floor, uncertainty grows monotonically with structural dissimilarity to the query, forming a heteroskedastic noise model in which the most structurally similar neighbors are the most reliable observations.

\evikal{} is a scalar Kalman filter that assumes this noise model and processes chosen neighbors sequentially, treating each as an independent observation of the query's true property and inheriting the posterior from the previous step (Eq.~\eqref{eq:gain}--\eqref{eq:var}). The evidential model's epistemic uncertainty $u_e^q$ initializes the prior $P_0$, and neighbor labels refine the estimate one at a time. \evikal{} is effective on smooth QSAR landscapes but has a structural limitation in that inter-neighbor correlations are never accounted for. Correlated neighbors are counted as independent evidence and collectively over-influence the prediction, a failure mode demonstrated on FreeSolv (Figure~\ref{fig:method_comparison}).

\evikalm{} resolves this by processing all $K$ neighbors simultaneously as a batch Gaussian process (Figure~\ref{fig:method_overview}, panels b and c). The Tanimoto covariance kernel encodes inter-neighbor correlations, automatically down-weighting structurally redundant neighbors in proportion to their mutual similarity. \evikalm{} shares \evikal{}'s noise model $R_k$ but replaces sequential updates with a single linear solve over the full neighbor covariance, eliminating the independence assumption. Both variants select neighbors based on Tanimoto similarity, which reliably proxies property similarity on smooth QSAR landscapes but does not guarantee property relevance. A structurally close neighbor may carry an uninformative label if the shared fingerprint features do not predict the target property.

\evikalmpg{} addresses this point at the selection stage. A lightweight property-distance model (PropDist),
trained once on fingerprint-label pairs from the same training set, re-ranks the top-500 Tanimoto
candidates by property relevance before the GP posterior is applied, Eq.~\eqref{eq:propdist}--\eqref{eq:a2sim}. The GP posterior and noise model are unchanged; only the neighbor selection changes. By selecting neighbors whose labels are consistent with the assumed noise model, \evikalmpg{} substantially improves both accuracy and calibration and is the recommended variant.

Fig.~\ref{fig:method_overview} illustrates the full inference pipeline for the Gaussian process used by \evikalmpg{} and \evikalm{}. An evidential neural network encodes the query molecule and outputs evidential parameters decomposing uncertainty into aleatoric and epistemic components, initializing the Gaussian process prior (panel a). The $K$ selected neighbors are treated as noisy sensors with observation noise $R_k$ scaling with structural dissimilarity (panel b). Their labels and mutual correlations are assembled and fused via a single GP batch solve, yielding a refined posterior mean and variance (panel c).

A key observation from panel (c) is that the neighbor covariance matrix $\mathbf{K}_{\mathrm{mat}}$ (blue) is dense, with off-diagonal entries encoding inter-neighbor correlations via Tanimoto similarity and diagonal entries equal to the Gaussian prior variance $P_0$. The observation noise matrix $\mathrm{diag}(R_k)$ (orange) is purely diagonal, with each entry reflecting how noisy a neighbor is as an observation of the query's property. The observation noise augments the diagonal of the prior covariance, yielding $\mathbf{K}_{\mathrm{obs}}$, the matrix that the Gaussian process inverts to compute the refined posterior.

\begin{figure*}[!ht]
\centering
\includegraphics[height=0.80\textheight, keepaspectratio]{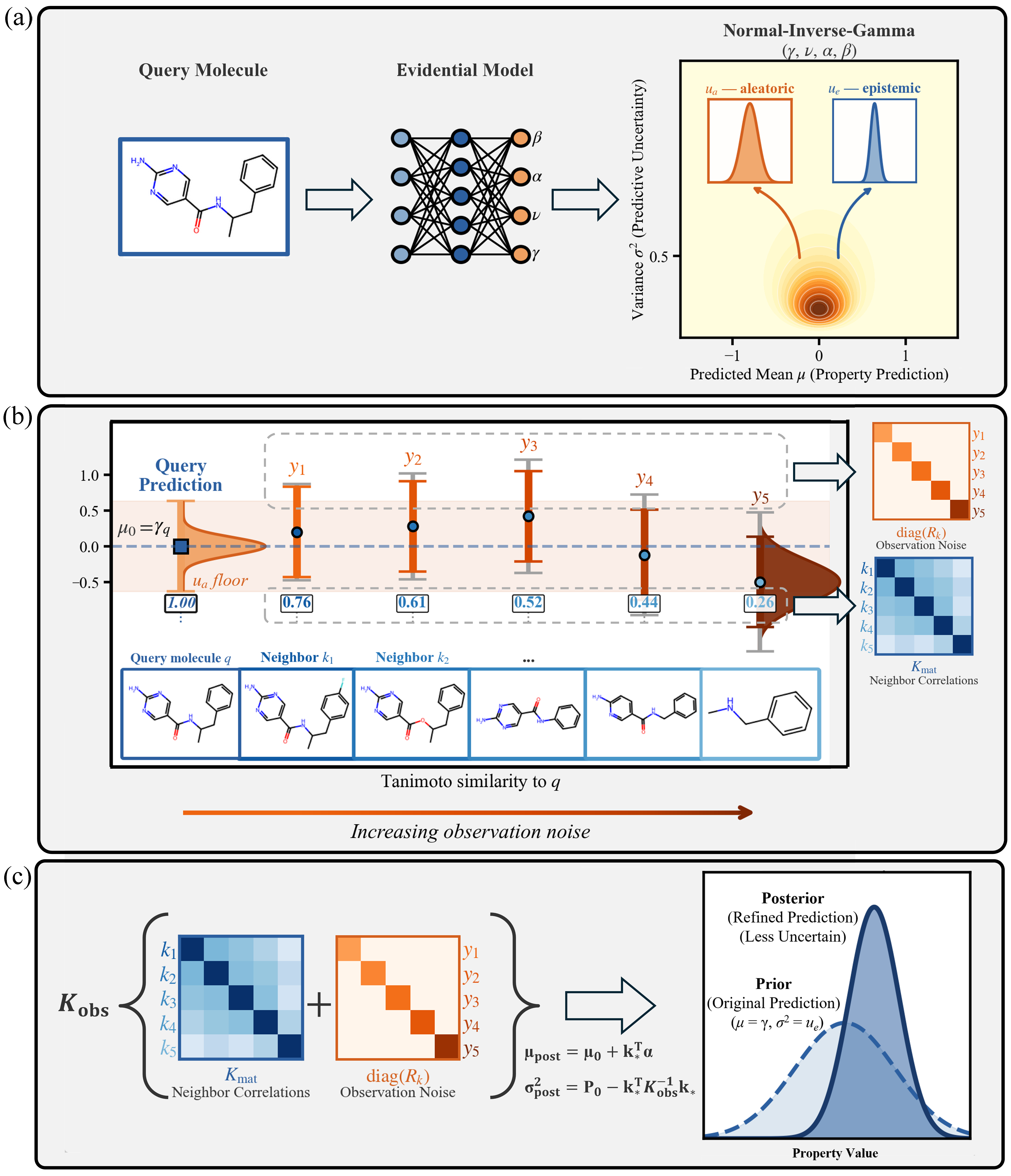}
\caption{\textbf{Gaussian process framework for neighbor fusion, shared by \evikalm{} and
\evikalmpg{}.}
  \textbf{(a)~The evidential inference pipeline for a query molecule $q$.} An evidential neural network (AttentiveFP) encodes the molecular structure and passes the embedding to a four-output evidential head that predicts the Normal-Inverse-Gamma parameters $(\gamma,\nu,\alpha,\beta)$; the predictive mean $\mu_0 = \gamma_q$ and epistemic uncertainty $P_0 = u_e^q$ derived from these parameters initialize the prior. \textbf{(b)~Neighbor selection and observation noise model.} Each of the $K$ training neighbors is treated as a noisy observation of the query's true property $y_q$; the observation noise $R_k$ grows with structural dissimilarity (Eq.~\eqref{eq:Rk}, gradient bar). Gray error bars refer to additional observation noise beyond the minimum noise, as defined by Eq.~\eqref{eq:Rk} based on Tanimoto dissimilarity. \evikalm{} selects neighbors via Tanimoto similarity (blue numbers); \evikalmpg{} uses the same selection but re-ranks neighbors by learned property distance (Section \ref{sec:a2_evikalm}). \textbf{(c)~GP batch posterior computation.} The $K{\times}K$ neighbor kernel matrix $\mathbf{K}_{\mathrm{mat}}$ and diagonal observation noise matrix $\mathrm{diag}(R_k)$ are summed to form $\mathbf{K}_{\mathrm{obs}}$, which is inverted in a single linear solve to yield the posterior mean $\mu_\mathrm{post}$ and variance $\sigma^2_\mathrm{post}$. This GP inference step is identical in both \evikalm{} and \evikalmpg{}; only the neighbor selection strategy before fusion differs.}
\label{fig:method_overview}
\end{figure*}

\subsection{Evidential Regression Head}
\label{sec:evid_head}

\noindent We use AttentiveFP \citep{xiong2020attentivefp} as the molecular encoder followed by an evidential regression model head
that outputs the parameters of a Normal-Inverse-Gamma distribution ($\gamma$, $\nu$, $\alpha$, $\beta$) per-molecule. The molecular property prediction $\hat{y}$ as well as the aleatoric ($u_a$) and epistemic ($u_e$) uncertainties are derived from these parameters:
%
\begin{equation}
  \hat{y} = \gamma, \quad
  \ua = \frac{\beta}{\alpha - 1} = \sigma_a^2, \quad
  \ue = \frac{\beta}{\nu(\alpha - 1)} = \sigma_e^2
  \label{eq:ue_ua}
\end{equation}

Both uncertainty quantities are variances. Aleatoric uncertainty $\ua$ captures the irreducible uncertainty of the molecular property prediction---uncertainty that cannot be removed simply by observing more examples of the same structural type.
Epistemic uncertainty $\ue$ measures how uncertain the model is in a molecular property prediction based on the knowledge encoded in its own model parameters---the uncertainty that shrinks when more relevant data is observed during training and additionally during test-time neighbor fusion.

We constrain the learning of $\nu$, $\alpha$, and $\beta$ to ensure finite, interpretable uncertainty estimates of model predictions:
\begin{equation}
  \nu > 0, \qquad
  \beta > 0, \qquad
  \alpha > 1
  \label{eq:ue_ua_constraints}
\end{equation}

\noindent The $\nu$ and $\beta$ constraints ensure that the epistemic uncertainty of a property prediction is interpretable (i.e., positive and finite), while the added $\alpha$ constraint ensures an interpretable aleatoric uncertainty. These constraints are enforced at the model architecture level via softplus activation functions with offsets, guaranteeing through the transform that all three parameters stay in their valid ranges during both training and test-time inference.

The evidential regression loss is a Normal-Inverse-Gamma negative log-likelihood $\mathcal{L}_{\mathrm{NLL}}$ plus an evidence regularizer $\mathcal{L}_{\text{reg}}$:
\begin{equation}
  \mathcal{L} = \mathcal{L}_{\mathrm{NLL}} + \lambda\,\mathcal{L}_{\text{reg}}, \qquad
  \mathcal{L}_{\text{reg}} = |\hat{y} - y|\,(2\nu + \alpha)
  \label{eq:evi_loss}
\end{equation}
where $\lambda$ is a regularization hyperparameter that prevents the model from accumulating spurious evidence on high-error predictions. We set $\lambda = 0.01$ throughout. Sensitivity of $\lambda$ is reported in Appendix~\ref{app:hyperparams}.

\subsection{From Evidential Decomposition to Kalman Noise Model}
\label{sec:kalman_noise}

\noindent The evidential head provides initial predictions $(\gamma_q, u_a^q, u_e^q)$ for a query molecule $q$ but only sees the query's structure. At test time, we refine this estimate by treating nearby
training molecules as noisy observations of the query's true property. This naturally fits a Kalman
filter framework \citep{kalman1960new}. The query's evidential model prediction and associated epistemic uncertainty serve as the prior to refine. Each neighbor's label becomes a noisy observation of the query molecule's property, with the noise increasing with the structural dissimilarity to the query. For a set of $K$ neighbors retrieved by Tanimoto similarity, we parameterize the observation noise as:
\begin{equation}
R_k = u_a^q + C(1 - \mathrm{sim}_k)^2,
\label{eq:Rk}
\end{equation}

The base noise $u_a^q$ is the irreducible aleatoric uncertainty of the query molecule---even a structurally identical neighbor's label is only as certain as the query molecule $q$ if we treat all neighbors as observations of the query. The dissimilarity term $C(1 - \mathrm{sim}_k)^2$ penalizes structural distance; $C > 0$ is a hyperparameter tuned on the validation set and sets the maximum observation noise for a training neighbor. A completely dissimilar neighbor $k$ has Tanimoto similarity $\mathrm{sim}_k = 0$ and therefore contributes maximal observation noise $u_a^q + C$. 

Neighbor fusion aims to refine the initial query property prediction and epistemic uncertainty. Necessarily, we initialize the Kalman filter \evikal{} directly from the evidential model's output. The initial mean $\mu$ is set to the evidential model's predicted query property ($\mu_0 = \gamma_q$). The initial uncertainty $P_0$ is set to the epistemic variance of the evidential query prediction ($P_0 = u_e^q$). As each neighboring observation is incorporated through iterative Kalman updates, the epistemic uncertainty $P_k$ at neighbor step $k$ shrinks with evidence, while the aleatoric component $u_a^q$ persists unchanged by definition. Updates are sequential. At each iteration, the filter uses the refined property/uncertainty estimates inherited by the last neighbor seen ($k-1$), as well as the current neighbor observation $k$.
\begin{equation}\label{eq:gain}
K_k     = \frac{P_{k-1}}{P_{k-1} + R_k}
\end{equation}
\begin{equation}\label{eq:mean}
\mu_k   = \mu_{k-1} + K_k (y_k - \mu_{k-1})
\end{equation}
\begin{equation}\label{eq:var}
P_k     = (1 - K_k) P_{k-1}
\end{equation}

This is \evikal{}. The Kalman gain $K_k$ in Eq.~\eqref{eq:gain} determines how much
each neighbor's label $y_k$ shifts the query property estimate in Eq.~\eqref{eq:mean}. When epistemic uncertainty $P_{k-1}$ is
large, then $K_k \approx 1$, maximizing the shift toward the neighbor. When uncertainty is low, $K_k
\approx 0$, leaving the estimate nearly unchanged. This automatic reweighting ensures uncertain
predictions are flexible while confident predictions are stable. 

By constraining the evidential parameters to Eq.~\eqref{eq:ue_ua_constraints}, we ensure the filter's gain in Eq.~\eqref{eq:gain} always receives positive, finite, and therefore interpretable uncertainty estimates. Thus, the update of the query property in Eq.~\eqref{eq:mean} remains numerically stable. The query estimate updates toward the neighbor's label in proportion to the Kalman gain scaled by the observation noise model $R_k$ (Eq.~\eqref{eq:Rk}). The epistemic uncertainty
$P_k$ shrinks with each neighbor incorporated---reflecting that additional observations reduce the
refined query property's epistemic uncertainty (Eq.~\eqref{eq:var}), thereby increasing confidence in the prediction. While a reduction is guaranteed at each step, its magnitude ranges from negligible (high observation noise $R_k$ or QSAR rough
relationships) to substantial (low noise and QSAR smooth relationships).

An adaptive outlier filter filters out neighbors whose property distance to the query is a statistical outlier. This filter prevents selecting mislabeled or atypical molecules that are so far away in property distance that they are not well-modeled as observations of the query, despite a high Tanimoto similarity. For \evikal{}, each neighbor $k$ is accepted as a neighbor for query property refinement if:
\begin{equation}\label{eq:gate}
|y_k - \mu_{k-1}| < \sigma_{\mathrm{gate}} \cdot \sqrt{P_{k-1} + R_k}
\end{equation}

This condition ensures that the property distance between the neighbor label $y_k$ and the current refined query prediction $\mu_{k-1}$ does not exceed the uncertainty of the current refined prediction ($P_{k-1}$) plus the added observation noise of the new neighbor ($R_k$), expressed as property distance ($\sqrt{P_{k-1} + R_k}$). When epistemic uncertainty $P_{k-1}$ is high, the tolerance widens and more neighbors contribute; when uncertainty is low, the tolerance tightens and outliers are rejected. This adaptive filtering prevents high-confidence estimates from being derailed by inconsistent observations. The gate threshold $\sigma_{\mathrm{gate}}$ is tuned on the validation set.

\subsection{\evikalm{}: Batch Posterior with Tanimoto Covariance}
\label{sec:gp_evikalm}

\noindent Sequential Kalman updates~\eqref{eq:gain}--\eqref{eq:var} classically treat incoming observations (i.e., neighbors) as independent.
Correlated neighbors break this assumption and can amplify errors (Section~\ref{sec:results_main}). The GP posterior handles correlations exactly under the specified covariance model (Fig.~\ref{fig:method_overview}c).
Let $\mathbf{k}_*$ be the vector of covariances between the query and each neighbor
under the Tanimoto kernel scaled by the prior variance:
$(\mathbf{k}_*)_k = P_0 \cdot \mathrm{sim}(q, x_k)$.
Let $\mathbf{K}_{\mathrm{obs}}$ be the $K\times K$ neighbor-neighbor covariance matrix
plus observation noise:
$({\bf K}_{\mathrm{obs}})_{ij} = P_0 \cdot \mathrm{sim}(x_i, x_j) + R_i\,\mathbf{1}[i=j]$.
The \evikalm{} posterior is:
\begin{align}
  \mu_{\mathrm{post}} &= \mu_0 + \mathbf{k}_*^\top \mathbf{K}_{\mathrm{obs}}^{-1}
                         (\mathbf{y} - \mu_0 \mathbf{1}),
                         \label{eq:gp_mean} \\
  \sigma^2_{\mathrm{post}} &= P_0 - \mathbf{k}_*^\top \mathbf{K}_{\mathrm{obs}}^{-1} \mathbf{k}_*.
                         \label{eq:gp_var}
\end{align}

This is the standard GP conditional distribution with prior mean $\mu_0 = \gamma_q$,
prior variance $P_0 = \ue^q$, Tanimoto kernel, and heteroscedastic observation noise $R_k$.
When $K=1$ and $\mathrm{sim}(q,x_1)=1$, Eq.~\eqref{eq:gp_mean}--\eqref{eq:gp_var} reduce to the scalar Kalman updates of \evikal{} (Eq.~\eqref{eq:gain}--\eqref{eq:var}, Appendix~\ref{app:k1_equiv}). This demonstrates that \evikalm{} is a generalization of \evikal{} that accounts for correlated neighbor effects. We also apply the adaptive outlier filter for \evikalm{}; the filter is applied before the linear batch solve, using the initial prior $P_0$ (i.e., the evidential model's epistemic uncertainty for a query property):
\begin{equation}\label{eq:gate_gp}
|y_k - \mu_0| < \sigma_{\mathrm{gate}} \cdot \sqrt{P_0 + R_k}
\end{equation}
\begin{algorithm}[t]
\caption{\evikalm{} inference for a query molecule $q$}
\label{alg:gpevikalm}
\begin{algorithmic}[1]
\REQUIRE Query SMILES $q$; trained evidential model $f_\theta$; training set $\{(x_i, y_i)\}$;
         hyperparameters $K, C, \sigma_{\mathrm{gate}}$
\STATE $(\gamma_q, \nu_q, \alpha_q, \beta_q) \leftarrow f_\theta(q)$
\STATE Compute $\ua^q$, $\ue^q$ via~\eqref{eq:ue_ua}; set $\mu_0 = \gamma_q$, $P_0 = \ue^q$
\STATE Compute ECFP4 fingerprint $\mathbf{fp}_q$
\STATE Retrieve $K$ nearest neighbors $\{(x_k, y_k, \mathrm{sim}_k)\}$ by Tanimoto similarity to $\mathbf{fp}_q$
\STATE Compute $R_k$ for each neighbor via~\eqref{eq:Rk}
\STATE Apply outlier filter: discard neighbor $k$ if $|y_k - \mu_0| \geq \sigma_{\mathrm{gate}}\sqrt{P_0 + R_k}$
\STATE Compute $\mathbf{k}_*$, $\mathbf{K}_{\mathrm{obs}}$ from retained neighbors
\STATE Clip: $R_k \leftarrow \max(R_k,\, \varepsilon_{\min})$; add jitter $\varepsilon I$ to $\mathbf{K}_{\mathrm{obs}}$ \hfill\COMMENT{ensures numerical stability; $\varepsilon=10^{-6}$, $\varepsilon_{\min}=10^{-4}$}
\STATE Solve $\mathbf{K}_{\mathrm{obs}}\,\boldsymbol{\alpha} = (\mathbf{y} - \mu_0\mathbf{1})$
\STATE $\mu_{\mathrm{post}} \leftarrow \mu_0 + \mathbf{k}_*^\top \boldsymbol{\alpha}$;
       $\;\sigma^2_{\mathrm{post}} \leftarrow P_0 - \mathbf{k}_*^\top \mathbf{K}_{\mathrm{obs}}^{-1} \mathbf{k}_*$
\ENSURE Prediction $\mu_{\mathrm{post}}$; uncertainty $\sigma^2_{\mathrm{post}}$
\end{algorithmic}
\end{algorithm}

\subsection{Why Neighbor Fusion Helps}
\label{sec:theory}
\noindent The batch update of Eq.~\eqref{eq:gp_mean}--\eqref{eq:gp_var} is not an ad hoc correction but exact Bayesian conditioning. The evidential property prediction and epistemic variance act as the prior, and the neighbor labels enter as noisy measurements of the query property. When this local model holds, \evikalm{} is optimal under squared loss, and the epistemic uncertainty can only shrink, strictly so whenever the neighbors are informative (Appendix~\ref{app:theorems}, Theorem~\ref{theorem:risk_reduction}).

In practice, the local observation model is only approximate, and fusion helps when the gap between the neighbor labels and the query's true property stays small relative to the prior error still available to correct (Appendix~\ref{app:theorems}, Proposition~\ref{prop:misspecification}). This explains the method's behavior across the benchmark; label gaps widen under activity cliffs, sparse structural coverage, or property-irrelevant neighbors, Section~\ref{sec:results}. The smoothness diagnostic of the next section measures these label gaps at the dataset level, and property-supervised selection (\evikalmpg{}, Section~\ref{sec:a2_evikalm}) narrows them by choosing neighbors whose labels track the query's property.

\subsection{QSAR Smoothness: A Diagnostic for When \evikalm{} Helps}
\label{sec:smoothness}

\noindent Both \evikal{} and \evikalm{} require that structural similarity predicts property similarity. On smooth QSAR landscapes, structurally close neighbors report consistent label values. On rough QSAR landscapes (e.g., activity cliffs, quantum-electronic properties), a structurally close neighbor may report a qualitatively different property value, and no amount of neighbor covariance modeling compensates for an observation model that does not match the landscape (Section~\ref{sec:results_main}).

We operationalize landscape smoothness with the QSAR smoothness ratio, a dataset-wide metric:
\begin{equation}
  s = \frac{\mathrm{median}_{q \in \mathcal{D}_\mathrm{test}} |y_q - y_{q,1}|}{\sigma_{\mathrm{train}}},
  \label{eq:smoothness}
\end{equation}
where $y_{q,1}$ is the label of the Tanimoto-nearest training neighbor of $q$
and $\sigma_{\mathrm{train}}$ is the training label standard deviation. This quantity refers to the typical property distance between the closest neighbor of each test query, normalized by the amount the molecular property varies in the training dataset. A low value of $s$ implies that the most similar training neighbor for any given query is typically close in property distance. A high value of $s$ implies a typically large property distance. QSAR smoothness $s$ is a measure of the dispersion of property distances between queries and their closest neighbors. The median---as opposed to the mean---is used to prevent outlier skew. 

Both \evikal{} and \evikalm{} rely on neighbors to refine the query estimate; therefore, trust is placed on
the similarity metric (Tanimoto similarity, $\mathrm{sim}$) to identify training neighbors that are effective observations of the query. A necessary condition is that neighbors to queries must carry more corrective signal than noise; their property distances to the query must be smaller than the dataset's overall property variation, giving the lower bound $s < 1$. When $s \geq 1$, neighbors differ too widely in large property distances to provide reliable signal about the query. Because the quantity $s$ uses the closest neighbor defined by the similarity metric, this theoretical criterion is tied to the similarity metric and independent of the evidential model choice.

Compared to the evidential model baseline in our benchmark, prediction refinement is usually reliable when $s < 0.65$. This empirical threshold is determined by selecting the value of $s$ that typically satisfies the effective signal-to-noise ratio of the Kalman update on the evidential model property predictions. Eq.~\eqref{eq:var} rearranges to give how much a refined query property changes at each iteration, $\Delta\mu =K_k (y_k - \mu_{k-1})$. This prediction correction helps only when the query's neighbors provide more corrective signal than observation noise:
\begin{equation}
\mathrm{SNR}_{\mathrm{eff}} = K_{\mathrm{gain}} \cdot \frac{1 - s^2}{s^2} \geq 1
\label{eq:SNRapprox}
\end{equation}

The ratio $\frac{1 - s^2}{s^2}$ is a signal-to-noise ratio---a ratio of variances. As $s$ increases, the typical property distance between the most structurally similar neighbor and a given query grows larger by definition; necessarily, increasing $s^2$ represents increasing noise with respect to property distance to the query. Likewise, the complement of $s^2$ (i.e., $1 - s^2$) oppositely represents signal. When $\mathrm{SNR}_{\mathrm{eff}} > 1$, the Kalman-weighted signal in the numerator exceeds the noise, improving predictions. When $\mathrm{SNR}_{\mathrm{eff}} < 1$, noise dominates. Theory and validation of criterion given in Appendix~\ref{app:snr_eff_validation}. 

The theoretical criterion $s < 1$ and empirical criterion $\mathrm{SNR}_{\mathrm{eff}}$ (usually satisfied by $s < 0.65$ for our model choice/hyperparameters) set the expectation on whether \evikal{} and \evikalm{} will typically improve molecular properties on a given dataset. Since $K_{\mathrm{gain}}$ dictates the $\mathrm{SNR}_{\mathrm{eff}}$ criterion and $K_{\mathrm{gain}}$ depends on the model configuration, the empirical value of $s$ that satisfies $\mathrm{SNR}_{\mathrm{eff}}$ should be reassessed for other architectures (Appendix~\ref{app:snr_eff_validation}). Beyond $s = 0.65$ in our benchmark, neighbors typically differ too much in property space to reliably improve predictions.

\subsection{Property-Supervised Similarity (\evikalmpg{})}
\label{sec:a2_evikalm}

\noindent Two molecules with identical Tanimoto similarity to the query may differ substantially in their usefulness as query observations. When the structural features driving that similarity are uninformative about the target property, the neighbor's label deviates from the query's true value by more than the noise model $R_k$ anticipates. The Kalman update then over-trusts that neighbor, biasing the posterior refined prediction. \evikalmpg{} addresses this at the query neighbor selection stage (Fig.~\ref{fig:method_overview}b) by changing the choice of neighbors. By learning which structural differences co-vary with property differences, it preferentially suggests neighbors whose labels are consistent with the assumed noise model.

We train an MLP, PropDist, to predict the absolute property gap $|y_i - y_j|$ from fingerprint features of molecule pairs drawn from the training set:
\begin{equation}
  \mathrm{PropDist}(x_i, x_j) = \mathrm{MLP}_\phi\!\bigl([\mathbf{fp}_i \wedge \mathbf{fp}_j;\; \mathbf{fp}_i \oplus \mathbf{fp}_j]\bigr),
  \label{eq:propdist}
\end{equation}
where $\wedge$ and $\oplus$ denote bitwise AND (shared features) and XOR (differing features) of the ECFP4 fingerprints, concatenated to a 4096-dimensional input. The architecture uses three hidden layers (256, 128, 64) with LayerNorm, ReLU, and Dropout(0.2), trained on 800k pairs per dataset (MSE loss, seeds 0--3). PropDist is trained once per dataset from the same training molecules as the evidential model, using no test labels. More training details given in Appendix~\ref{app:hyperparams}.3.

The property-guided similarity combines the structural gate of Tanimoto with PropDist's property-relevance score:
\begin{equation}
  \mathrm{sim}^{\mathrm{PG}}(q, x_k) = \mathrm{Tanimoto}(q, x_k) \cdot \exp\!\bigl(-\mathrm{PropDist}(q, x_k)\bigr),
  \label{eq:a2sim}
\end{equation}
Tanimoto gates out structurally foreign molecules; PropDist re-ranks within the accessible pool by how property-predictive the shared features are. At inference, a GPU-batched Tanimoto prescreen retrieves the top-500 candidates, PropDist scores those 500, and \evikalmpg{} selects the top-$K$. PropDist identifies  which neighbors are informative for selection; Tanimoto provides the kernel structure that ensures well-behaved GP inference. When comparing \evikalmpg{} to \evikalm{}, only the neighbor selection changes. PropDist adds one GPU forward pass over 500 candidates per query, a small overhead relative to the evidential forward pass that dominates inference cost (Appendix~\ref{app:runtime}).


\section{Experimental Setup}
\label{sec:experiments}

\paragraph{Datasets.}
We evaluate across seven domains of molecular property prediction (sixteen datasets total).
\textbf{Solubility and lipophilicity} (scaffold split): ESOL (water solubility, 1128 molecules), FreeSolv (hydration free energy, 642 molecules), and Lipophilicity (log $D$, 4200 molecules) \citep{wu2018moleculenet}. \textbf{Drug safety and receptor binding} (random split): BACE (enzyme inhibition, 1513 molecules) \citep{wu2018moleculenet}, CDK2 (kinase inhibition, 2060 molecules), hERG (cardiac ion channel, 8362 molecules), 5-HT2A (serotonin receptor, 5388 molecules), and D2 (dopamine receptor, 8433 molecules) \citep{gaulton2012chembl}. \textbf{Organic photovoltaics}: HOPV (power conversion efficiency, 340 molecules) \citep{lopez2016harvard}. \textbf{Aqueous pKa} (8745 molecules from the OPERA pKa training set~\citep{mansouri2019opera}, random split). \textbf{Toxicity}: LD50 acute oral toxicity (7385 molecules \citep{zhu2009quantitative}, random split). \textbf{High-throughput screens}: Thermosol (kinetic solubility, 1763 molecules) \citep{wu2018moleculenet}, NCI-60 mean cytotoxicity \citep{reinhold2012ncialmanac} (19126 molecules). \textbf{Quantum chemistry} (random split): QM7 atomization energies \citep{blum2009970, rupp2012fast}, QM8 first singlet excitation energies (E$_1^{\text{CC2}}$; 21786 molecules \citep{ramakrishnan2015electronic}), and QM9 HOMO--LUMO gap \citep{ramakrishnan2014qm9} (130,831 molecules in full dataset; we use 120,000 in a 100k/10k/10k random split). QSAR smoothness ratios span $s = 0.136$ (QM8, QM9) to $s = 0.834$ (QM7, roughest). Dataset statistics and data curation details are in Appendix~\ref{app:datasets}.

\paragraph{Baselines.}
\textbf{\mcd{}} \citep{gal2016dropout}: 50 stochastic forward passes with dropout rate.  Mean and variance of outputs serve as the prediction and the uncertainty, respectively. \textbf{\evid{}}: AttentiveFP with Normal-Inverse-Gamma head, no test-time inference (Section~\ref{sec:evid_head}). \textbf{Sequential \evikal{}}: scalar Kalman filter with outlier filter, Eq.~\eqref{eq:gain}--\eqref{eq:var}. \textbf{\evikalm{}}: batch GP posterior, Eq.~\eqref{eq:gp_mean}--\eqref{eq:gp_var} (ours) with outlier filter. \textbf{GP-Tanimoto}: MLL-optimized zero-mean GP with $\sigma_f^2 \cdot \mathrm{Tanimoto}$ kernel plus a WhiteKernel noise term, fit to the same training set.

\paragraph{Metrics.}
RMSE and MAE measure predictive accuracy. PICP@90\% measures calibration: what fraction of test labels fall within the predicted 90\% prediction interval (target: 0.90). ECE is the expected calibration error (Section~\ref{sec:background}). NLL is the negative log-likelihood under the predictive distribution \citep{gneiting2007strictly, lakshminarayanan2017simple}. All experiments report mean $\pm$ std over 5 random seeds with fixed dataset splits.

\paragraph{Hyperparameters.}
\evikalm{} uses $K = 5$ neighbors; \evikalmpg{} uses $K = 50$ (the larger candidate pool allows PropDist to surface property-relevant neighbors beyond the Tanimoto top-5; see Appendix~\ref{app:ablation_K}). $C$ and $\sigma_{\mathrm{gate}}$ are tuned on the validation set per dataset via grid search ($C \in \{0.1, 0.5, 1.0, 2.0, 5.0, 10.0, 20.0, 50.0, 100.0, 200.0\}$; $\sigma_{\mathrm{gate}} \in \{0.0, 0.5, 1.0, 2.0, 3.0\}$; $\sigma_{\mathrm{gate}} = 0.0$ disables the innovation gate). AttentiveFP uses the published hyperparameters from \citet{xiong2020attentivefp}; training details and ablations over $K$ and $C$ are in Appendix~\ref{app:hyperparams}.


\section{Results}
\label{sec:results}

\subsection{Neighbor Selection: From Structural Similarity to Property Relevance}
\label{sec:results_main}

\begin{figure*}[!ht]
\centering
\includegraphics[width=\linewidth]{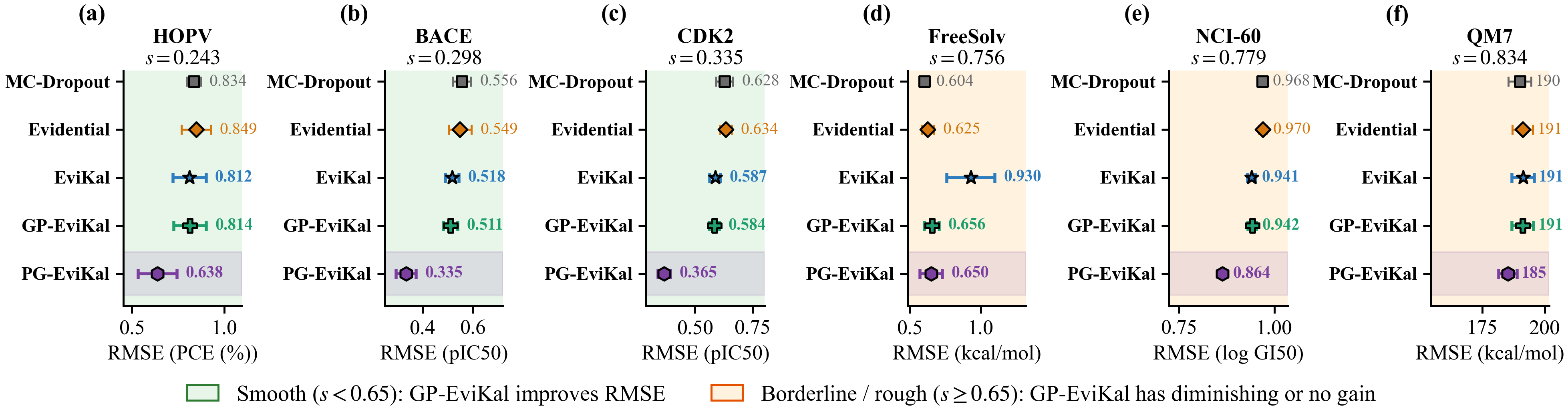}
\caption{
\textbf{Method comparison across six representative datasets} (RMSE, 5 seeds;
full comparison in Figure~\ref{fig:method_comparison_full}).
Datasets span key cases: HOPV (smooth, small-dataset case), CDK2 and BACE (smooth, largest \evikalmpg{} gains), FreeSolv (sparse structural coverage of neighbors that fails neighbor fusion),
NCI-60 (adequate structural coverage of neighbors despite similar roughness to FreeSolv), and QM7 (rough). QM7 highlights that topological molecular representations (the atom/bond graph the model encodes and the ECFP4 fingerprints used for neighbor retrieval) are too coarse for effective neighbor selection and fusion on certain quantum chemical properties.
\evikalmpg{} (purple) achieves the best RMSE across all cases except the FreeSolv dataset.
}
\label{fig:method_comparison}
\end{figure*}

\noindent Not all structurally similar neighbors are effective for refining a query's property. Neighbor selection is critical. Selection based on Tanimoto similarity assumes that structural similarity predicts property closeness between a query and a neighbor. This holds on smooth QSAR landscapes (Figure~\ref{fig:method_comparison}) but breaks down when molecules with high structural similarity have vastly different properties (i.e., activity cliffs \citep{maggiora2006outliers, vantilborg2022exposing}) or when structurally distant molecules share property-relevant features. 

\evikal{} and \evikalm{} suffer from this limitation. \evikalmpg{} addresses it through a two-stage approach: (1) use Tanimoto to quickly retrieve a large candidate pool (top 500 neighbors) and (2) re-rank candidates by a learned property-distance metric {$\mathrm{sim}^{\mathrm{PG}}$} before selecting the final $K$ neighbors. Tanimoto similarity selects the neighbor candidate pool; $\mathrm{sim}^{\mathrm{PG}}$ selects the most relevant neighbors from that pool. The payoff is substantial on smooth datasets, where the property-guided selection of \evikalmpg{} amplifies gains compared to \evikalm{}, dramatically for CDK2 ($-7.9\% \to -42.4\%$), BACE ($-6.9\% \to -38.9\%$), and HOPV ($-4.0\% \to -24.9\%$) (Figure~\ref{fig:method_comparison}a--c). On rough or sparse landscapes, property-guided selection prevents the over-selection of structurally similar but property-irrelevant neighbors.

\noindent \textbf{Smooth regimes ($s < 0.65$).}
Tanimoto-selected neighbors are reasonably reliable in this regime, so \evikal{} and \evikalm{} consistently improve RMSE. Intuitively, a low $s$ means the typical distance from a query to its most similar
training neighbor is small, implying both structural and property closeness. Such neighbors help only when they carry more corrective signal than noise---a condition quantified by the effective signal-to-noise ratio $\mathrm{SNR}_{\mathrm{eff}}$ (Eq.~\ref{eq:SNR}). The empirical threshold $s < 0.65$ marks when the signal component of this ratio exceeds the noise component across the benchmark for the given evidential model and hyperparameters (Figure~\ref{fig:smoothness_analysis}a). Beyond this threshold, observations become noisier than informative (Figure~\ref{fig:smoothness_analysis}b). We confirm that this criterion is satisfied for eleven out of twelve smooth datasets where \evikalm{} improves RMSE; theoretical derivation and empirical validation are in Appendix~\ref{app:snr_eff_validation}, and a dataset-by-dataset reading of both panels of Figure~\ref{fig:smoothness_analysis} is given in Appendix~\ref{app:smoothness}.

\evikalmpg{} amplifies this substantially: eleven of twelve smooth datasets improve, ranging from modest (ESOL: $-9.4\%$) to dramatic (CDK2, BACE, LD50: $-38$ to $-42\%$) (Table~\ref{tab:main}). The twelfth, QM9, is unchanged. Together with QM8 (Table~\ref{tab:main}), these two datasets represent an accuracy ceiling case. The evidential model choice, AttentiveFP, has a topological molecular-graph representation that is exhausted, and the evidential model has already learned essentially all the information that representation can encode. Neighbors add little regardless of the selection strategy because subtle electronic-structure effects tied to orbital interactions and charge distribution are captured neither by the graph the model encodes nor by the Tanimoto fingerprints used to retrieve neighbors; both describe only local connectivity and discrete atom/bond features, not long-range electronic interactions. No amount of neighbor fusion can extract a signal that the representation was never designed to capture.

\begin{figure*}[!ht]
\centering
\includegraphics[width=\linewidth]{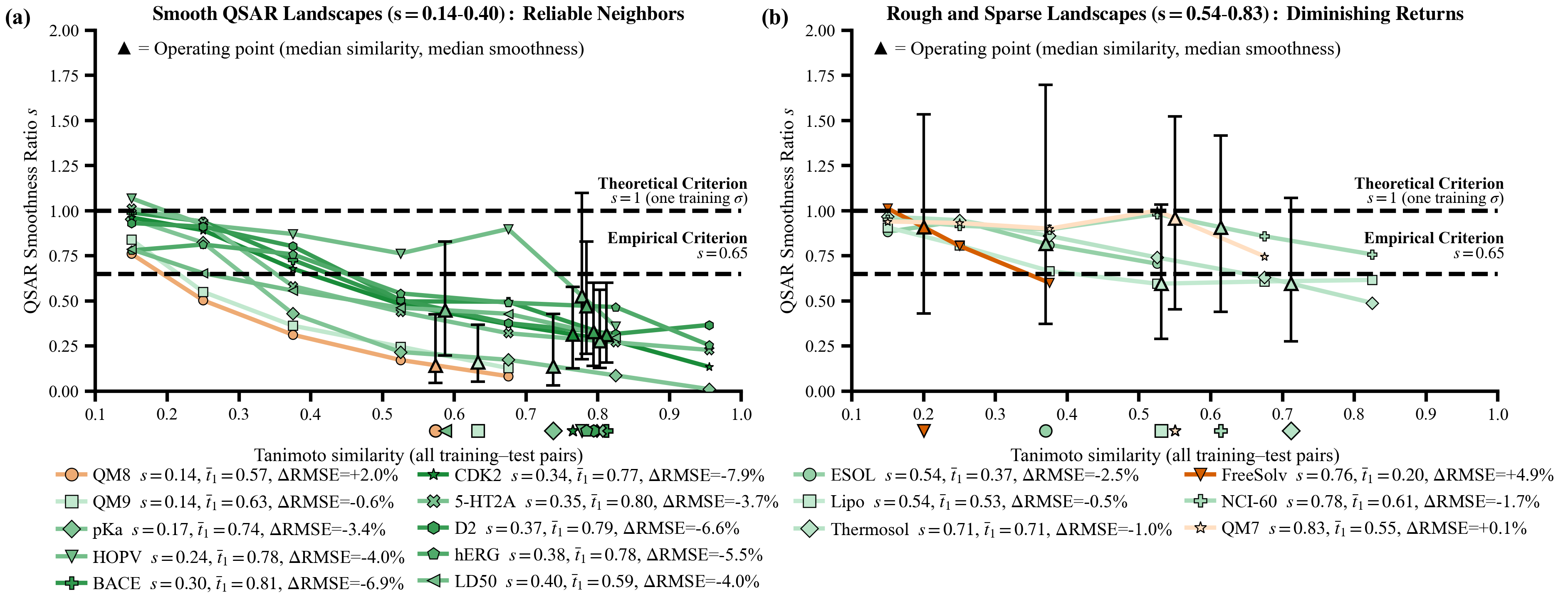}
\caption{
\textbf{When are neighbors reliable?}
Each curve plots the median QSAR smoothness across all test query molecules, binned by Tanimoto similarity. The triangles ($\blacktriangle$) represent the actual \evikalm{} operating point for each dataset: the median similarity and property distance between a test query and its closest neighbor when performing neighbor fusion. Operating point bars represent the interquartile range, or how much the property distance between query and closest neighbor varies at the typical neighbor similarity seen. Markers below x-axis line up with a given dataset's operating point. \textbf{(a) Smooth Datasets:} Operating points sit in a zone of high neighbor similarity and low property distance, where neighbors are reliable teachers for query property refinement. \textbf{(b) Sparse and Rough Datasets:} FreeSolv's operating point sits far left with high structural dissimilarity between closest neighbors and given queries; NCI-60 has similar QSAR roughness compared to FreeSolv, but the measured similarity between closest neighbors and queries is much higher. The contrast in $\Delta{RMSE}$ demonstrates the importance of structurally similar and available neighbors for fusion.
}
\label{fig:smoothness_analysis}
\end{figure*}

\noindent \textbf{Sparse and rough regimes ($s > 0.65$).}
FreeSolv ($s=0.756$) shows why Tanimoto-only selection fails. Its neighbors are structurally sparse as measured by the typical Tanimoto similarity between a query and closest neighbor (median top-1 Tanimoto $= 0.20$, Figure~\ref{fig:smoothness_analysis}b), so even property-guided re-ranking cannot recover much signal (+4.0\%, Table~\ref{tab:main}). Yet, NCI-60 at nearly identical smoothness ($s=0.779$) but denser neighbor coverage (top-1 Tanimoto $= 0.61$, Figure~\ref{fig:smoothness_analysis}b) achieves $-10.9\%$ improvement. This contrast demonstrates that structural availability, not landscape roughness alone, determines whether neighbor fusion helps. Where neighbors exist, property-guided selection makes them count; where they do not exist, no method recovers the missing information.
\begin{figure*}[!ht]
\centering
\includegraphics[width=\linewidth]{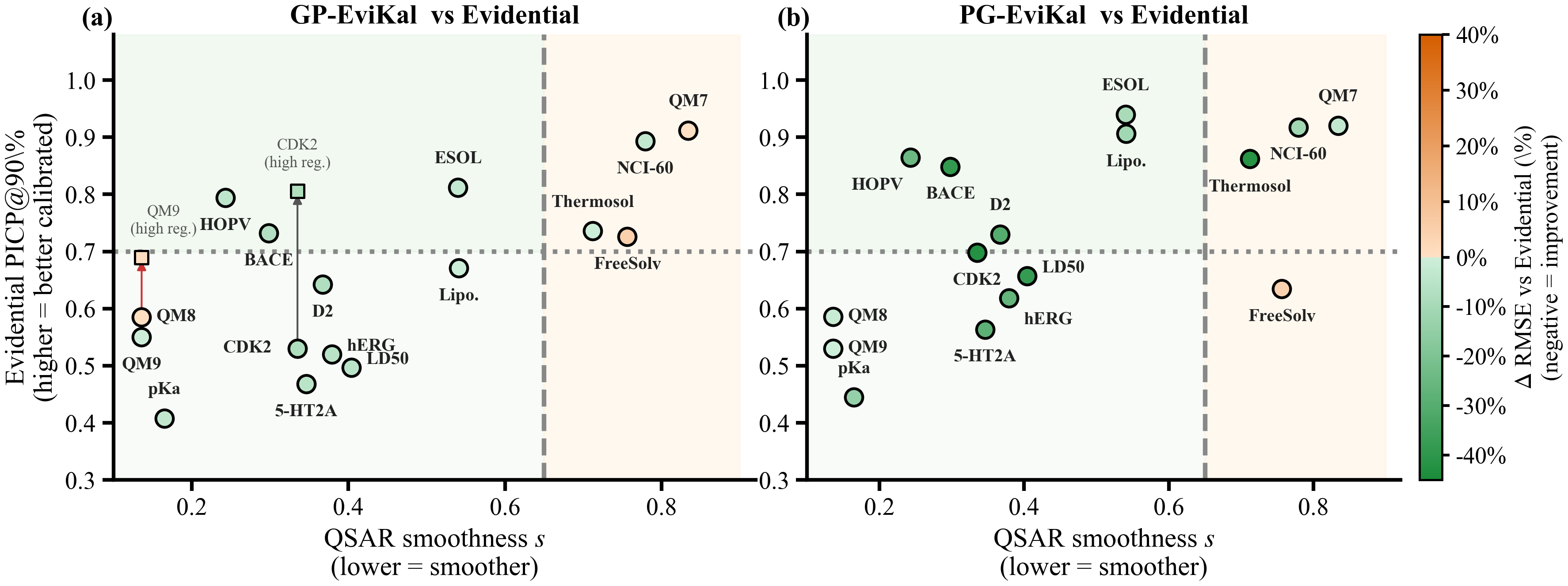}
\caption{
\textbf{Property-guided neighbor selection outperforms structural similarity alone.}
Neighbor fusion succeeds when the QSAR landscape is smooth ($s < 0.65$). \textbf{(a)} \evikalm{} uses Tanimoto structural similarity alone to select neighbors and shows consistent gains compared to the evidential baseline but fails on rough landscapes. Regularization ablations (arrows) show that improving calibration of the evidential model alone does not drive neighbor fusion accuracy. \textbf{(b)} \evikalmpg{} adds learned property distance to neighbor selection, amplifying gains substantially ($-$38 to $-$42\% on CDK2, BACE, LD50). Tanimoto captures only mutual structural similarity (\evikalm); property distance additionally captures which structurally similar molecules chosen are informative for property refinement (\evikalmpg). \evikalmpg{}'s performance over \evikalm{} shows that neighbor selection largely drives accuracy gains.
}
\label{fig:diagnostic}
\end{figure*}

\noindent \textbf{Calibration cost.}
All methods compress predictive intervals (by design: $\sigma^2_\mathrm{post} < P_0$). On well-calibrated priors (PICP@90\% $\geq 0.70$), compression is modest ($\leq 7$\%) and leaves intervals useful for decision-making. On overconfident priors, increasing regularization ($\lambda$) recovers calibration while preserving RMSE gains, Appendix~\ref{app:hyperparams}.5.

Figure~\ref{fig:diagnostic} maps RMSE improvement of \evikalm{} and \evikalmpg{} compared to the evidential model baseline across QSAR smoothness and calibration. The upper-left quadrant (smooth, well-calibrated) is the clear win zone for \evikalmpg{}, where it delivers substantial gains ($-38$ to $-42\%$). The advantage extends beyond this quadrant: across all quadrants, \evikalmpg{} improves or matches \evikalm{} by prioritizing informative neighbors over structurally similar ones. The FreeSolv outlier (rough landscape, $s=0.756$) shows the fundamental limit. When the structure--property landscape is too QSAR rough, no neighbor selection strategy recovers accuracy. There are no neighbors to find. Elsewhere, property-guided re-ranking ensures that available neighbors are ranked by property relevance rather than by structure alone.

\subsection{Calibration: How Property-Guided Selection Improves Evidential Regression}
\label{sec:results_calibration}

\noindent A refined prediction is more actionable when it is accompanied by a trustworthy estimate of its own error. Having established that property-guided neighbor selection sharpens accuracy (Section~\ref{sec:results_main}), we now examine whether the same mechanism also yields calibrated uncertainty. Across the benchmark, baseline calibration varies widely. \mcd{} is consistently overconfident, so the uncertainty intervals centered on its predictions severely undercover the ground truth across most of the benchmark, achieving PICP@90\% between 0.09 and 0.72, with most datasets below 0.54 (Table~\ref{tab:main}). Evidential regression is substantially better calibrated through its structured aleatoric/epistemic decomposition, outperforming \mcd{} on 13 of 16 datasets (Table~\ref{tab:main}). Even so, evidential regression alone remains overconfident on several datasets (CDK2, hERG, pKa), reaching only 0.41--0.53 PICP@90\%. This overconfidence can be corrected by adjusting the evidential hyperparameter $\lambda$ (Appendix~\ref{app:hyperparams}.5), but we hold $\lambda$ fixed for comparability across the benchmark.

Property-guided neighbor selection (\evikalmpg{}) improves calibration of an overconfident evidential model. By selecting neighbors that are both structurally and property-relevant, Kalman updates refine the posterior meaningfully rather than incorporating noisy signals. Where Tanimoto similarity alone pairs structurally similar but property-dissimilar neighbors, \evikalmpg{} better balances which neighbors inform the posterior. \evikalmpg{} achieves better calibration than base evidential regression on 14 of the 16 benchmark datasets (Table~\ref{tab:main}; the reliability curves are in Appendix~\ref{app:reliability}).
Together, structured uncertainty decomposition and property-guided neighbor selection improve both accuracy and calibration.

This calibration is what distinguishes \evikalmpg{} from the simpler alternative of simply averaging the labels of the neighbors. As analyzed in Appendix~\ref{app:averaging}, simple neighbor averaging can exceed \evikalmpg{} on point prediction RMSE when the evidential prior is inaccurate, but simple averaging only returns a bare point estimate. \evikalmpg{} instead provides both a refined prediction and a calibrated interval that quantifies how far the ground truth is expected to lie from that prediction. This point prediction gap, however, is not fundamental. A standard Gaussian process refinement that lets the neighbors bias the GP's prior mean rather than holding it fixed at the evidential model prediction closes the gap \citep{matheron1963principles,cressie1993statistics,rasmussen2006gaussian}, matching simple averaging on the affected datasets and surpassing it on most, while retaining the calibrated uncertainty interval (Appendix~\ref{app:averaging}).

\subsection{Why the Neural Mean Is Essential}
\label{sec:results_gp}
\noindent A plain Gaussian process with a Tanimoto kernel achieves near-perfect calibration but has insufficient accuracy relative to the evidential neural network (Figure~\ref{fig:gp_tanimoto}). Thus, a plain GP Tanimoto kernel alone is insufficient; it must be anchored to a global mean function that is well-fitted to the structure--property trend.

\evikalm{} combines neural mean accuracy with GP covariance structure, providing both precision and principled uncertainty (Deep Ensemble comparison: Appendix~\ref{app:ensemble}). However, the method may select structurally similar but property-dissimilar
neighbors---poor choices for neighbor fusion. \evikalmpg{} uses the same \evikalm{} framework and inherits its advantages but extends Tanimoto-only neighbor ranking with learned property distance. By doing so, property-guided selection better recovers both accuracy and calibration where Tanimoto-only selection fails.

\begin{figure}[t]
\centering
\includegraphics[width=0.90\linewidth]{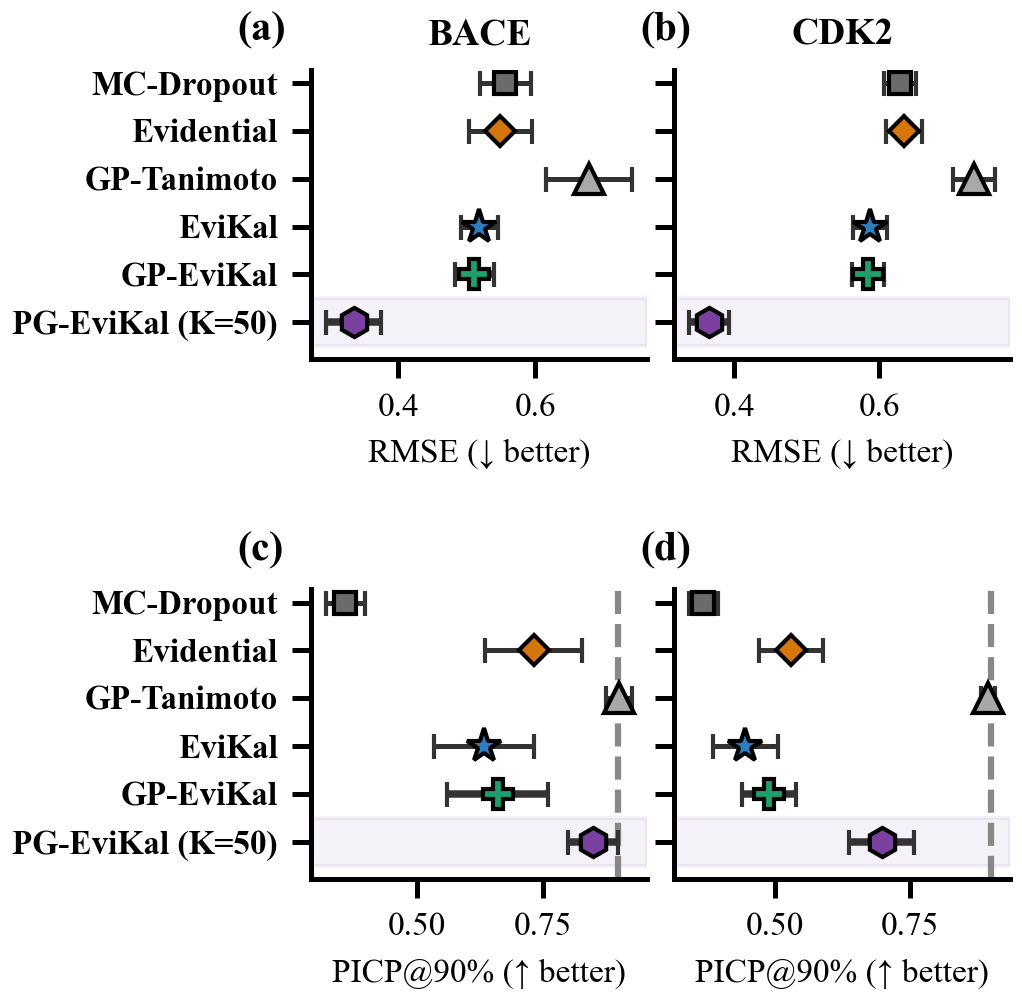}
\caption{
  \textbf{Why use an evidential neural network over a plain Gaussian process for neighbor fusion?}
  The evidential model is much more accurate than the plain GP. \evikalmpg{} (shaded purple) leverages \evikalm{} with property-guided neighbor selection, becoming substantially more accurate than the base evidential model and much closer in calibration to the plain GP.
}
\label{fig:gp_tanimoto}
\end{figure}

\subsection{\evikalmpg{}: Property-Supervised Neighbor Selection}
\label{sec:results_a2}
\noindent Does learning the property distance between a query and neighbors improve neighbor selection enough to help on rougher landscapes? Property-guided selection enables \evikalmpg{} to improve on 14 of 16 datasets at $K=50$ (Table~\ref{tab:main}). On smooth datasets where \evikalm{} already helped, the gains amplify substantially. Thermosol transforms from barely moving ($-0.7\%$ with Tanimoto) to $-41.6\%$ with learned property distance. This reveals the core problem: Tanimoto was selecting structurally similar but property-irrelevant neighbors that degraded the posterior. NCI-60 shows the same pattern. At $K=5$ neighbors, Tanimoto selection yields minimal gain; at $K=50$ neighbors, a larger candidate pool allows PropDist to identify informative neighbors even on a QSAR rough landscape.

FreeSolv is the sole failure ($+4.0\%$). The limitation is data scarcity: 513 training molecules provide insufficient examples for PropDist to learn generalizable property distances. Structural sparsity (median top-1 neighbor similarity = 0.20) compounds this---every candidate is distant, so no re-ranking strategy gives rise to meaningfully better alternatives. The core insight is that improvement is driven by neighbor quality, not quantity. The PropDist MLP utilized by \evikalmpg{} does not increase the neighbors $K$. Rather, the predicted property distance is used to reorder the neighbor set of Tanimoto-prescreen candidates to prioritize property-informative molecules. The GP posterior then fuses them exactly as before (Eq.~\eqref{eq:gp_mean}--\eqref{eq:gp_var}).

\subsection{Online \evikalmpg{}: Test-Time Assay Incorporation Without Retraining}
\label{sec:results_online}

\noindent \evikalmpg{} incorporates assay results without retraining (Figure~\ref{fig:online_kalman_summary}). As new data arrives, a pre-trained evidential model applies posterior updates to refine the predictions at test time.
This assay scenario matters when effective retraining is infeasible, whether due to limited data, computational constraints, training instability, or the need for immediate refinement from new data.

For this assay scenario, practitioners hold a seed model trained on 20\% of data and receive experimental results (assay measurements) in eight sequential batches. Figure~\ref{fig:online_kalman_summary} shows four representative datasets; full results in Appendix~\ref{app:additional} (Figure \ref{fig:online_kalman}).
At each round, online neighbor fusion incorporates the new batch through posterior updates.
We compared this against retraining the model from scratch on all accumulated data (Figure~\ref{fig:online_kalman_summary}, dashed gray baseline vs.\ squares). 

On smooth landscapes (Figure~\ref{fig:online_kalman_summary}a,b) property-guided neighbor selection (purple hexagons) matches or beats retraining, especially with small seed sets. HOPV is the clearest example. 54 molecules are too few for the model to learn well, yet \evikalmpg{} reaches -41.3\% error compared to the seed model without any retraining. On rough landscapes (Figure~\ref{fig:online_kalman_summary}c,d), retraining eventually wins by refining the encoder, which is an advantage that these posterior updates cannot replicate. 

The takeaway is straightforward: when retraining is unstable or infeasible, property-guided neighbor fusion provides a practical path forward.
On smooth landscapes with small seed sets, \evikalmpg{} succeeds where gradient-based fine-tuning fails. On rough landscapes, retraining dominates once data accumulates.
\evikalmpg{} in this assay scenario fills a real gap; practitioners can incorporate new experimental data without retraining.


\section{Discussion}
\label{sec:discussion}

\begin{figure*}[!t]
\centering
\includegraphics[width=1.0\linewidth]{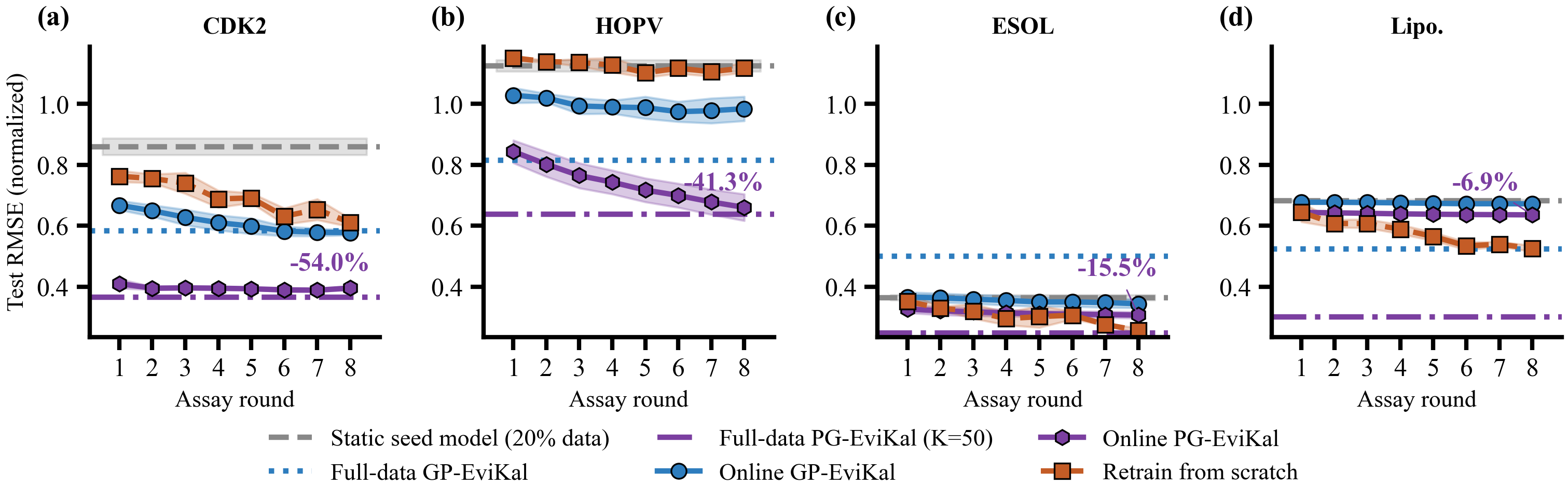}
\caption{\textbf{\evikalmpg{} for sequential assay incorporation without retraining.} Starting from a seed evidential model (20\% of data), online inference applies posterior updates as assay batches arrive over eight rounds, without retraining. Property-guided neighbor selection prioritizes informative neighbors, with opposing outcomes across QSAR landscapes. Percentages inset refer to test set RMSE reduction of \evikalmpg{} compared to the seed model baseline. \textbf{(a) CDK2:} on smooth landscapes, property-guided selection substantially outperforms retraining. \textbf{(b) HOPV:} 54 seed molecules are too few for stable fine-tuning, yet online \evikalmpg{} succeeds without any gradient updates. \textbf{(c, d) ESOL and Lipophilicity:} On rough landscapes, retraining dominates by refining encoder representations, something the posterior updates cannot match.}
\label{fig:online_kalman_summary}
\end{figure*}

\paragraph{When to use \evikalmpg{}.}
\evikalmpg{} is the recommended variant: it extends \evikalm{} by replacing Tanimoto neighbor selection with a learned property-distance metric and improves RMSE on 14 of 16 datasets (Table~\ref{tab:main}). PropDist training requires only the training molecules used for the evidential model (no test labels, no retraining) and should be the default when the training set is large enough for the
metric to generalize ($\gtrsim 1{,}000$ molecules).

Two characteristics determine whether neighbor fusion will help (Figure~\ref{fig:diagnostic}). The QSAR smoothness ratio $s$ measures whether structurally similar molecules tend to have similar properties. When $s < 0.65$ in our benchmark, neighbors reliably inform refinement. Above this, they become too unpredictable to reliably sharpen the posterior. The evidential model's calibration (PICP@90\%) indicates whether the prior uncertainty is well-estimated. When PICP@90\% $\geq 0.70$ in our benchmark, neighbor fusion improves accuracy and maintains uncertainty quality. Below this threshold, RMSE still improves, but corrections compress the posterior more aggressively. Calibration recovery requires careful regularization tuning before deployment (Section~\ref{sec:results_calibration}). Together, these two criteria constitute an explicit domain of applicability for neighbor fusion in the OECD sense: $s$ (with the underlying $\mathrm{SNR}_{\mathrm{eff}}$ condition and the neighbor coverage it depends on) bounds the chemical neighborhoods in which a training label is an informative observation of a query, while PICP@90\% bounds the model states in which the evidential prior can be trusted; queries outside these bounds lie outside the domain where test-time refinement is expected to help. These thresholds are empirical and should be validated on each dataset in practice.

\paragraph{Computational cost.}
\evikalmpg{} is practical. PropDist adds one GPU-batched forward pass over prescreen candidates and a GP solve, bringing total inference time to roughly 3× a single-model forward pass (Appendix~\ref{app:runtime}). This is substantially cheaper than \mcd{}. Both \evikalm{} and PropDist are post-hoc refinements; training is identical to the evidential baseline.

\paragraph{Limitations.}
Despite \evikalmpg{}'s effectiveness on 14 of 16 datasets, several limitations warrant attention. The noise model assumes structural dissimilarity and aleatoric uncertainty contribute independently to observation noise. On datasets where this assumption breaks down---such as those with prevalent activity cliffs \citep{vantilborg2022exposing, stumpfe2012exploring}---\evikalmpg{} provides less benefit than the smoothness diagnostic predicts.

Neighbor fusion is fundamentally limited by the feature representation. Both the AttentiveFP molecular graph the base model encodes and the ECFP4 fingerprints used for neighbor retrieval are topological representations of a molecule; they encode discrete structural features (atom types, bond connectivity, functional groups) but cannot capture well the continuous variations in electronic structure that drive quantum-chemical properties \citep{faber2017prediction, gilmer2017neural}. On QM8 and QM9, properties depend on subtle orbital interactions and electron distributions that such topological representations cannot represent in fine detail. The evidential model learns to predict these properties from its molecular-graph embedding as best as possible, but Tanimoto-similar neighbors share the same topological blind spot and offer no additional information about these unrepresented effects. On such datasets, improving the molecular representation itself is necessary \citep{schutt2018schnet}.

The smoothness threshold ($s < 0.65$) and calibration threshold (PICP@90\% $\geq 0.70$) are coupled to the choice of evidential neural network and should be re-assessed on each dataset before deployment. Similarly, the GP posterior assumes independent measurement noise per neighbor. Systematic batch effects in experimental data can violate this assumption and bias the posterior.

A counterintuitive failure mode occurs on datasets where the evidential model's epistemic uncertainty is paradoxically low for out-of-distribution queries. Here, \evikalmpg{} applies the smallest correction precisely where caution is most warranted (Appendix~\ref{app:ood}). Finally, neighbor fusion is a test-time refinement. It cannot substitute for retraining when new data introduces qualitatively different chemical scaffolds or regions of chemical space. It is most effective for property correction on the in-domain chemical space the evidential model was trained on.


\section{Conclusions \& Future Work}
\label{sec:conclusion}

\noindent Molecular property models are trained and then used to make test predictions, yet the training set remains available at inference time. We show that evidential models, which decompose predictive uncertainty into aleatoric and epistemic components, can exploit this information. When a molecular structure-property relationship is smooth (i.e., when structurally similar molecules tend to have similar properties), similar training molecules act as noisy observations that refine a query's property estimate without any retraining or architectural changes.

The key insight is that one can leverage the uncertainty decomposition of evidential neural networks to model this fusion. If a training neighbor is modeled as an observation of the query property, then the aleatoric uncertainty of the query sets a floor: even a perfectly similar neighbor inherits this same measurement noise and cannot reduce uncertainty below it. Epistemic uncertainty, by contrast, represents model ignorance that neighbors can reduce. This naturally leads to treating neighbors as information sources. Neighbors serve as a corrective signal for a molecule's evidential property prediction and epistemic uncertainty.

Structurally similar neighbors are information sources, but which neighbors matter most? Tanimoto similarity, a standard choice, captures only structural overlap. However, the QSAR smoothness assumption depends on whether shared structures actually predict the target property. \evikalmpg{} learns property distance from the training set, recognizing that structural and property relevance are equally important for refining predictions. Across 16 datasets, \evikalmpg{} improves accuracy on 14 of 16 (Table~\ref{tab:main}) and recovers calibration where the Tanimoto-only selection of \evikalm{} fails. On borderline datasets like Thermosol and NCI-60, \evikalmpg{} recovers meaningful gains by prioritizing property-relevant neighbors---revealing that selection quality, not the fusion algorithm, was the bottleneck.

\evikalmpg{} uses a drop-in Gaussian process on top of the pre-trained evidential model, as well as a pre-trained MLP for predicting property distance for neighbor re-ranking. At test time, no retraining is required to refine molecular property predictions of the evidential model. On smooth landscapes receiving assay data sequentially, \evikalmpg{} matches or exceeds full retraining without gradient updates. This work shows that evidential regression is more than a calibration tool. The uncertainty decomposition is an actionable inference resource. Evidential uncertainty, properly decomposed, carries information that neighbors can help resolve, transforming uncertainty estimates into an active inference resource for refining molecular properties.

\paragraph{Future work.}
One avenue of exploration is whether property-supervised similarity generalizes beyond ECFP4 fingerprints. Foundation models trained on large chemical datasets \citep{chithrananda2020chemberta, ross2022large} could, in principle, provide richer structural representations tailored to a property by using the model's learned embeddings. Rather than using the similarity of molecular fingerprints, the similarity of a query to a neighbor could be measured between a model's learned embeddings. More importantly, foundation models, which are usually trained on large chemical spaces, enable access to potentially denser chemical neighborhoods and therefore potentially higher quality sets of neighbors for fusion. With similarity measured across learned embeddings and an abundance of similar neighbors per query, the QSAR smoothness assumption may hold better, and therefore, the property refinement may be more robust.

Another avenue is to explore unifying the PropDist training with the evidential objective end-to-end, allowing both to jointly optimize. FreeSolv remains the principal failure case---limited by training-pair scarcity, not landscape roughness. Together with Thermosol and NCI-60, these three datasets form natural test beds to extend neighbor fusion.

\paragraph{Data availability.}
All sixteen benchmark datasets, the trained evidential model checkpoints (five seeds each), the property-distance neighbor-selection models, per-dataset normalization statistics, and the aggregated result files are archived at \href{https://doi.org/10.5281/zenodo.21287083}{10.5281/zenodo.21287083}, together with per-dataset metadata and SHA-256 checksums.

\paragraph{Code availability.}
All experiments, baselines, and the QSAR smoothness diagnostic are implemented in an open-source repository at \url{https://github.com/CGruich/PG-EviKal}~\href{https://github.com/CGruich/PG-EviKal}{\faGithub}, released under the MIT license. The repository includes all code, the configuration files to reproduce every result in Table~\ref{tab:main}, and a self-contained tutorial notebook that walks through both \evikalmpg{} and the biased-mean Gaussian process of Appendix~\ref{app:averaging}; the trained model checkpoints are provided in the data archive above. The software is archived at \href{https://doi.org/10.5281/zenodo.21287045}{10.5281/zenodo.21287045}, a version-independent DOI that resolves to all releases; this paper corresponds to release v1.0.0. All reported results were produced with Python 3.10.20, PyTorch 2.2.2 (CUDA 12.1), PyTorch Geometric 2.5.2, RDKit 2023.09.6, DeepChem 2.8.0, NumPy 1.26.4, SciPy 1.15.2, scikit-learn 1.7.2, pandas 2.3.3, and Matplotlib 3.10.9.

\paragraph{Acknowledgments.}
C.G. acknowledges support from the NSF Graduate Research Fellowship Program under Grant No. DGE 1841052. All authors also acknowledge NSF Grant No. 2435696 for financial support.

\paragraph{Contributions.}
Conceptualization, C.G.; Methodology, C.G., W.Y.; Code, C.G.; Visualization, C.G.; Resources, C.G., Y.W., B.R.G.; Data Curation, C.G.; Writing—Original Draft, C.G.; Writing—Review \& Editing, C.G., W.Y., Y.W., B.R.G.

\paragraph{Conflict of Interest.}
The authors declare no competing financial interest.

\paragraph{ORCID IDs.}~\\
\noindent Cameron Gruich \orcidlink{0000-0002-3801-1296} \\
\noindent Weichi Yao \orcidlink{0000-0002-3412-5317} \\
\noindent Yixin Wang \orcidlink{0009-0008-1816-4982} \\
\noindent Bryan R. Goldsmith  \orcidlink{0000-0002-6617-4842}

\FloatBarrier
\bibliography{refs}
\bibliographystyle{icml2026}

\appendix
\onecolumn
\raggedbottom  


\section{Datasets and Data Curation}
\label{app:datasets}

\noindent Table~\ref{tab:datasets} summarizes the sixteen molecular datasets used in this work.
Scaffold splits use Bemis--Murcko frameworks~\citep{bemis1996properties};
random splits are uniformly random $80/10/10$ partitions.
The QSAR smoothness ratio $s$ (Eq.~\ref{eq:smoothness} in the main text) is computed
on the test set of seed 0 using ECFP4 fingerprints with radius 2 and 2048 bits.

\subsection{Data curation}
\label{app:curation}

\noindent All sixteen datasets are featurized through a common pipeline. SMILES strings are parsed
and canonicalized with RDKit~\citep{rdkit}. Structures that fail RDKit sanitization are discarded,
and the RDKit canonical SMILES of each surviving structure is used as its identity key. Records
with a missing structure or a missing label are dropped. Where a source reports replicate
measurements for the same canonical SMILES, the replicates are aggregated to their median value
before splitting. The 5-HT2A and D2 sets are additionally desalted by retaining the largest organic
fragment of each structure, removing counter-ions that are predominantly hydrochloride; the
remaining datasets are used as supplied by their sources. Molecular fingerprints are ECFP4 (Morgan,
radius 2, 2048 bits) computed from these canonical structures, and the same structures are supplied
to the evidential neural network.

\paragraph{ChEMBL bioactivity datasets (CDK2, hERG, 5-HT2A, D2).}
Activities are restricted to IC$_{50}$ measurements reported with an exact relation ($=$), in nM
units, from binding assays (ChEMBL \texttt{assay\_type} $=$ B) against single-protein targets.
The pChEMBL value is used where reported; otherwise, the standard value is converted to
pIC$_{50}$. Activities outside pIC$_{50} \in [3, 12]$ are discarded as lying beyond the reliable
dynamic range of the assays. Replicate measurements of the same canonical structure are
aggregated to their median pIC$_{50}$.

\paragraph{pKa.}
Records are taken from the OPERA pKa training set~\citep{mansouri2019opera}, filtered to the
$[0, 14]$ pKa range and aggregated to the median pKa per canonical structure.

\paragraph{Public benchmark datasets.}
ESOL, FreeSolv, Lipophilicity, BACE, Thermosol, QM7, QM8, and QM9 are used as distributed by
MoleculeNet/DeepChem~\citep{wu2018moleculenet} and are curated upstream by those repositories.
LD50 is the acute oral toxicity set of \citet{zhu2009quantitative}. HOPV is deduplicated by
canonical SMILES~\citep{lopez2016harvard}. NCI-60 is the mean log GI$_{50}$ across cell
lines~\citep{reinhold2012ncialmanac}.

\paragraph{Scope of standardization.}
Beyond the steps above, structures are used as supplied by their source repositories. Fingerprints
and molecular graphs are computed from the same sanitized, canonicalized structures. Both the
neighbor retrieval fingerprints and the model inputs are derived from these curated structures, so
a query and its retrieved neighbors are always compared in a single consistent representation.

\subsection{Datasets}
\label{app:dataset_descriptions}

\begin{table}[h]
\centering
\caption{Dataset statistics. $s$ = QSAR smoothness ratio (Eq.~\ref{eq:smoothness});
         PICP@90\% is the 5-seed mean evidential baseline.
         Split sizes are from seed 0; other seeds differ by $<5$ molecules.
         QM8 ($s=0.136$) and QM9 ($s=0.136$) are the two accuracy-ceiling cases; QM7 ($s=0.834$) is the roughest dataset.}
\label{tab:datasets}
\small
\setlength{\tabcolsep}{4pt}
\begin{tabular}{lcccccccc}
\toprule
Dataset & Property & Unit & Split & $N_\mathrm{train}$ & $N_\mathrm{val}$ & $N_\mathrm{test}$ & $s$ & PICP@90\% \\
\midrule
QM8         & E$_1$ (CC2, excitation)  & eV    & random   & 17429  & 2179  & 2178  & 0.136 & 0.585 \\
pKa         & Aqueous pKa              & pKa unit & random & 6996  & 874   & 875   & 0.165 & 0.408 \\
HOPV        & Power conversion eff.    & \%    & random   & 272    & 34    & 34    & 0.243 & 0.794 \\
BACE        & pIC$_{50}$ (BACE1)       & ---   & random   & 1210   & 151   & 152   & 0.298 & 0.732 \\
CDK2        & pIC$_{50}$ (CDK2)        & ---   & random   & 1648   & 206   & 206   & 0.335 & 0.530 \\
5-HT2A      & pIC$_{50}$ (5-HT2A)      & ---   & random   & 4310   & 539   & 539   & 0.346 & 0.468 \\
D2          & pIC$_{50}$ (D2R)         & ---   & random   & 6746   & 843   & 844   & 0.367 & 0.642 \\
hERG        & pIC$_{50}$ (hERG)        & ---   & random   & 6690   & 836   & 836   & 0.379 & 0.520 \\
LD50        & Acute oral toxicity      & log mol/kg & random & 5908 & 738 & 739  & 0.404 & 0.497 \\
ESOL        & Water solubility         & log mol/L & scaffold & 902  & 112  & 114   & 0.540 & 0.812 \\
Lipo.       & log $D$ (pH 7.4)         & ---   & scaffold & 3360   & 420   & 420   & 0.541 & 0.671 \\
Thermosol   & Kinetic solubility       & log $\mu$M & random & 1410 & 176  & 177   & 0.712 & 0.736 \\
FreeSolv    & Hydration free energy    & kcal/mol & scaffold & 513  & 64   & 65    & 0.756 & 0.726 \\
NCI-60      & Mean log GI$_{50}$       & log $\mu$M & random & 15301 & 1913 & 1912 & 0.779 & 0.893 \\
QM7         & Atomization energy       & kcal/mol & random & 5470  & 684   & 684   & 0.834 & 0.912 \\
QM9         & HOMO--LUMO gap           & eV    & random   & 100000 & 10000 & 10000 & 0.136 & 0.550 \\
\bottomrule
\end{tabular}
\end{table}

\paragraph{QM8.}
The QM8 dataset~\citep{ramakrishnan2015electronic} contains 21786 small organic molecules (up to 8 heavy atoms: C, N, O, F)
from the GDB-17 chemical space, with electronic spectra computed at multiple levels of theory
using time-dependent DFT and coupled-cluster (CC2) theory.
We use the lowest singlet excitation energy $E_1^{\mathrm{CC2}}$ as the prediction target.
Despite being a quantum-chemical property, QM8 has the lowest QSAR smoothness ratio in our benchmark
($s = 0.136$). Electronic excitation energies are dominated by the conjugation network
and chromophore group, both of which AttentiveFP captures well from the molecular graph.
Yet, QM8 does not benefit from \evikalm{} ($+0.5\%$, essentially neutral) due to a base model accuracy
ceiling ($R^2 \approx 0.977$). The neural prior learns the property to the extent its topological molecular-graph representation permits, and that representation is a coarse, discretized signal for the fine electronic structure these properties depend on; the Tanimoto-similar (ECFP4) neighbors used for fusion inherit the same limitation, so they add no signal.
This illustrates that the QSAR smoothness ratio is necessary but not sufficient for \evikalm{} benefit.
The base model must also have room to improve.

\paragraph{pKa.}
The aqueous pKa dataset contains 8745 drug-like compounds from the OPERA pKa training set~\citep{mansouri2019opera}
with measured acid dissociation constants (pKa) in the range 0--14.
We use the median pKa per SMILES for multi-protic compounds.
pKa is the smoothest drug-like dataset in our benchmark ($s = 0.165$). Neighboring compounds
share functional groups that dominate the acid--base equilibrium,
making ECFP4-based neighbor labels highly predictive.
Despite this, the evidential model is overconfident (PICP@90\% = 0.41), placing pKa
in the ``smooth but overconfident'' quadrant of the diagnostic (Figure~\ref{fig:diagnostic}). 
\evikalm{} improves RMSE by $-3.5\%$ but calibration
does not recover without increased evidence regularization.

\paragraph{HOPV.}
The Harvard Organic Photovoltaic Dataset~\citep{lopez2016harvard} contains
340 organic photovoltaic molecules (after RDKit standardization) with experimentally
measured power conversion efficiency (PCE).
The small size and chemically heterogeneous structures make this the worst \mcd{} result
on any smooth dataset (PICP@90\% = 0.177).
Despite heterogeneity, the QSAR landscape is smooth ($s = 0.243$), likely because
PCE is dominated by a small number of shared structural motifs across OPV candidates.

\paragraph{BACE and CDK2.}
Both datasets contain pIC$_{50}$ measurements (negative log of half-maximal inhibitory
concentration) from ChEMBL~\citep{gaulton2012chembl}.
BACE inhibitors share a common pharmacophore, producing high intra-dataset Tanimoto
similarity (median test-to-train: 0.812 for BACE random split); this is the primary
reason \evikalm{} performs particularly well here ($-6.9\%$), the third largest improvement in the benchmark behind CDK2 ($-7.9\%$) and D2 ($-7.6\%$).
CDK2 inhibitors are slightly more diverse (smoothness 0.335 vs.\ 0.298 for BACE)
but still meet the smooth QSAR criterion.

\paragraph{5-HT2A and D2 receptors.}
These two datasets contain pIC$_{50}$ measurements for serotonin receptor 5-HT2A and dopamine
receptor D2R, curated from ChEMBL~\citep{gaulton2012chembl}.
Both are central nervous system (CNS) GPCR targets relevant to antipsychotic drug design.
The 5-HT2A dataset (5388 molecules, $s = 0.346$) and D2 dataset (8433 molecules, $s = 0.367$)
show smooth QSAR landscapes consistent with other GPCR-binding pIC$_{50}$ targets (hERG, CDK2).
Salt forms (predominantly hydrochloride counter-ions) were removed by retaining the largest
organic fragment of each structure before fingerprint computation.

\paragraph{LD50.}
The LD50 dataset from Zhu et al.~\citep{zhu2009quantitative} contains 7385 drug-like molecules
with measured acute oral median lethal dose (LD50) converted to $-\log_{10}(\mathrm{mol/kg})$.
Despite the mechanistic diversity of toxic endpoints, the dataset is dominated by drug-like
organic compounds forming congeneric series, yielding a smooth landscape ($s = 0.404$).
The dataset was accessed from the Therapeutics Data Commons (TDC) collection~\citep{huang2021therapeutics}.

\paragraph{hERG.}
The hERG cardiac ion channel inhibition dataset (8362 molecules, pIC$_{50}$ from ChEMBL~\citep{gaulton2012chembl}) is the largest random-split drug-safety bioactivity dataset in our benchmark.
Despite its size, the smooth QSAR landscape ($s = 0.379$) allows \evikalm{} to achieve $-5.5\%$ RMSE improvement,
confirming that the benefit is driven by structural landscape geometry, not dataset size.

\paragraph{ESOL, FreeSolv, Lipophilicity.}
These three MoleculeNet benchmark datasets~\citep{wu2018moleculenet} use scaffold splits,
which specifically test cross-scaffold generalization.
Scaffold splits reduce test-to-train Tanimoto similarity relative to random splits,
which partially explains the weaker (but still positive) \evikal{} gains on ESOL and
Lipophilicity compared to the random-split drug safety datasets (Appendix~\ref{app:additional}).
FreeSolv is the failure case, but the mechanism is structural sparsity rather than
landscape roughness. Test molecules find their nearest training neighbor at only
0.20 Tanimoto similarity on average ($s = 0.756$, above the reliable-benefit
threshold of $s < 0.65$, Figure~\ref{fig:smoothness_analysis}b).
At that operating point, the observation noise is so large that each Kalman update
injects more error than it corrects, producing the sequential \evikal{} catastrophe
(+48.7\% RMSE). \evikalm{}'s batch posterior limits but does not eliminate this (+4.9\%) due to the Tanimoto covariance kernel that handles correlated neighbor effects.

\paragraph{QM7.}
The QM7 dataset contains 6{,}838 small organic molecules (up to 7 heavy atoms: C, N, O, S) with DFT-computed atomization energies at the PBE0/tier-2 level~\citep{rupp2012fast}.
SMILES strings are obtained from the DeepChem/MoleculeNet distribution and converted to molecular graphs
with standard ECFP4 fingerprints.
We use a random 80/10/10 split.
The QSAR smoothness ratio $s = 0.834$ places QM7 in the borderline regime ($0.65 \leq s < 1.0$).
Structurally similar molecules share scaffold fragments that partially predict atomization energy,
but the property also depends on fine geometric and electronic details invisible to ECFP4.
The evidential model is well-calibrated (PICP@90\% $= 0.91$), so the calibration condition is met.
However, with median top-1 Tanimoto similarity of only 0.55, the noise injected on neighbors by our Tanimoto dissimilarity noise model is in many cases too large for reliable Kalman corrections on these molecules.
\evikalm{} yields $+0.05\%$ RMSE---a noise-level gain close to zero---confirming
the diagnostic prediction. QM7's QSAR landscape is rough globally, and at its operating point the nearest neighbors are too dissimilar to provide useful structural signal for \evikalm{} property refinement.

\paragraph{Thermosol.}
The Thermosol dataset~\citep{wu2018moleculenet} contains 1763 drug-like molecules with kinetic solubility measurements (log $\mu$M)
from a high-throughput kinetic nephelometry assay, accessed via the DeepChem library.
The dataset spans a structurally diverse collection of pharmaceutical compounds with
$s = 0.712$ and median top-1 Tanimoto similarity $= 0.71$ (random 80/10/10 split).
Thermosol sits above the smooth QSAR threshold ($s = 0.712 > 0.65$), placing it in the
borderline zone. However, \evikalm{} still yields a small but consistent $-0.7\%$ RMSE improvement across 5 seeds,
placing Thermosol at the borderline edge of the benefit regime.
The dataset is also included in the online \evikal{} simulation (Section~\ref{sec:results_online}),
representing a realistic HTS solubility screening scenario where compounds are assayed in batches.

\paragraph{NCI-60.}
The NCI-60 dataset~\citep{reinhold2012ncialmanac} is derived from the NCI Developmental Therapeutics Program's 60-cell-line GI$_{50}$ panel (concentration that inhibits growth by 50\%), one of the most widely used phenotypic anticancer screens.
We computed the mean log GI$_{50}$ (in $\mu$M) across all cell lines for which a valid measurement
existed, yielding 19126 molecules.
The QSAR smoothness ratio $s = 0.779$ reflects the target's mechanistic heterogeneity. Structurally
similar compounds can act through distinct cytotoxic mechanisms on different cell lines, producing
higher label variance and therefore lower QSAR smoothness across structural neighbors than single-target bioactivity assays.

Despite this roughness, the dense structural coverage (median top-1 Tanimoto similarity $= 0.61$)
enables \evikalm{} to achieve $-2.9\%$ RMSE improvement. The reason \evikalm{} helps on NCI-60 but not
FreeSolv---despite nearly identical smoothness ratios---is visible in the operating point triangles of
Figure~\ref{fig:smoothness_analysis}. NCI-60's neighbors are genuinely close in structure (top-1 sim $=
0.61$), so even on a rough landscape they carry enough signal to correct the prediction. FreeSolv's
neighbors are far away (top-1 sim $= 0.20$), placing the operating point in a region where property
gaps are large and neighbors are too dissimilar to be useful sensors. Structural coverage, not
roughness alone, determines whether neighbor fusion helps.

The evidential baseline is well-calibrated on NCI-60 (PICP@90\% $= 0.89$), consistent with
the diverse training distribution that regularizes the evidential prior.
NCI-60 is included in the online \evikalmpg{} simulation (Section~\ref{sec:results_online}) as a
borderline dataset. Retraining wins decisively here because diverse cytotoxic mechanisms require
updated encoder representations that the posterior updates cannot replicate.

\paragraph{QM9.}
We use the HOMO--LUMO gap target from the QM9 quantum chemistry dataset~\citep{ramakrishnan2014qm9},
containing 130,831 small organic molecules computed at the B3LYP/6-31G(2df,p) level of theory. 120,000
molecules were used in a 100,000/10,000/10,000 random split. QM9 has a smooth QSAR landscape ($s =
0.136$, nearly identical to QM8). At high Tanimoto similarity, ECFP4 fingerprints reliably track
HOMO--LUMO gap similarity because molecules that are highly structurally similar tend to share the same
conjugation patterns, ring systems, and substituents that collectively determine frontier orbital
energies. This local structural agreement at the operating point Tanimoto similarities
(Figure~\ref{fig:smoothness_analysis}a) is sufficient
for the QSAR smoothness assumption to hold.

With 100k training molecules, however, the evidential model reaches an accuracy ceiling because
its topological molecular-graph representation (discretized atom and bond features) can only encode so much meaningful information about
electronic structure. QM9 reaches $R^2 \approx 0.977$ under the same hyperparameters as QM8 (their coefficients of determination are virtually identical), and its small mean epistemic variance ($\bar{u}_e \approx 0.006\,\sigma_{\mathrm{train}}^2$, i.e.\ $\approx 0.009$~eV$^2$) indicates the model has absorbed most of the variance that this representation can encode.
The 0.189 eV RMSE is chemically imprecise for HOMO--LUMO gap prediction, but the
representation-inaccessible gains leave negligible room for refinement, like QM8. \evikalm{} yields
$-0.6\%$ RMSE (evidential $0.189 \pm 0.008$ eV vs.\ \evikalm{} $0.188 \pm 0.007$ eV), which is
noise-level and near zero. A high-regularization ablation ($\lambda = 0.05$) worsens the evidential
baseline and yields $+0.7\%$.
\begin{table*}[t]
\centering
\caption{
Main results across sixteen molecular datasets (5 seeds, mean $\pm$ std). Ordered by QSAR smoothness $s$ (Eq.~\ref{eq:smoothness}). \textbf{Bold} marks the best value in each column among all methods except
GP-Tanimoto, which attains high calibration only by widening its intervals at the cost of large prediction error. Deep Ensemble ($M=5$) shown for HOPV and BACE; full results in Appendix~\ref{app:ensemble}.
\dag GP-Tanimoto on BACE/CDK2 only. RMSE and $\pm$std are normalized by the training-set standard deviation $\sigma_{\mathrm{train}}$ (reported per dataset in the first column) for all datasets; multiply by
$\sigma_{\mathrm{train}}$ to recover physical units.
}
\label{tab:main}
\small
\setlength{\tabcolsep}{4pt}
\begin{tabular}{llccccc}
\toprule
Dataset ($s$) & Method & RMSE $\downarrow$ & $\pm$std & PICP@90\% $\uparrow$ & PICP@95\% $\uparrow$ &
ECE $\downarrow$ \\
\midrule
\multirow{6}{*}{\shortstack[l]{QM8\\($s=0.136$)\\($\sigma_{\mathrm{train}}=0.0438$)}}
& \mcd{}     & 0.149 & 0.007 & \textbf{0.721} & \textbf{0.782} & \textbf{0.106} \\
& \evid{}    & 0.151 & 0.007 & 0.585 & 0.652 & 0.199 \\
& \evikal{}    & 0.151 & 0.005 & 0.584 & 0.651 & 0.200 \\
& \evikalm{} & 0.151 & 0.005 & 0.584 & 0.651 & 0.199 \\
& \evikalmpg{}& 0.149 & 0.007 & 0.586 & 0.654 & 0.198 \\
& \evikalmpg{} ($K=50$) & \textbf{0.146} & 0.007 & 0.586 & 0.653 & 0.197 \\
\midrule
\multirow{6}{*}{\shortstack[l]{QM9\\($s=0.136$)\\($\sigma_{\mathrm{train}}=1.26$)}}
& \mcd{}     & \textbf{0.147} & 0.007 & \textbf{0.671} & \textbf{0.740} & \textbf{0.146} \\
& \evid{}    & 0.150 & 0.006 & 0.550 & 0.620 & 0.222 \\
& \evikal{}    & 0.150 & 0.005 & 0.525 & 0.594 & 0.236 \\
& \evikalm{} & 0.149 & 0.006 & 0.545 & 0.614 & 0.224 \\
& \evikalmpg{}& 0.148 & 0.006 & 0.528 & 0.597 & 0.230 \\
& \evikalmpg{} ($K=50$) & 0.150 & 0.006 & 0.530 & 0.599 & 0.228 \\
\midrule
\multirow{6}{*}{\shortstack[l]{pKa\\($s=0.165$)\\($\sigma_{\mathrm{train}}=2.31$)}}
& \mcd{}     & 0.575 & 0.035 & \textbf{0.542} & \textbf{0.604} & \textbf{0.215} \\
& \evid{}    & 0.588 & 0.049 & 0.408 & 0.464 & 0.288 \\
& \evikal{}    & 0.565 & 0.031 & 0.364 & 0.419 & 0.315 \\
& \evikalm{} & 0.568 & 0.039 & 0.393 & 0.448 & 0.294 \\
& \evikalmpg{}& 0.559 & 0.047 & 0.437 & 0.504 & 0.271 \\
& \evikalmpg{} ($K=50$) & \textbf{0.506} & 0.047 & 0.445 & 0.513 & 0.263 \\
\midrule
\multirow{7}{*}{\shortstack[l]{HOPV\\($s=0.243$)\\($\sigma_{\mathrm{train}}=2.32$)}}
& \mcd{}     & 0.834 & 0.038 & 0.177 & 0.241 & 0.409 \\
& Deep Ens.  & 1.437 & 0.232 & 0.335 & 0.394 & 0.344 \\
& \evid{}    & 0.849 & 0.080 & 0.794 & 0.871 & 0.076 \\
& \evikal{}    & 0.812 & 0.090 & 0.724 & 0.812 & 0.107 \\
& \evikalm{} & 0.814 & 0.089 & 0.771 & 0.841 & 0.079 \\
& \evikalmpg{}& 0.740 & 0.099 & 0.848 & 0.912 & 0.088 \\
& \evikalmpg{} ($K=50$) & \textbf{0.638} & 0.105 & \textbf{0.865} & \textbf{0.923} & \textbf{0.075} \\
\midrule
\multirow{8}{*}{\shortstack[l]{BACE\\($s=0.298$)\\($\sigma_{\mathrm{train}}=1.34$)}}
& \mcd{}           & 0.556 & 0.037 & 0.359 & 0.428 & 0.316 \\
& Deep Ens.        & 0.515 & 0.083 & 0.482 & 0.554 & 0.255 \\
& \evid{}          & 0.549 & 0.046 & 0.732 & 0.790 & 0.104 \\
& GP-Tanimoto\dag  & 0.678 & 0.062 & 0.901 & ---   & 0.072 \\
& \evikal{}          & 0.518 & 0.027 & 0.633 & 0.679 & 0.158 \\
& \evikalm{}       & 0.511 & 0.029 & 0.661 & 0.738 & 0.128 \\
& \evikalmpg{}      & 0.406 & 0.043 & 0.819 & 0.891 & 0.091 \\
& \evikalmpg{} ($K=50$) & \textbf{0.335} & 0.040 & \textbf{0.849} & \textbf{0.908} & \textbf{0.080} \\
\midrule
\multirow{7}{*}{\shortstack[l]{CDK2\\($s=0.335$)\\($\sigma_{\mathrm{train}}=1.27$)}}
& \mcd{}           & 0.628 & 0.037 & 0.368 & 0.424 & 0.319 \\
& \evid{}          & 0.634 & 0.025 & 0.530 & 0.590 & 0.225 \\
& GP-Tanimoto\dag  & 0.731 & 0.029 & 0.894 & ---   & 0.032 \\
& \evikal{}          & 0.587 & 0.024 & 0.445 & 0.504 & 0.272 \\
& \evikalm{}       & 0.584 & 0.022 & 0.490 & 0.550 & 0.253 \\
& \evikalmpg{}      & 0.472 & 0.024 & 0.678 & 0.762 & 0.159 \\
& \evikalmpg{} ($K=50$) & \textbf{0.365} & 0.027 & \textbf{0.698} & \textbf{0.776} & \textbf{0.147} \\
\midrule
\multirow{6}{*}{\shortstack[l]{5-HT2A\\($s=0.346$)\\($\sigma_{\mathrm{train}}=1.19$)}}
& \mcd{}     & 0.557 & 0.016 & 0.411 & 0.477 & 0.295 \\
& \evid{}    & 0.567 & 0.021 & 0.468 & 0.540 & 0.264 \\
& \evikal{}    & 0.543 & 0.014 & 0.365 & 0.424 & 0.328 \\
& \evikalm{} & 0.538 & 0.014 & 0.371 & 0.427 & 0.323 \\
& \evikalmpg{}& 0.474 & 0.024 & 0.555 & 0.631 & 0.219 \\
& \evikalmpg{} ($K=50$) & \textbf{0.409} & 0.050 & \textbf{0.564} & \textbf{0.640} & \textbf{0.206} \\
\midrule
\multirow{6}{*}{\shortstack[l]{D2\\($s=0.367$)\\($\sigma_{\mathrm{train}}=1.00$)}}
& \mcd{}     & 0.615 & 0.023 & 0.399 & 0.467 & 0.300 \\
& \evid{}    & 0.638 & 0.049 & 0.642 & 0.703 & 0.160 \\
& \evikal{}    & 0.596 & 0.028 & 0.440 & 0.506 & 0.280 \\
& \evikalm{} & 0.589 & 0.029 & 0.508 & 0.579 & 0.240 \\
& \evikalmpg{}& 0.515 & 0.049 & 0.715 & 0.783 & 0.145 \\
& \evikalmpg{} ($K=50$) & \textbf{0.445} & 0.041 & \textbf{0.730} & \textbf{0.797} & \textbf{0.135} \\
\midrule
\end{tabular}
\end{table*}

\begin{table*}[t]
\centering
\small
\setlength{\tabcolsep}{4pt}
\begin{tabular}{llccccc}
\toprule
\multicolumn{7}{l}{\textit{Table~\ref{tab:main} continued}} \\
\midrule
Dataset ($s$) & Method & RMSE $\downarrow$ & $\pm$std & PICP@90\% $\uparrow$ & PICP@95\% $\uparrow$ &
ECE $\downarrow$ \\
\midrule
\multirow{6}{*}{\shortstack[l]{hERG\\($s=0.379$)\\($\sigma_{\mathrm{train}}=0.916$)}}
& \mcd{}     & 0.649 & 0.005 & 0.358 & 0.414 & 0.322 \\
& \evid{}    & 0.628 & 0.026 & 0.520 & 0.583 & 0.234 \\
& \evikal{}    & 0.597 & 0.013 & 0.397 & 0.454 & 0.308 \\
& \evikalm{} & 0.593 & 0.013 & 0.411 & 0.468 & 0.296 \\
& \evikalmpg{}& 0.535 & 0.018 & 0.603 & 0.671 & 0.194 \\
& \evikalmpg{} ($K=50$) & \textbf{0.462} & 0.028 & \textbf{0.618} & \textbf{0.684} & \textbf{0.182} \\
\midrule
\multirow{6}{*}{\shortstack[l]{LD50\\($s=0.404$)\\($\sigma_{\mathrm{train}}=0.962$)}}
& \mcd{}     & 0.626 & 0.029 & 0.362 & 0.431 & 0.324 \\
& \evid{}    & 0.631 & 0.040 & 0.497 & 0.558 & 0.244 \\
& \evikal{}    & 0.618 & 0.040 & 0.328 & 0.374 & 0.337 \\
& \evikalm{} & 0.599 & 0.032 & 0.441 & 0.504 & 0.274 \\
& \evikalmpg{}& 0.521 & 0.037 & 0.643 & 0.718 & 0.174 \\
& \evikalmpg{} ($K=50$) & \textbf{0.390} & 0.044 & \textbf{0.657} & \textbf{0.731} & \textbf{0.163} \\
\midrule
\multirow{6}{*}{\shortstack[l]{ESOL\\($s=0.540$)\\($\sigma_{\mathrm{train}}=2.07$)}}
& \mcd{}     & 0.484 & 0.013 & 0.330 & 0.389 & 0.347 \\
& \evid{}    & 0.512 & 0.016 & 0.812 & 0.861 & 0.065 \\
& \evikal{}    & 0.496 & 0.014 & 0.809 & 0.856 & 0.072 \\
& \evikalm{} & 0.500 & 0.014 & 0.814 & 0.860 & 0.076 \\
& \evikalmpg{}& 0.500 & 0.039 & 0.930 & 0.965 & 0.041 \\
& \evikalmpg{} ($K=50$) & \textbf{0.464} & 0.032 & \textbf{0.940} & \textbf{0.972} & \textbf{0.037} \\
\midrule
\multirow{6}{*}{\shortstack[l]{Lipo.\\($s=0.541$)\\($\sigma_{\mathrm{train}}=1.20$)}}
& \mcd{}     & 0.546 & 0.016 & 0.354 & 0.407 & 0.329 \\
& \evid{}    & 0.526 & 0.011 & 0.671 & 0.734 & 0.144 \\
& \evikal{}    & 0.525 & 0.011 & 0.634 & 0.701 & 0.171 \\
& \evikalm{} & 0.524 & 0.010 & 0.648 & 0.712 & 0.156 \\
& \evikalmpg{}& 0.519 & 0.053 & 0.898 & 0.936 & 0.067 \\
& \evikalmpg{} ($K=50$) & \textbf{0.467} & 0.031 & \textbf{0.906} & \textbf{0.942} & \textbf{0.061} \\
\midrule
\multirow{6}{*}{\shortstack[l]{Thermosol\\($s=0.712$)\\($\sigma_{\mathrm{train}}=0.972$)}}
& \mcd{}     & 0.730 & 0.049 & 0.308 & 0.362 & 0.351 \\
& \evid{}    & 0.754 & 0.050 & 0.736 & 0.797 & 0.117 \\
& \evikal{}    & 0.748 & 0.073 & 0.592 & 0.678 & 0.213 \\
& \evikalm{} & 0.749 & 0.069 & 0.658 & 0.728 & 0.176 \\
& \evikalmpg{}& 0.618 & 0.055 & 0.848 & 0.909 & 0.099 \\
& \evikalmpg{} ($K=50$) & \textbf{0.440} & 0.076 & \textbf{0.862} & \textbf{0.920} & \textbf{0.088} \\
\midrule
\multirow{6}{*}{\shortstack[l]{FreeSolv\\($s=0.756$)\\($\sigma_{\mathrm{train}}=3.28$)}}
& \mcd{}     & \textbf{0.604} & 0.019 & 0.326 & 0.379 & 0.338 \\
& \evid{}    & 0.625 & 0.044 & \textbf{0.726} & \textbf{0.785} & \textbf{0.101} \\
& \evikal{}    & 0.930 & 0.168 & 0.723 & 0.785 & 0.106 \\
& \evikalm{} & 0.656 & 0.052 & 0.723 & 0.785 & 0.104 \\
& \evikalmpg{}& 0.675 & 0.099 & 0.627 & 0.698 & 0.221 \\
& \evikalmpg{} ($K=50$) & 0.650 & 0.080 & 0.634 & 0.705 & 0.216 \\
\midrule
\multirow{6}{*}{\shortstack[l]{NCI-60\\($s=0.779$)\\($\sigma_{\mathrm{train}}=0.0391$)}}
& \mcd{}     & 0.968 & 0.012 & 0.089 & 0.110 & 0.457 \\
& \evid{}    & 0.970 & 0.011 & 0.893 & 0.942 & \textbf{0.012} \\
& \evikal{}    & 0.941 & 0.011 & 0.774 & 0.844 & 0.094 \\
& \evikalm{} & 0.942 & 0.011 & 0.851 & 0.904 & 0.036 \\
& \evikalmpg{}& 0.947 & 0.011 & 0.912 & 0.955 & 0.033 \\
& \evikalmpg{} ($K=50$) & \textbf{0.864} & 0.012 & \textbf{0.917} & \textbf{0.959} & 0.031 \\
\midrule
\multirow{6}{*}{\shortstack[l]{QM7\\($s=0.834$)\\($\sigma_{\mathrm{train}}=222$)}}
& \mcd{}     & 0.855 & 0.020 & 0.182 & 0.215 & 0.413 \\
& \evid{}    & 0.860 & 0.018 & 0.912 & 0.956 & 0.015 \\
& \evikal{}    & 0.860 & 0.020 & 0.902 & 0.949 & 0.012 \\
& \evikalm{} & 0.860 & 0.019 & 0.908 & 0.955 & 0.014 \\
& \evikalmpg{}& 0.855 & 0.018 & 0.917 & 0.961 & 0.013 \\
& \evikalmpg{} ($K=50$) & \textbf{0.833} & 0.017 & \textbf{0.920} & \textbf{0.963} & \textbf{0.012} \\
\bottomrule
\end{tabular}
\end{table*}
\section{Hyperparameter Details}
\label{app:hyperparams}

\subsection{AttentiveFP Training}

\noindent All evidential models use the AttentiveFP architecture~\citep{xiong2020attentivefp}
with the published default hyperparameters: hidden size 200, 3 graph attention layers,
2 output layers, dropout 0.1, batch size 200.
We replace the final linear layer with a 4-output Normal-Inverse-Gamma head (Section~\ref{sec:evid_head}).
Training uses Adam with learning rate $10^{-3}$, weight decay $10^{-5}$,
and early stopping on validation RMSE with patience 50 epochs (maximum 300 epochs).
The evidence regularization coefficient is $\lambda = 0.01$ throughout,
except for the CDK2 and QM9 $\lambda = 0.05$ variants discussed in Section~\ref{sec:results_main} and Appendix~\ref{app:hyperparams}. Atomic features for AttentiveFP comprise one-hot encodings for atomic number ($1$--$118$), degree ($0$--$10$), formal charge ($-5$--$+5$), hydrogen count ($0$--$8$), and hybridization (SP, SP\textsuperscript{2}, SP\textsuperscript{3}, SP\textsuperscript{3}D, SP\textsuperscript{3}D\textsuperscript{2}), plus boolean flags for aromaticity and ring membership. Bond features encode bond type (single, double, triple, aromatic) and flags for conjugacy and ring membership.

\subsection{\evikalm{} Hyperparameter Tuning}
\label{app:tuning}

\noindent The hyperparameter grid is:
\begin{itemize}[leftmargin=1.5em, topsep=2pt, itemsep=0pt]
  \item $C \in \{0.1,\, 0.5,\, 1.0,\, 2.0,\, 5.0,\, 10.0,\, 20.0,\, 50.0,\, 100.0,\, 200.0\}$
  \item $\sigma_\mathrm{gate} \in \{0.0,\, 0.5,\, 1.0,\, 2.0,\, 3.0\}$
        ($\sigma_\mathrm{gate} = 0.0$ disables the outlier filter)
  \item $K = 5$ fixed (ablation over $K$ in Table~\ref{tab:ablation_K} below)
\end{itemize}

\noindent Tuning is performed on the validation set of seed 0.
To make the grid efficient, the attentive model is run once to produce the test set
and validation set predictions and neighbor lists;
the grid then sweeps over $(C, \sigma_\mathrm{gate})$ without re-running the model.
Wall-clock time for the full 50-point grid (including neighbor precomputation) is
under 3 minutes for all datasets except QM9, where we subsample 2000 validation points
and the sweep takes approximately 7 minutes. The same tuned hyperparameters are applied to all 5 seeds; we do not re-tune per seed.
This is conservative: per-seed tuning would likely show modestly better results
but introduces the risk of overfitting to seed-specific validation fluctuations. Table~\ref{tab:hyperparams} reports the selected hyperparameters per dataset.

\begin{table}[h]
\centering
\caption{Selected hyperparameters per dataset (tuned on seed-0 validation set via grid search),
         ordered by QSAR smoothness $s$.
         Kalman = sequential \evikal{}; GP = \evikalm{}.
         $\sigma_\mathrm{gate} = 0.0$ means no outlier filtering.}
\label{tab:hyperparams}
\small
\setlength{\tabcolsep}{5pt}
\begin{tabular}{lcccc}
\toprule
Dataset & \multicolumn{2}{c}{Sequential \evikal{}} & \multicolumn{2}{c}{\evikalm{}} \\
\cmidrule(lr){2-3} \cmidrule(lr){4-5}
        & $C$ & $\sigma_\mathrm{gate}$ & $C$ & $\sigma_\mathrm{gate}$ \\
\midrule
QM8    & 10.0   & 0.5 & 10.0  & 0.5 \\
QM9    & 0.1    & 3.0 & 0.1   & 3.0 \\
pKa    & 1.0    & 2.0 & 0.5   & 0.0 \\
HOPV   & 10.0   & 1.0 & 10.0  & 1.0 \\
BACE   & 1.0    & 3.0 & 1.0   & 3.0 \\
CDK2   & 0.5    & 0.0 & 0.1   & 0.0 \\
5-HT2A & 1.0    & 0.0 & 0.1   & 0.0 \\
D2     & 0.5    & 0.0 & 0.1   & 0.0 \\
hERG   & 1.0    & 0.0 & 0.1   & 0.0 \\
LD50   & 0.1    & 0.0 & 0.1   & 0.0 \\
ESOL   & 5.0    & 1.0 & 0.1   & 2.0 \\
Lipo.  & 2.0    & 0.5 & 2.0   & 1.0 \\
Thermosol & 5.0 & 3.0 & 2.0   & 3.0 \\
FreeSolv  & 200.0 & 0.0 & 50.0 & 0.5 \\
NCI-60 & 20.0   & 2.0 & 20.0  & 2.0 \\
QM7    & 200.0  & 0.5 & 200.0 & 0.5 \\
\bottomrule
\end{tabular}
\end{table}

\paragraph{Interpretation.}
All sixteen datasets were tuned via the same 50-point grid search on the seed-0 validation set.
The selected hyperparameters follow interpretable patterns.
High $C$ on FreeSolv (200 for sequential, 50 for GP) and QM7 ($C=200$) indicates
heavy penalization of dissimilar neighbors---effectively deweighting all but
the most similar molecules.
FreeSolv has sparse structural coverage (median top-1 Tanimoto similarity $= 0.20$);
QM7 has borderline landscape smoothness ($s = 0.834$), so distant neighbors are unreliable
in both cases.
NCI-60 also selects high $C = 20$ with innovation gate $\sigma_\mathrm{gate} = 2.0$.
The mechanistically heterogeneous cytotoxicity landscape benefits from conservative
neighbor incorporation and moderate outlier rejection.
At the other extreme, smooth GPCR and kinase datasets (CDK2, D2, hERG, 5-HT2A)
select very low $C = 0.1$ for \evikalm{}. The highly smooth QSAR landscape means
even structurally distant neighbors carry reliable label information.
LD50 similarly selects $C = 0.1$. Despite endpoint mechanistic diversity, the dataset
is dominated by drug-like congeneric series with a smooth QSAR.
QM8 selects $C = 10.0$ with modest outlier filtering ($\sigma_\mathrm{gate} = 0.5$).
Even though QM8 is smooth ($s = 0.136$), the accuracy ceiling ($R^2 \approx 0.977$)
means the property corrections adds noise; high $C$ minimizes neighbor influence,
limiting degradation to $+0.5\%$.
Thermosol and BACE select $\sigma_\mathrm{gate} = 3.0$.
Thermosol sits in the borderline zone ($s = 0.712$) where occasional distant neighbors
produce noisy updates, and BACE benefits from rejecting only extreme label outliers
while exploiting the high intra-dataset Tanimoto similarity.
The gate on sequential pKa ($\sigma_\mathrm{gate} = 2.0$) provides modest filtering
against tautomer and ionization-state ambiguities in the ChEMBL data.
QM9 selects $C = 0.1$---the minimum---because
the tuner finds that aggressive neighbor incorporation produces negligible
improvement when the accuracy ceiling binds ($R^2 \approx 0.977$, $s = 0.136$);
the low $C$ minimally perturbs the already near-perfect prior.

\subsection{PropDist Training Details}
\label{app:propdist_hparams}

\noindent PropDist is a symmetric MLP trained to predict the absolute property gap $|y_i - y_j|$
between molecule pairs drawn from the training set.
The inputs are bitwise AND and XOR of the two ECFP4 fingerprints (every 2048 bits),
concatenated to a 4096-dimensional vector.
Table~\ref{tab:propdist_arch} summarizes the architecture and training configuration.
Pairs are sampled uniformly from the training set without replacement; pairs from seeds 0--3
are pooled into a single combined dataset of 800k pairs and shuffled for training.
The model that achieves the lowest training MSE across all 80 epochs is saved; this is the checkpoint used at inference.
No validation set is used for PropDist training (since the evidential training labels
fully define the pair targets, there is no risk of label leakage from the test set).

\begin{table}[h]
\centering
\caption{PropDist architecture and training configuration (identical across all 16 datasets).}
\label{tab:propdist_arch}
\small
\begin{tabular}{ll}
\toprule
Hyperparameter & Value \\
\midrule
Input features & 4096 (ECFP4 AND $\|$ XOR) \\
Hidden layers & 256 -- 128 -- 64 \\
Normalization & LayerNorm after each linear layer \\
Activation & ReLU \\
Dropout & 0.2 (after each hidden layer) \\
Output & Scalar $\hat{d} \geq 0$ (absolute property gap, z-scored units) \\
\midrule
Loss & MSE on $|y_i - y_j|$ \\
Optimizer & Adam ($\mathrm{lr} = 10^{-3}$, weight decay $= 10^{-4}$) \\
LR schedule & Cosine annealing ($T_{\max} = 80$ epochs) \\
Gradient clip & Max norm 1.0 \\
Epochs & 80 \\
Training pairs & 200{,}000 per seed $\times$ 4 seeds $= 800{,}000$ total \\
Batch size & 512 \\
\midrule
Parameters & $\approx 361{,}000$ \\
\bottomrule
\end{tabular}
\end{table}

At inference, PropDist is applied to the top-500 Tanimoto candidates (prescreen step),
scoring each candidate's expected property distance to the query.
The PG similarity score $\mathrm{sim}^{\mathrm{PG}}(q, x_k) = \mathrm{Tanimoto}(q, x_k) \cdot \exp(-\hat{d}_{qk})$
is then used to select the top-$K$ neighbors.
PropDist inference over 500 candidates requires a single GPU-batched forward pass
($\approx 0.4$~ms on an RTX 3090), small relative to the evidential forward pass that dominates inference cost (Table~\ref{tab:runtime}).
\noindent \evikalm{} performance is largely insensitive to $K$ above 5. The GP posterior correctly
deweights distant Tanimoto neighbors, so adding more marginally-similar molecules does not help.
We use $K = 5$ as the \evikalm{} default that balances coverage and $K \times K$ matrix inversion cost.

\subsection{Ablation: Number of Neighbors $K$ (\evikalm{})}
\label{app:ablation_K}

\noindent Table~\ref{tab:ablation_K} reports \evikalm{} RMSE on BACE and ESOL for $K \in \{1, 5, 10, 20, 50\}$ neighbors,
with other hyperparameters fixed at the values in Table~\ref{tab:hyperparams}.
Results are from seed 0.

\begin{table}[h]
\centering
\caption{Ablation over number of neighbors $K$ for \evikalm{} (seed 0, Tanimoto neighbor selection).
         Evidential baseline: BACE 0.549, ESOL 0.512.}
\label{tab:ablation_K}
\small
\begin{tabular}{lcc}
\toprule
$K$ & BACE RMSE & ESOL RMSE \\
\midrule
1   & 0.527 & 0.503 \\
5   & 0.511 & 0.500 \\
10  & 0.511 & 0.500 \\
20  & 0.512 & 0.501 \\
50  & 0.514 & 0.502 \\
\bottomrule
\end{tabular}
\end{table}

\textbf{\evikalmpg{} uses $K = 50$.}
PropDist neighbor selection changes the $K$-sensitivity picture qualitatively.
With Tanimoto-only selection, neighbors 6--50 are structurally distant and carry little signal
for the target property; the GP posterior deweights them, and adding more does not help.
With PropDist re-ranking, neighbors 6--50 are property-relevant---they may differ
structurally from the query but share the features that predict the target.
At $K=5$, \evikalmpg{} reduces \evikalm{} RMSE by a median of 5.3\% across all sixteen datasets
(13 of 16; ESOL tied, FreeSolv and NCI-60 degrade); at $K = 50$, the same PropDist scores a larger pool and the median improvement grows
to 16.3\% (15 of 16).
Table~\ref{tab:pg_K_comparison} (Appendix~\ref{app:pg_K_comparison}) shows the full K=5 vs K=50
comparison across all datasets.
The $50 \times 50$ GP solve adds negligible overhead because the PropDist prescreen
(top-500 Tanimoto candidates, then top-$K$ by PropDist score) is the computational bottleneck.

\subsection{Ablation: Evidence Regularization $\lambda$}

\noindent The QM9 $\lambda = 0.05$ experiment (Table~\ref{tab:ablation_lambda}) shows that $\lambda$
controls the calibration of the evidential prior, which in turn affects \evikalm{}.
Table~\ref{tab:ablation_lambda} compares two values on CDK2 and QM9.
The CDK2 result confirms that calibration improvement does not come at the cost of
RMSE. Both $\lambda$ values produce essentially the same \evikalm{} accuracy gain.
The QM9 result confirms the accuracy-ceiling diagnostic: improving calibration alone
(PICP@90\%: $0.550 \to 0.690$) does not rescue \evikalm{} accuracy gains on QM9.
QM9's failure mode is the accuracy ceiling. With $s = 0.136$ the landscape is smooth,
but 100k training molecules drive $R^2 \approx 0.977$ and a small mean epistemic variance ($\bar{u}_e \approx 0.006\,\sigma_{\mathrm{train}}^2$), leaving no room for refinement.
Higher regularization ($\lambda=0.05$) worsens the evidential baseline RMSE, so \evikalm{}
yields $+0.7\%$ rather than $-0.6\%$. Base model saturation is the binding constraint,
not calibration or landscape roughness.

\begin{table}[h]
\centering
\caption{Effect of evidence regularization $\lambda$ on evidential calibration
         and \evikalm{} RMSE. \evikalm{} is evaluated with the hyperparameters
         selected for each $\lambda$ variant separately. $\Delta{RMSE}$ refers to \evikalm{} accuracy change relative to the evidential baseline.}
\label{tab:ablation_lambda}
\small
\setlength{\tabcolsep}{5pt}
\begin{tabular}{llccc}
\toprule
Dataset & $\lambda$ & Evid.\ PICP@90\% & \evikalm{} RMSE & $\Delta$RMSE \\
\midrule
CDK2 & 0.01  & 0.530 & 0.584 & $-7.9\%$ \\
CDK2 & 0.05  & 0.806 & 0.613 & $-7.5\%$ \\
\midrule
QM9  & 0.01  & 0.550 & 0.188 & $-0.6\%$ \\
QM9  & 0.05  & 0.690 & 0.193 & $+0.7\%$ \\
\bottomrule
\end{tabular}
\end{table}

\section{Additional Experimental Results}
\label{app:additional}


\subsection{Scaffold Split Results for CDK2 and hERG}
\label{app:scaffold}

\noindent Table~\ref{tab:main} uses random splits for CDK2 and hERG and scaffold splits for the MoleculeNet datasets.
A concern is that random splits inflate train--test Tanimoto similarity,
artificially favoring \evikalm{} neighbor retrieval.
Table~\ref{tab:scaffold} shows results on scaffold splits for CDK2 and hERG,
together with the median top-1 Tanimoto similarity between test and training molecules
under each split protocol.
Both split conditions use $K=5$ neighbors throughout.

\begin{table}[h]
\centering
\caption{\evikalm{} on scaffold vs.\ random splits for CDK2 and hERG (5 seeds, mean $\pm$ std).
         Median top-1 Tanimoto (seed 0) quantifies how much neighbor similarity is lost by scaffolding.
         \evikalm{} still outperforms evidential on both datasets despite lower test-to-train similarity.}
\label{tab:scaffold}
\small
\setlength{\tabcolsep}{4pt}
\begin{tabular}{llccccc}
\toprule
Dataset & Method & Med.\ Tan. & RMSE $\downarrow$ & $\pm$std & PICP@90\% $\uparrow$ & ECE $\downarrow$ \\
\midrule
\multirow{2}{*}{CDK2 (random)}
 & \evid{}    & 0.765 & 0.634 & 0.025 & 0.530 & 0.225 \\
 & \evikalm{} & 0.765 & \textbf{0.584} & 0.022 & 0.490 & 0.253 \\
\midrule
\multirow{2}{*}{CDK2 (scaffold)}
 & \evid{}    & 0.629 & 0.852 & 0.019 & 0.508 & 0.237 \\
 & \evikalm{} & 0.629 & \textbf{0.796} & 0.025 & 0.516 & 0.247 \\
\midrule
\multirow{2}{*}{hERG (random)}
 & \evid{}    & 0.784 & 0.628 & 0.026 & 0.520 & 0.234 \\
 & \evikalm{} & 0.784 & \textbf{0.593} & 0.013 & 0.411 & 0.296 \\
\midrule
\multirow{2}{*}{hERG (scaffold)}
 & \evid{}    & 0.658 & 0.821 & 0.018 & 0.633 & 0.174 \\
 & \evikalm{} & 0.658 & \textbf{0.758} & 0.012 & 0.567 & 0.219 \\
\bottomrule
\end{tabular}
\end{table}

\noindent Scaffold splitting reduces median top-1 Tanimoto similarity from 0.765 to 0.629 for CDK2
and from 0.784 to 0.658 for hERG, confirming that test molecules are genuinely less similar
to their training neighbors under the scaffold protocol.

The CDK2 RMSE improvement falls from $-7.9\%$ (random) to $-6.6\%$ (scaffold), consistent
with the expectation that lower neighbor similarity reduces \evikalm{} gains.
The hERG result is reversed. Scaffold improvement ($-7.7\%$) exceeds random ($-5.5\%$).
The mechanistic explanation is that scaffold splitting hurts the evidential baseline more
than it hurts \evikalm{}.
Under scaffold splits, the neural network must generalize across scaffold cores it has
not seen---a harder task that increases absolute RMSE from 0.628 to 0.821.
\evikalm{} partially compensates by retrieving training neighbors that share local
functional group patterns with the test molecule, even when the scaffold cores differ.
ECFP4 fingerprints capture circular substructures at radius 2, so two molecules can
share a high Tanimoto score even with different Bemis--Murcko scaffolds, as long as
their R-groups or heterocyclic side chains overlap.
The larger absolute \evikalm{} correction on scaffold hERG (0.821 to 0.758, $\Delta = -0.063$)
versus random hERG ($\Delta = -0.035$) reflects this. When the base model is more uncertain
(higher evidential RMSE), \evikalm{}'s test-time neighborhood correction has more room to help.


\subsection{Deep Ensemble Comparison}
\label{app:ensemble}

\noindent Deep ensembles~\citep{lakshminarayanan2017simple} of $M=5$ independently trained
MSE models (same backbone, different random initializations and data shuffle) provide
a strong uncertainty baseline that requires $5\times$ training compute and $5\times$
inference cost.
Table~\ref{tab:ensemble} reports results across the eight assessed datasets.

\begin{table}[h]
\centering
\caption{Deep Ensemble ($M=5$) vs.\ \evikalm{} (5 data seeds, mean $\pm$ std; ordered
         by QSAR smoothness $s$, smoothest first).
         Ensembles can match \evikalm{} RMSE on large smooth datasets but are unreliable
         on small datasets (HOPV: $+69\%$ vs.\ evidential) and sensitive to unlucky
         initializations (pKa seed 0: RMSE=1.21 vs.\ mean 0.53 for remaining seeds).}
\label{tab:ensemble}
\small
\setlength{\tabcolsep}{3pt}
\begin{tabular}{lllcccc}
\toprule
Dataset & $s$ & Method & RMSE $\downarrow$ & $\pm$std & PICP@90\% $\uparrow$ & ECE $\downarrow$ \\
\midrule
\multirow{3}{*}{pKa}    & \multirow{3}{*}{0.165}
 & Deep Ens.  & 0.692 & 0.263 & 0.843 & 0.048 \\
 & & \evid{}    & 0.588 & 0.049 & 0.408 & 0.288 \\
 & & \evikalm{} & \textbf{0.568} & 0.039 & 0.393 & 0.294 \\
\midrule
\multirow{3}{*}{HOPV}   & \multirow{3}{*}{0.243}
 & Deep Ens.  & 1.437 & 0.232 & 0.335 & 0.344 \\
 & & \evid{}    & 0.849 & 0.080 & 0.794 & 0.076 \\
 & & \evikalm{} & \textbf{0.814} & 0.089 & 0.771 & 0.079 \\
\midrule
\multirow{3}{*}{BACE}   & \multirow{3}{*}{0.298}
 & Deep Ens.  & 0.515 & 0.083 & 0.482 & 0.255 \\
 & & \evid{}    & 0.549 & 0.046 & 0.732 & 0.104 \\
 & & \evikalm{} & \textbf{0.511} & 0.029 & 0.661 & 0.128 \\
\midrule
\multirow{3}{*}{CDK2}   & \multirow{3}{*}{0.335}
 & Deep Ens.  & \textbf{0.520} & 0.121 & 0.589 & 0.190 \\
 & & \evid{}    & 0.634 & 0.025 & 0.530 & 0.225 \\
 & & \evikalm{} & 0.584 & 0.022 & 0.490 & 0.253 \\
\midrule
\multirow{3}{*}{hERG}   & \multirow{3}{*}{0.379}
 & Deep Ens.  & \textbf{0.447} & 0.047 & 0.543 & 0.225 \\
 & & \evid{}    & 0.628 & 0.026 & 0.520 & 0.234 \\
 & & \evikalm{} & 0.593 & 0.013 & 0.411 & 0.296 \\
\midrule
\multirow{3}{*}{ESOL}   & \multirow{3}{*}{0.540}
 & Deep Ens.  & \textbf{0.429} & 0.042 & 0.359 & 0.324 \\
 & & \evid{}    & 0.512 & 0.016 & 0.812 & 0.065 \\
 & & \evikalm{} & 0.500 & 0.014 & 0.814 & 0.076 \\
\midrule
\multirow{3}{*}{Lipo.}  & \multirow{3}{*}{0.541}
 & Deep Ens.  & \textbf{0.399} & 0.077 & 0.656 & 0.164 \\
 & & \evid{}    & 0.526 & 0.011 & 0.671 & 0.144 \\
 & & \evikalm{} & 0.524 & 0.010 & 0.648 & 0.156 \\
\midrule
\multirow{3}{*}{FreeSolv} & \multirow{3}{*}{0.756}
 & Deep Ens.  & 0.711 & 0.178 & 0.656 & 0.158 \\
 & & \evid{}    & 0.625 & 0.044 & 0.726 & 0.101 \\
 & & \evikalm{} & \textbf{0.656} & 0.052 & 0.723 & 0.104 \\
\bottomrule
\end{tabular}
\end{table}

Four patterns emerge from the full comparison.
First, ensembles improve RMSE substantially on large smooth datasets (CDK2: $-18\%$,
hERG: $-29\%$, ESOL: $-16\%$, Lipo: $-24\%$ vs.\ evidential) at $5\times$ training
and inference cost. Second, ensembles fail catastrophically on the smallest dataset: HOPV ($N_\text{train} = 272$) sees ensemble RMSE increase by $+69\%$ vs.\ the evidential baseline, while \evikalm{} reduces it by $-4.0\%$. With a training set of only 272 molecules, the five ensemble members cannot be reliably trained--- the variance across seeds ($\pm$0.232) exceeds the target RMSE itself. Third, even a large training set provides no guarantee of ensemble stability: pKa ($N_\text{train} = 6996$, $s = 0.165$) produces a catastrophic outlier on seed 0 (RMSE = 1.21 vs.\ 0.49--0.66 for the remaining four seeds), yielding a mean RMSE of 0.692 $\pm$ 0.263 that is worse than the single-model evidential baseline (0.588).

\evikalm{} achieves 0.568 $\pm$ 0.039 on pKa---more accurate and more stable. The pKa failure mode (unlucky initialization causing a collapsed local minimum) differs from the HOPV failure mode (insufficient data), but both produce unreliable ensemble behavior that \evikalm{} avoids by construction. Fourth, and most importantly for drug discovery applications, all deep ensemble configurations are poorly calibrated (PICP@90\% = 0.34--0.84 vs.\ target 0.90). pKa is the sole dataset where the ensemble approaches the calibration target (PICP@90\% = 0.843 vs.\ target 0.90). The high seed-to-seed variance (RMSE std = 0.263) suggests the ensemble mean is partially driven by the outlier seed's large errors, which may inflate the effective ensemble predictive variance toward calibrated coverage; per-seed calibration analysis would be needed to confirm this mechanism. The evidential baseline (PICP@90\% = 0.41--0.81 on smooth datasets) and \evikalm{}'s maintained calibration provide more reliable uncertainty estimates despite requiring only a single trained model.


\subsection{Out-of-Distribution Detection}
\label{app:ood}

\noindent We evaluate the ability of uncertainty measures to detect structurally novel (OOD)
molecules.
A model is trained on the random-split training set; the random test set serves as
in-distribution and the scaffold test set (molecules on structurally distinct scaffolds)
as the OOD set.
We compute AUROC of each uncertainty measure as a binary OOD detector
(0 = in-distribution, 1 = OOD).

\begin{table}[h]
\centering
\caption{OOD detection AUROC (random test = in-distribution; scaffold test = OOD).
         \mcd{} variance is the strongest OOD detector on BACE; epistemic uncertainty
         and GP posterior variance are more consistent across datasets.}
\label{tab:ood}
\small
\begin{tabular}{lcccc}
\toprule
Dataset & $u_e$ (Evid.) & $u_\text{total}$ & GP var. & MCD var. \\
\midrule
BACE & 0.445 & 0.452 & 0.414 & \textbf{0.675} \\
CDK2 & \textbf{0.603} & 0.599 & 0.536 & 0.551 \\
\bottomrule
\end{tabular}
\end{table}

The BACE result ($u_e = 0.445$) indicates that the Normal-Inverse-Gamma evidential model assigns
lower epistemic uncertainty to scaffold-OOD molecules than to in-distribution ones.
This reflects the diversity of the random-split training set: congeneric series with high
within-distribution potency variation make the model genuinely uncertain about many
in-distribution molecules, while scaffold-OOD molecules occupy internally coherent
families the Normal-Inverse-Gamma evidential regularizer assigns only moderate uncertainty.
The CDK2 result ($u_e = 0.603$) is the more typical case: epistemic uncertainty is higher
for cross-scaffold molecules, producing above-chance OOD detection.
\mcd{}'s strong AUROC on BACE (0.675) comes from sharper activation boundaries
at unseen scaffolds, but at $50\times$ inference cost (Section~\ref{app:runtime}).


\subsection{Computational Cost}
\label{app:runtime}

\noindent Table~\ref{tab:runtime} reports wall-clock inference time measured on an NVIDIA RTX 3090
(median over 50 repetitions, 10 warmup; BACE random-split training set with 1210 training
molecules).

\begin{table}[h]
\centering
\caption{Inference time per query molecule (RTX 3090, median over 50 reps).
         \evikalm{} and \evikalmpg{} add only a ${\approx}3\times$ overhead vs.\ a single evidential forward pass,
         compared to $49.8\times$ for \mcd{} with $T=50$ passes.}
\label{tab:runtime}
\small
\begin{tabular}{lcc}
\toprule
Method & Time (ms/mol) & Relative overhead \\
\midrule
Evidential (batch=1)   & 2.3  & 1.0$\times$ \\
Evidential (batch=64)  & 0.04 & ---          \\
\mcd{} ($T=50$, batch=1)  & 115 & 49.8$\times$ \\
\mcd{} ($T=50$, batch=64) & 1.9 & ---          \\
Tanimoto lookup ($K=5$)       & 0.8 & ---          \\
GP solve ($K=5$)              & 0.02 & ---         \\
PropDist prescreen ($500$ cand.) & 0.4 & ---       \\
GP solve ($K=50$)             & 0.1  & ---          \\
\textbf{\evikalm{} ($K=5$, batch=1)}    & \textbf{6.9} & \textbf{3.0$\times$} \\
\textbf{\evikalmpg{} ($K=50$, batch=1)} & \textbf{7.4} & \textbf{3.2$\times$} \\
\midrule
Deep Ensemble ($M=5$) & $5\times$ (train) & 5.0$\times$ (infer.) \\
\bottomrule
\end{tabular}
\end{table}

\noindent \evikalm{}'s $3\times$ overhead over the evidential baseline (6.9 vs.\ 2.3 ms) decomposes into a single evidential forward pass (2.3 ms), the ECFP4 Tanimoto lookup against the training molecules (0.8 ms), and the $5\times5$ GP posterior solve (0.02 ms), with the small remainder from query fingerprint generation. \evikalmpg{} reuses this entire pipeline and adds only two steps, a single PropDist forward pass over the top-500 Tanimoto prescreen candidates ($\approx0.4$ ms) and a larger $50\times50$ GP assembly and solve ($\approx0.1$ ms). Property-guided selection therefore costs roughly half a millisecond more than \evikalm{} per query, a total of $\approx7.4$ ms ($\approx3\times$ the forward pass). Both variants stay far below \mcd{}'s $49.8\times$ and, unlike a deep ensemble, require only a single trained model.

\noindent \textbf{Scaling to large reference sets.} Per query, every step above is constant in the dataset size except the neighbor lookup, which is $\mathcal{O}(N_\text{train})$ because it scans the query fingerprint against all training fingerprints in one dense GPU operation. This scan is cheap in practice. On the RTX 3090 the lookup stays near $1$ ms per query for training sets up to $10^5$ molecules and reaches only $\approx10$ ms at $10^6$, so a reference set a hundred times larger than any benchmark here adds only a few milliseconds. Batching queries collapses the scan into a single matrix product and amortizes it below $0.05$ ms per molecule. Concretely, scoring a $100{,}000$ molecule screening library with \evikalmpg{} takes about twelve minutes processed one molecule at a time, or under a minute with batched inference, on a single GPU. Beyond $N_\text{train}\sim10^5$ molecules, approximate nearest-neighbor indexing (e.g.\ FAISS, \citealt{johnson2019billion}) feasibly reduces the neighbor lookup from linear to sublinear algorithmic time complexity.


\subsection{Method Comparison: All Sixteen Datasets}
\label{app:method_comparison_full}

\noindent Figure~\ref{fig:method_comparison_full} extends the six-panel summary in the main text
(Figure~\ref{fig:method_comparison}) to all sixteen datasets, ordered by QSAR smoothness $s$.

\begin{figure*}[h]
\centering
\includegraphics[width=\linewidth]{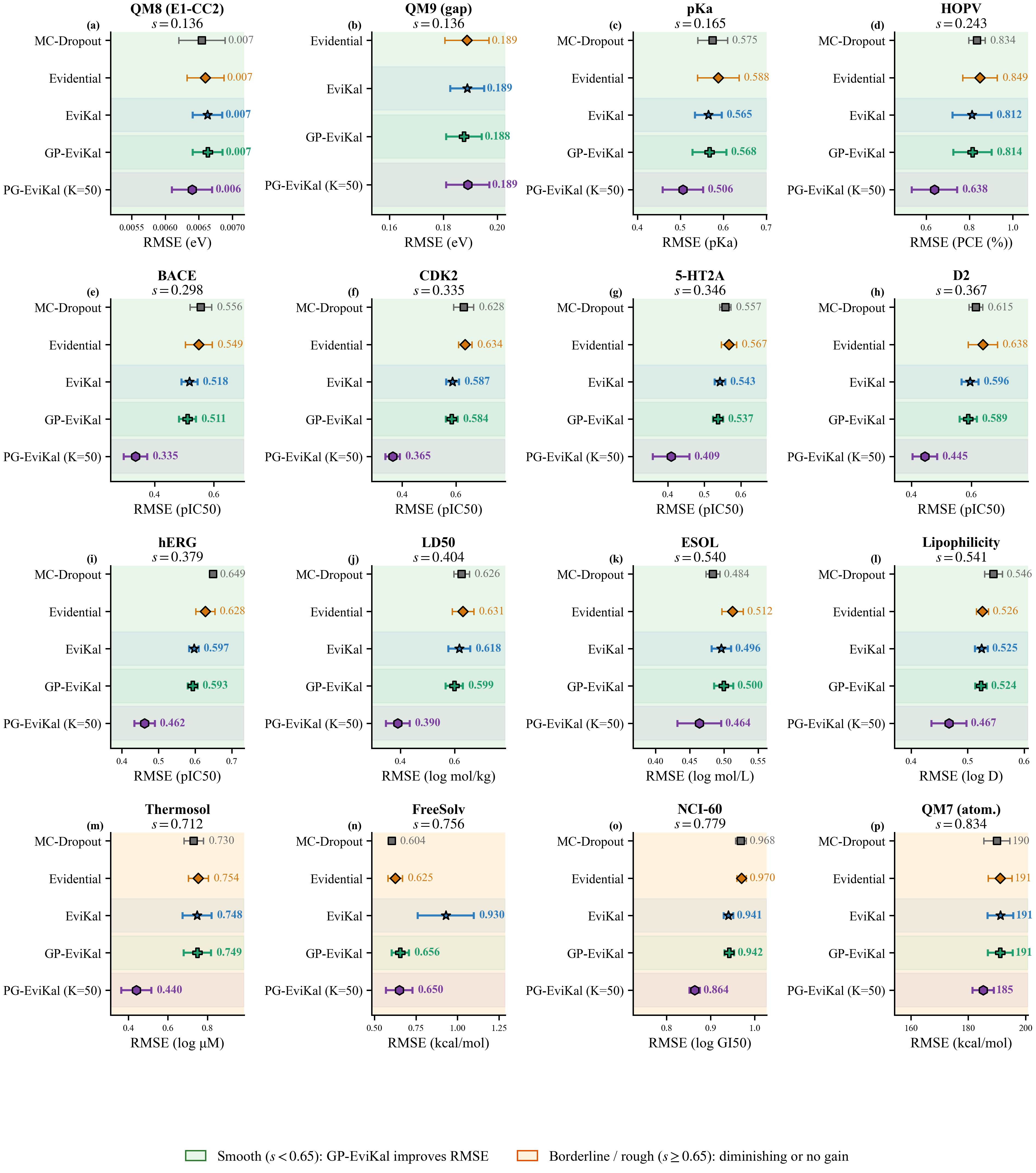}
\caption{
  \textbf{Method comparison across all sixteen datasets} (RMSE, 5 seeds),
  in order of increasing QSAR smoothness $s$.
  On smooth datasets ($s < 0.65$), \evikalmpg{} (purple) achieves the lowest RMSE,
  with \evikalm{} (green) giving smaller but consistent reductions relative to the
  evidential baseline.
  On the borderline datasets Thermosol ($s=0.712$) and NCI-60 ($s=0.779$),
  \evikalm{} still improves modestly ($-0.7\%$, $-2.9\%$),
  and on the roughest dataset QM7 ($s=0.834$) it is essentially unchanged ($+0.05\%$).
  The two quantum-chemistry datasets QM8 and QM9 (both $s=0.136$) are the exception among
  smooth landscapes. Their base evidential models already sit at a representation-imposed
  accuracy ceiling leaving little room for refinement (\evikalm{}: $+0.5\%$ on QM8, $-0.6\%$ on QM9).
  \mcd{} (gray) and \evid{} (orange) are the non-fusion baselines.
}
\label{fig:method_comparison_full}
\end{figure*}


\subsection{Online \evikalmpg{}: All Thirteen Datasets}
\label{app:online_kalman}

\noindent Figure~\ref{fig:online_kalman} shows the full thirteen-dataset online \evikalmpg{} experiment
described in Section~\ref{sec:results_online}.

\begin{figure*}[h]
  \centering
  \includegraphics[width=\linewidth]{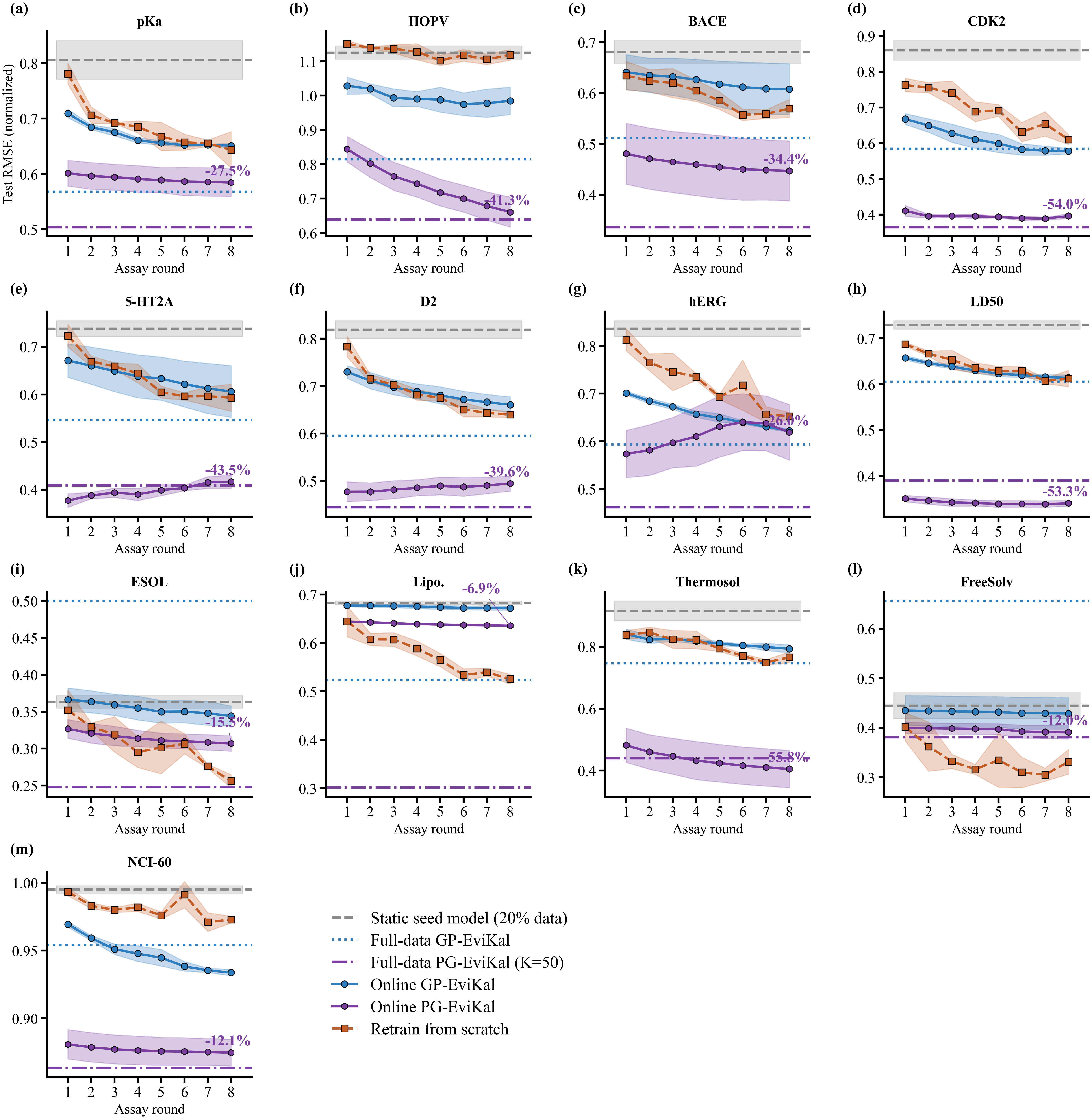}
  \caption{
    \textbf{Online \evikalm{} and \evikalmpg{}: sequential assay incorporation without retraining, all thirteen datasets}
    (QM7, QM8, and QM9 excluded as non-sequential assay scenarios).
    All percentages are relative to a seed model trained on 20\% of the data.
    Dashed gray = static seed model; orange squares = retrain-from-scratch oracle
    (patience 15 epochs, 3 simulations); dotted teal = full-data \evikalm{} reference;
    dash-dot purple = full-data \evikalmpg{} reference;
    blue circles = online \evikalm{}; purple hexagons = online \evikalmpg{}.
    Annotated percentage = \evikalmpg{} gain at round~8 vs.\ seed model.
    On smooth datasets ($s < 0.65$), \evikalmpg{} substantially exceeds both \evikalm{} and the retrain oracle;
    on near-threshold datasets (ESOL, Lipo., FreeSolv, NCI-60), retraining wins decisively
    because additional data improves encoder representations that neighbor fusion cannot replicate.
  }
  \label{fig:online_kalman}
\end{figure*}


\subsection{\evikalm{} vs.\ Sequential \evikal{}: Full Comparison}
\label{app:robustness}

\noindent Figure~\ref{fig:robustness} shows the RMSE change relative to the evidential baseline for
both \evikalm{} and sequential \evikal{} across all sixteen datasets.

\begin{figure*}[h]
\centering
\includegraphics[width=\linewidth]{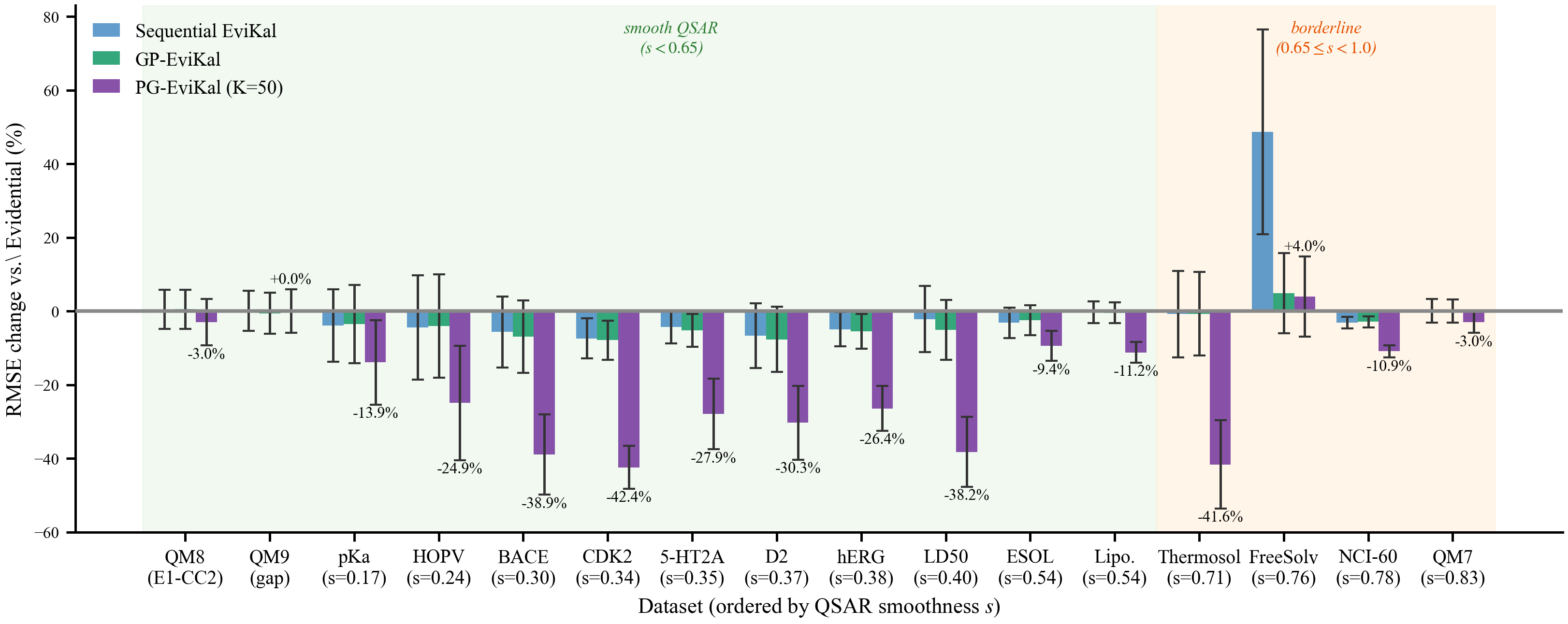}
\caption{
  \textbf{Sequential \evikal{} vs.\ \evikalm{} vs.\ \evikalmpg{}: RMSE change relative to evidential baseline
  across all sixteen datasets ordered by QSAR smoothness $s$.}
  Green background = smooth ($s < 0.65$); light-orange = borderline ($0.65 \leq s < 1.0$).
  Negative bars = improvement; positive = degradation. Error bars show propagated seed variation
  ($\pm1\sigma$).
  Sequential \evikal{} (blue) fails catastrophically on FreeSolv ($+48.7\%$)
  because correlated neighbors are treated as independent measurements;
  \evikalm{} (green) reduces this to $+4.9\%$ by accounting for neighbor correlations in the batch
  posterior.
  \evikalmpg{} (purple) replaces Tanimoto neighbor selection with PropDist-supervised similarity
  and achieves the largest gains on smooth datasets: CDK2 $-42.4\%$, BACE $-38.9\%$, LD50 $-38.2\%$.
  QM8 ($s=0.136$, $-3.0\%$) and QM9 ($s=0.136$, $0.0\%$) show negligible or noise-level change
  due to the accuracy ceiling ($R^2\approx0.977$, Appendix~\ref{app:datasets}); Thermosol ($s=0.712$) and NCI-60 ($s=0.779$) lie in the borderline zone yet still yield consistent gains (\evikalm{}: $-0.7\%$/$-2.9\%$; \evikalmpg{}: $-41.6\%$/$-10.9\%$) owing to adequate structural coverage.
  QM7 ($s=0.834$) shows near-zero \evikalm{} change ($+0.05\%$) and only a marginal \evikalmpg{} gain
  ($-3.0\%$), consistent with borderline smoothness.
}
\label{fig:robustness}
\end{figure*}


\subsection{\evikalmpg{} Calibration: Three-Way PICP Comparison}
\label{app:pg_picp}

\noindent \evikalm{} compresses predictive intervals by construction ($\sigma^2_\mathrm{post} < P_0$),
so PICP@90\% decreases relative to the evidential baseline on most datasets.
Table~\ref{tab:pg_picp} shows how \evikalmpg{} (K=50) changes this picture relative to
both the evidential baseline and Tanimoto \evikalm{} (K=50).

\begin{table}[h]
\centering
\caption{PICP@90\% across all 16 datasets: Evidential, Tanimoto \evikalm{} ($K=50$),
    and \evikalmpg{} ($K=50$).
    $\Delta$ columns show change vs.\ the evidential baseline.
    Datasets ordered by QSAR smoothness $s$ (smoothest first).
    Tanimoto K=50 compresses PICP on 15 of 16 datasets; \evikalmpg{} K=50 recovers
    and raises PICP@90\% above the evidential baseline on 14 of 16 datasets.
    \textdagger{} ESOL, Lipo., and FreeSolv rows use random-split test sets
    (the \evikalmpg{} K=50 evaluation pipeline does not apply scaffold splits for these datasets);
    Evidential PICP@90\% values therefore differ from Table~\ref{tab:main}.}
\label{tab:pg_picp}
\small
\setlength{\tabcolsep}{4pt}
\begin{tabular}{lcccccc}
\toprule
Dataset & $s$ & Evidential & \shortstack{\evikalm{}\\($K=50$, Tan.)} & $\Delta_\text{Tan}$ & \shortstack{\evikalmpg{}\\($K=50$)} & $\Delta_\text{PG}$ \\
\midrule
QM8       & 0.136 & 0.585 & 0.584 & $-0.1$\% & 0.586 & $+0.1$\% \\
QM9       & 0.136 & 0.550 & 0.508 & $-4.2$\% & 0.530 & $-2.0$\% \\
pKa       & 0.165 & 0.408 & 0.338 & $-7.0$\% & 0.445 & $+3.7$\% \\
HOPV      & 0.243 & 0.794 & 0.782 & $-1.2$\% & \textbf{0.865} & $+7.1$\% \\
BACE      & 0.298 & 0.732 & 0.641 & $-9.1$\% & \textbf{0.849} & $+11.7$\% \\
CDK2      & 0.335 & 0.530 & 0.448 & $-8.3$\% & \textbf{0.698} & $+16.8$\% \\
5-HT2A    & 0.346 & 0.468 & 0.337 & $-13.1$\% & \textbf{0.564} & $+9.6$\% \\
D2        & 0.367 & 0.642 & 0.479 & $-16.3$\% & \textbf{0.730} & $+8.8$\% \\
hERG      & 0.379 & 0.520 & 0.358 & $-16.2$\% & \textbf{0.618} & $+9.9$\% \\
LD50      & 0.404 & 0.497 & 0.403 & $-9.4$\% & \textbf{0.657} & $+16.0$\% \\
ESOL$^\dagger$      & 0.540 & 0.853 & 0.897 & $+4.4$\% & \textbf{0.940} & $+8.7$\% \\
Lipo.$^\dagger$     & 0.541 & 0.888 & 0.853 & $-3.5$\% & 0.906 & $+1.9$\% \\
Thermosol & 0.712 & 0.736 & 0.598 & $-13.8$\% & \textbf{0.862} & $+12.7$\% \\
FreeSolv$^\dagger$  & 0.756 & 0.637 & 0.591 & $-4.7$\% & 0.634 & $-0.3$\% \\
NCI-60    & 0.779 & 0.893 & 0.836 & $-5.7$\% & 0.917 & $+2.4$\% \\
QM7       & 0.834 & 0.912 & 0.905 & $-0.7$\% & 0.920 & $+0.8$\% \\
\midrule
\multicolumn{3}{l}{\textit{Improves vs.\ Evidential}} & \multicolumn{2}{l}{1 of 16 (ESOL only)} & \multicolumn{2}{l}{14 of 16} \\
\bottomrule
\end{tabular}
\end{table}

\noindent Two patterns persist. First, Tanimoto \evikalm{} at $K=50$ systematically compresses posteriors.
The K=50 neighbor pool provides substantial evidence that shifts the posterior mean
toward neighbor labels while contracting posterior variance, regardless of whether
the neighbors are truly property-informative.
On overconfident priors (pKa, 5-HT2A, D2, hERG: evidential PICP@90\% $\leq 0.52$),
the additional compression from 50 Tanimoto neighbors is severe ($-8$--$-16$~\%).
ESOL is the exception. Its above-average calibration and high structural coverage of neighbors
mean that Tanimoto K=50 actually improves calibration (PICP@90\% $= 0.853 \to 0.897$ on the random-split test set) because 50 similar and property-consistent neighbors sharpen the posterior without miscalibrating it.

Second, \evikalmpg{} K=50 substantially reverses the uncertainty interval compression from Tanimoto-only neighbor selection and exceeds the evidential baseline on 14 of 16 datasets.
The calibration restoration is mechanistically tied to neighbor quality.
PropDist selects neighbors whose labels agree closely with the query's true property.
The posterior mean is pulled toward the truth more precisely, and the posterior variance contracts proportionately.
The result is a posterior whose interval coverage is better realized; the mean
is more accurate, and the variance reflects the true residual uncertainty.
The largest restorations occur on drug-like smooth datasets with overconfident priors:
CDK2 ($0.530 \to 0.448 \to 0.698$), LD50 ($0.497 \to 0.403 \to 0.657$), Thermosol ($0.736 \to 0.598 \to 0.862$),
and BACE ($0.732 \to 0.641 \to 0.849$).
On these datasets, \evikalmpg{} is simultaneously the best RMSE and substantially better calibrated
than the evidential baseline.

FreeSolv and QM9 are the two exceptions where \evikalmpg{} PICP falls slightly below the evidential baseline
($-0.3$~\% and $-2.0$~\% respectively).
The FreeSolv pattern is consistent with the PropDist training-pair scarcity failure.
With only $\approx 513$ training molecules, PropDist cannot reliably learn property distances,
so neighbor selection remains uncertain and the posterior is not well-targeted.
QM9's slight shortfall ($-2.0$~\%) reflects the accuracy ceiling of the model's topological representation (Appendix~\ref{app:datasets}):
the evidential prior is already near-optimal ($R^2 \approx 0.977$),
so any neighbor-based correction---property-guided or not---adds small perturbations.

\subsection{\evikalmpg{}: $K=5$ vs.\ $K=50$ Comparison}
\label{app:pg_K_comparison}

\noindent Table~\ref{tab:pg_K_comparison} compares \evikalmpg{} performance at $K=5$ and $K=50$
across all sixteen datasets.
In contrast to \evikalm{} (which is insensitive to $K$ above 5; Appendix~\ref{app:ablation_K}),
\evikalmpg{} benefits substantially from larger $K$ because PropDist can surface property-relevant
molecules that Tanimoto would rank far outside the top-5.

\begin{table}[H]
\centering
\caption{\evikalmpg{} RMSE at $K=5$ and $K=50$ across all 16 datasets (mean $\pm$ std, 5 seeds).
         $K=50$ improves over $K=5$ on 15 of 16 datasets; QM9 and QM8 are accuracy-ceiling
         cases where both accuracy gains over the evidential baseline are noise-level.}
\label{tab:pg_K_comparison}
\small
\setlength{\tabcolsep}{4pt}
\begin{tabular}{lccccc}
\toprule
Dataset & $s$ & \evikalmpg{} ($K=5$) & \evikalmpg{} ($K=50$) & $\Delta_{K=5 \to 50}$ & vs.\ Evidential \\
\midrule
QM8       & 0.136 & 0.0065 $\pm$ 0.0003 & 0.0064 $\pm$ 0.0003 & $-0.6\%$ & $-3.0\%$ \\
QM9       & 0.136 & 0.187 $\pm$ 0.008   & 0.189 $\pm$ 0.008   & $+0.8\%$ & $0.0\%$ \\
pKa       & 0.165 & 0.559 $\pm$ 0.047   & 0.506 $\pm$ 0.047   & $-9.5\%$ & $-13.9\%$ \\
HOPV      & 0.243 & 0.740 $\pm$ 0.099   & 0.638 $\pm$ 0.105   & $-13.8\%$ & $-24.9\%$ \\
BACE      & 0.298 & 0.406 $\pm$ 0.043   & 0.335 $\pm$ 0.040   & $-17.5\%$ & $-38.9\%$ \\
CDK2      & 0.335 & 0.472 $\pm$ 0.024   & 0.365 $\pm$ 0.027   & $-22.7\%$ & $-42.4\%$ \\
5-HT2A    & 0.346 & 0.474 $\pm$ 0.024   & 0.409 $\pm$ 0.050   & $-13.7\%$ & $-27.9\%$ \\
D2        & 0.367 & 0.515 $\pm$ 0.049   & 0.445 $\pm$ 0.041   & $-13.6\%$ & $-30.3\%$ \\
hERG      & 0.379 & 0.535 $\pm$ 0.018   & 0.462 $\pm$ 0.028   & $-13.6\%$ & $-26.4\%$ \\
LD50      & 0.404 & 0.521 $\pm$ 0.037   & 0.390 $\pm$ 0.044   & $-25.1\%$ & $-38.2\%$ \\
ESOL      & 0.540 & 0.500 $\pm$ 0.039   & 0.464 $\pm$ 0.032   & $-7.2\%$  & $-9.4\%$ \\
Lipo.     & 0.541 & 0.519 $\pm$ 0.053   & 0.467 $\pm$ 0.031   & $-10.0\%$ & $-11.2\%$ \\
Thermosol & 0.712 & 0.618 $\pm$ 0.055   & 0.440 $\pm$ 0.076   & $-28.8\%$ & $-41.6\%$ \\
FreeSolv  & 0.756 & 0.675 $\pm$ 0.099   & 0.650 $\pm$ 0.080   & $-3.7\%$  & $+4.0\%$ \\
NCI-60    & 0.779 & 0.947 $\pm$ 0.011   & 0.864 $\pm$ 0.012   & $-8.8\%$  & $-10.9\%$ \\
QM7       & 0.834 & 190.0 $\pm$ 4.1     & 185.2 $\pm$ 3.7     & $-2.5\%$  & $-3.0\%$ \\
\bottomrule
\end{tabular}
\end{table}

\noindent The $K=5 \to 50$ gains are substantial across all dataset types.
On smooth drug-like datasets, PropDist at $K=50$ surfaces property-relevant neighbors
that Tanimoto would never retrieve within the top-5, producing $K=5 \to K=50$ improvements
of 14--29\% on CDK2, LD50, Thermosol, BACE, HOPV, and D2.
On borderline datasets (Thermosol, NCI-60, QM7), the larger pool allows PropDist to
compensate partially for the weaker QSAR signal. Thermosol improves $-28.8\%$ from
$K=5$ to $K=50$, pushing its vs-evidential gain from $0\%$ to $-41.6\%$.
FreeSolv remains the sole failure. Even at $K=50$, PropDist cannot learn reliable
property distances from $\approx 513$ training molecules, and both K values degrade
slightly vs.\ the \evikalm{} K=50 baseline.
QM9 and QM8 show essentially no sensitivity to $K$ in either direction,
consistent with the accuracy-ceiling failure mode where the prior is already near-optimal (Appendix~\ref{app:datasets}).

\subsection{Is the GP Fusion Necessary? Comparison to Simple Neighbor Averaging}
\label{app:averaging}

\noindent Because \evikalmpg{} ultimately predicts from a set of retrieved neighbors, it is worth asking how much of its accuracy comes from the Gaussian process fusion of Eq.~\eqref{eq:gp_mean}--\eqref{eq:gp_var} rather than from simply averaging the neighbor labels. Table~\ref{tab:averaging} answers this by placing \evikalmpg{} beside two averaging baselines over the same top-$K$ neighbors ($K=50$), one drawn from Tanimoto similarity alone (Tan-avg) and one from the property-guided re-ranking that defines \evikalmpg{} (PG-avg), together with a biased-mean GP that keeps the evidential prediction $\mu_0$ as its prior mean but allows the neighbors to correct it. We develop the biased-mean variant in the second half of this section.

\begin{table}[h]
\centering
\caption{Neighbor fusion versus simple averaging on all sixteen datasets, ordered by QSAR smoothness $s$ and reported as $z$-scored test RMSE averaged over five seeds. \evikalmpg{} is the fixed-mean GP of the main text, while the biased-mean GP anchors its prior mean at the evidential prediction $\mu_0$ and lets the neighbors estimate a shrunk offset, with the offset-prior variance $\tau_\beta^\ast$ and diagonal noise $\tau_n^\ast$ selected on validation (median over seeds) and the resulting posterior coverage in the rightmost column. PG-avg and Tan-avg average the labels of the same top-$K$ neighbors selected with and without property-guided re-ranking. Bold marks the lowest RMSE among \evikalmpg{}, the biased-mean GP, and PG-avg.}
\label{tab:averaging}
\small
\setlength{\tabcolsep}{3pt}
\begin{tabular}{lccccccrrc}
\toprule
Dataset & $s$ & Evid. & \evikalmpg{} & Biased-GP & PG-avg & Tan-avg & $\tau_\beta^\ast$ & $\tau_n^\ast$ & PICP@90 \\
\midrule
QM8       & 0.136 & 0.151 & 0.149 & \textbf{0.136} & 0.171 & 0.466 & 0.1 & 3.0  & 0.989 \\
QM9       & 0.136 & 0.150 & 0.150 & \textbf{0.148} & 0.218 & 0.411 & 0.0 & 0.3  & 0.722 \\
pKa       & 0.165 & 0.588 & 0.504 & 0.267 & \textbf{0.261} & 0.824 & 100 & 3.0  & 0.891 \\
HOPV      & 0.243 & 0.848 & 0.638 & 0.311 & \textbf{0.303} & 0.943 & 100 & 0.0  & 0.988 \\
BACE      & 0.298 & 0.549 & 0.336 & \textbf{0.223} & 0.255 & 0.744 & 3.0 & 0.1  & 0.949 \\
CDK2      & 0.335 & 0.634 & 0.365 & \textbf{0.210} & 0.225 & 0.739 & 100 & 0.3  & 0.910 \\
5-HT2A    & 0.346 & 0.567 & 0.409 & \textbf{0.260} & 0.262 & 0.715 & 100 & 0.3  & 0.821 \\
D2        & 0.367 & 0.638 & 0.444 & \textbf{0.341} & 0.347 & 0.752 & 100 & 0.3  & 0.841 \\
hERG      & 0.379 & 0.628 & 0.462 & \textbf{0.300} & 0.308 & 0.800 & 100 & 0.3  & 0.807 \\
LD50      & 0.404 & 0.631 & 0.390 & \textbf{0.255} & 0.262 & 0.727 & 100 & 0.03 & 0.821 \\
ESOL      & 0.540 & 0.330 & \textbf{0.244} & 0.247 & 0.312 & 0.768 & 0.0 & 0.0  & 0.942 \\
Lipo.     & 0.541 & 0.366 & \textbf{0.298} & 0.301 & 0.334 & 0.829 & 0.1 & 0.0  & 0.905 \\
Thermosol & 0.712 & 0.754 & 0.440 & \textbf{0.149} & 0.162 & 0.906 & 100 & 1.0  & 0.992 \\
FreeSolv  & 0.756 & 0.428 & 0.381 & \textbf{0.365} & 0.543 & 0.726 & 0.3 & 10.0 & 0.916 \\
NCI-60    & 0.779 & 0.970 & 0.863 & 0.570 & \textbf{0.567} & 0.969 & 100 & 30.0 & 0.992 \\
QM7       & 0.834 & 0.856 & 0.829 & 0.282 & \textbf{0.273} & 0.981 & 100 & 0.0  & 1.000 \\
\bottomrule
\end{tabular}
\end{table}

Averaging without property-guided selection fails outright, since Tan-avg trails the evidential baseline on every dataset and shows that the property-distance metric behind \evikalmpg{} neighbor selection is what makes retrieved neighbors trustworthy enough to combine. Once that selection is in place, however, a plain average of the neighbors is a strong predictor and attains a lower RMSE than the full GP posterior on the eleven datasets whose evidential prior is least accurate, and \evikalmpg{} regains the advantage only where the prior is already good, on the five datasets with the smallest evidential error. FreeSolv is the clean counter-case. With roughly $513$ training molecules, PropDist cannot learn reliable property distances, the neighbors are selected poorly, and the PG-avg average ($0.543$) falls behind even the evidential model itself ($0.428$). Yet, the Gaussian process, working from those same weak neighbors, still down-weights them enough to beat the average ($0.381$ against $0.543$).

This split follows directly from the error decomposition of Proposition~\ref{prop:misspecification}, in which the posterior error is $(1-a)e_0 + \mathbf{w}^\top\mathbf{d}$ for evidential prior error $e_0$ and neighbor--query property mismatch $\mathbf{d}$. A poor prior with large $e_0$ and reliable neighbors with small $\mathbf{d}$ makes the GP's reliance on $\mu_0$ a liability, so averaging wins, whereas an accurate prior reverses the balance. What averaging cannot supply is calibration, since it returns a bare point estimate. The GP in \evikalmpg{} also returns a posterior variance that is well calibrated across most of the datasets we test (Appendix~\ref{app:pg_picp}), quantifying how far the ground truth is expected to lie from the prediction. In the decision-making and active-assay settings this work targets, that interval is itself a deliverable, and no averaging scheme provides one.

The RMSE gap on the poor-prior datasets is not intrinsic to the Gaussian process but stems from two specific modeling choices, each with a standard remedy \citep{rasmussen2006gaussian,matheron1963principles,stein1999interpolation}. The first is that \evikalmpg{} holds the prior mean fixed at the evidential prediction $\mu_0$ (Figure~\ref{fig:method_overview}), so a biased $\mu_0$ pulls every posterior toward the wrong value. The standard alternative keeps $\mu_0$ as the anchor but lets the neighbors estimate a correction to it, replacing the prior mean with $\mu_0 + c$ for an offset $c$ drawn from a zero-centered Gaussian of variance $\tau_\beta$ \citep{matheron1963principles,cressie1993statistics}. That variance is a dial, since at $\tau_\beta = 0$ the offset is set to zero and therefore the method reduces exactly to \evikalmpg{}. As $\tau_\beta$ grows the neighbors are increasingly free to move the mean until, in the limit, the fusion weights are forced to sum to one. Tuning $\tau_\beta$ on validation therefore shrinks the $(1-a)e_0$ prior error term by as much as the neighbor evidence warrants without discarding the evidential mean where it is accurate.

The second choice concerns the kernel diagonal, where adding a small observation-noise variance is the standard way to discount correlated evidence \citep{cressie1993statistics,rasmussen2006gaussian}. This noise acts as a second dial. As it grows, the GP progressively discounts interneighbor correlation until, at the high noise extreme, the posterior collapses to a plain average of the neighbor labels. Inflating the prior variance $P_0$ instead, for instance through the evidential regularizer $\lambda$, does not have this effect because it scales the prior variance while leaving the diagonal noise fixed.

Tuning the offset prior $\tau_\beta$ and the noise $\tau_n$ jointly on validation produces the biased-mean GP of Table~\ref{tab:averaging}, and its behavior is systematic. On all eleven datasets where uniform averaging had beaten \evikalmpg{}, the biased-mean GP matches or beats the average, and it strictly improves on both the average and the fixed-mean GP on seven of them (e.g., CDK2 ($0.210$ against $0.225$ and $0.365$), hERG ($0.300$ against $0.308$ and $0.462$), and Thermosol ($0.149$ against $0.162$ and $0.440$)). On the remaining four datasets, it recovers the average to within $0.01$ in $z$-RMSE, driven there by the large validation-selected dials that collapse the GP toward a plain average.

Crucially, the estimator remains a Gaussian process posterior, and its $90\%$ interval attains a median PICP@90 of $0.91$ across those eleven datasets, spanning $0.81$ to $1.00$ and sitting essentially at the nominal target. This supplies the calibrated uncertainty interval of the molecular property prediction that averaging cannot. Where the prior is already accurate, the same tuning drives $\tau_\beta$ toward zero and collapses the biased mean back onto the fixed-mean GP rather than discarding it. Whether to let the neighbors bias the mean or to trust it, and how much noise to add, are thus decided on validation according to the reliability of the evidential prior.

The broader picture is that the fixed-mean GP, the biased-mean GP with added noise, and uniform averaging occupy a single spectrum of correlation discounting. Sequential \evikal{} sits at the over-counting extreme, treating correlated neighbors as independent evidence and failing catastrophically on FreeSolv ($+48.7\%$, Appendix~\ref{app:robustness}), and uniform averaging applies that same over-counting to the property-guided pool of \evikalmpg{}. \evikalmpg{} sits at the correlation-aware end through its Tanimoto kernel, which is what makes it robust where over-counting is harmful. The added noise term exposes the dial between the two regimes and, on the poor-prior datasets, discounts neighbor correlation part of the way without removing it entirely.


\subsection{Reliability Diagrams}
\label{app:reliability}

Figure~\ref{fig:reliability} shows calibration curves (predicted confidence vs.\ empirical
coverage) for all methods, complementing the PICP@90\% values reported in Table~\ref{tab:main}.

\begin{figure*}[h]
\centering
\includegraphics[width=\linewidth]{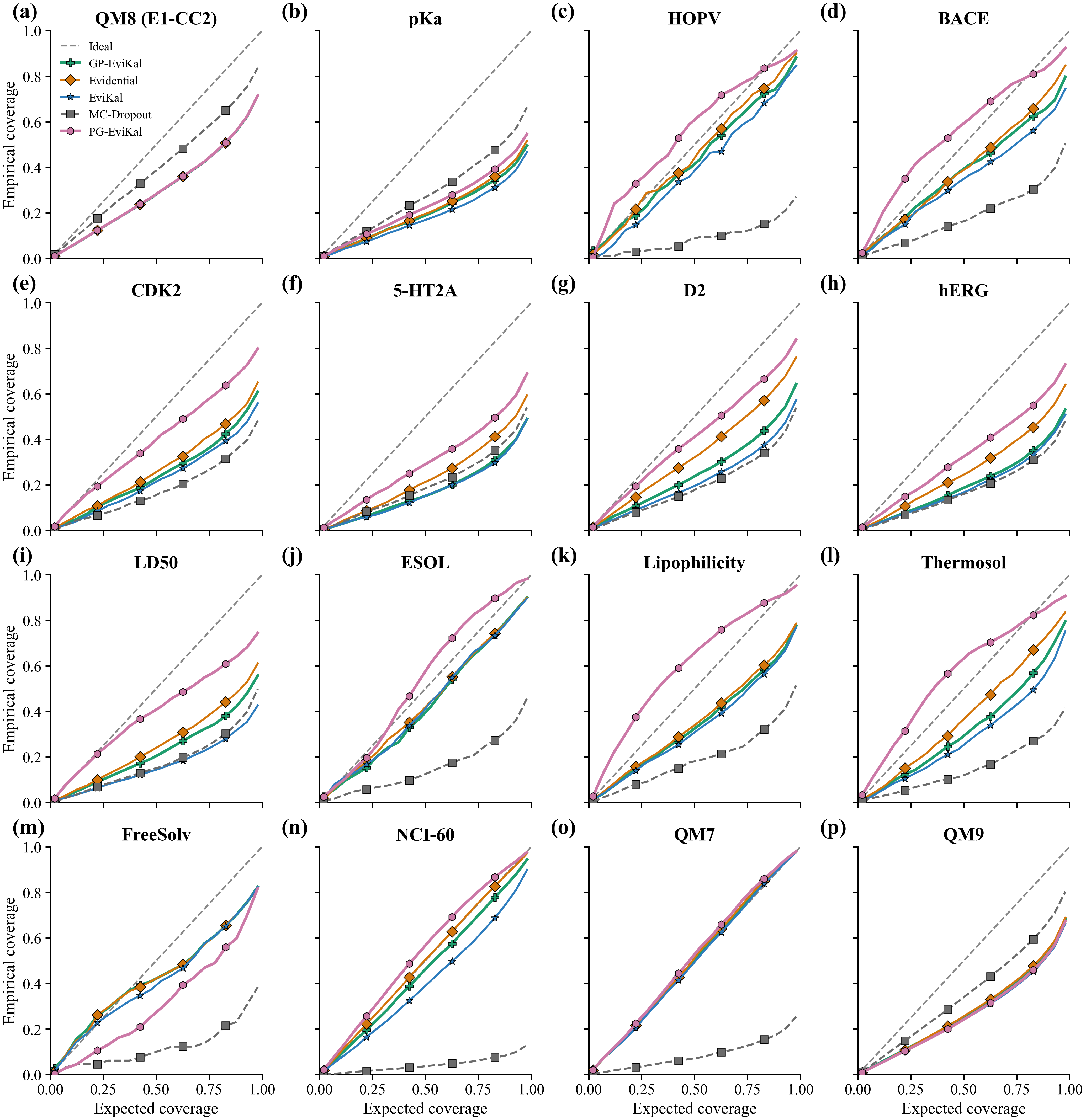}
\caption{
\textbf{Reliability diagrams} (averaged across all 5 training seeds). Diagonal is the $y=x$ target (i.e., perfect calibration). Above the diagonal represents underconfidence; below the diagonal represents
overconfidence.
\mcd{} (gray dashed) lies far below the diagonal on most datasets, indicating severe overconfidence.
\evid{} (orange) tracks the diagonal far more closely than \mcd{} but remains overconfident on several drug-like datasets (e.g., pKa, CDK2, 5-HT2A, hERG, LD50). \evikal{} (blue, sequential) and \evikalm{} (green, GP posterior) compress the predictive interval ($\sigma^2_{\mathrm{post}} < P_0$) and therefore sit at or slightly below the evidential curve. \evikalmpg{} (purple) improves calibration over the evidential baseline on 14 of 16 datasets. It reaches or exceeds the diagonal on ESOL, Lipophilicity, NCI-60, and QM7, improves substantially while remaining below the diagonal on HOPV, BACE, CDK2, 5-HT2A, D2, hERG, LD50, and Thermosol, and yields smaller gains on pKa and QM8. QM9 and FreeSolv are the two exceptions where \evikalmpg{} does not improve calibration.
}
\label{fig:reliability}
\end{figure*}

\section{Proofs and Derivations}
\label{app:proofs}

\subsection{K=1 Equivalence: Scalar Kalman Update = GP Posterior}
\label{app:k1_equiv}

\noindent We show that for a single neighbor ($K=1$), the sequential Kalman update in
Eq.~\eqref{eq:gain}--\eqref{eq:var} and the GP posterior in
Eq.~\eqref{eq:gp_mean}--\eqref{eq:gp_var} produce identical predictions in the
limit of perfect structural similarity ($\mathrm{sim}(q, x_1) \to 1$).

\noindent \textbf{Setup.} Let $\mu_0 = \gamma_q$, $P_0 = \ue^q$ be the evidential prior,
and let $(y_1, R_1)$ denote the single neighbor's label and observation noise.
The Tanimoto covariance between the query and itself is
$k(q, q) = \sigma_f^2$ (the GP prior variance), and between the query and the
single training neighbor is $k(q, x_1) = \sigma_f^2 \cdot \mathrm{sim}(q, x_1)$.
We identify $\sigma_f^2 = P_0 = \ue^q$ and set $k(x_1, x_1) = P_0$ (the neighbor's
prior variance, equal to the query's under the Tanimoto kernel).

\noindent \textbf{GP posterior ($K=1$).}
With a single training point, the GP posterior in Eq.~\eqref{eq:gp_mean}--\eqref{eq:gp_var}
reduces to:
\begin{align}
\mathbf{K}_{\mathrm{obs}} &= k(x_1, x_1) + R_1 = P_0 + R_1, \\
\mathbf{k}_* &= k(q, x_1) = P_0 \cdot \mathrm{sim}(q, x_1). \notag
\end{align}
Substituting into Eq.~\eqref{eq:gp_mean}:
\begin{align}
\mu_{\mathrm{GP}} &= \mu_0 + \frac{P_0 \cdot \mathrm{sim}(q, x_1)}{P_0 + R_1}(y_1 - \mu_0).
\end{align}
And Eq.~\eqref{eq:gp_var}:
\begin{align}
\sigma^2_{\mathrm{GP}} &= P_0 - \frac{(P_0 \cdot \mathrm{sim}(q,x_1))^2}{P_0 + R_1}.
\end{align}

\noindent \textbf{Scalar Kalman update ($K=1$).}
The Kalman gain from Eq.~\eqref{eq:gain} is:
\begin{align}
K_1 = \frac{P_0}{P_0 + R_1}.
\end{align}
The posterior mean from Eq.~\eqref{eq:mean}:
\begin{align}
\mu_1 = \mu_0 + K_1 (y_1 - \mu_0) = \mu_0 + \frac{P_0}{P_0 + R_1}(y_1 - \mu_0).
\end{align}
The posterior variance from Eq.~\eqref{eq:var}:
\begin{align}
P_1 = (1 - K_1) P_0 = \frac{R_1}{P_0 + R_1} P_0 = \frac{P_0 R_1}{P_0 + R_1}.
\end{align}

\noindent \textbf{Comparison.} The Kalman posterior mean $\mu_1$ equals the GP posterior mean
$\mu_{\mathrm{GP}}$ when $\mathrm{sim}(q, x_1) = 1$ (perfect structural similarity).
For $\mathrm{sim}(q, x_1) < 1$, the GP posterior automatically down-weights the
neighbor's influence through the off-diagonal covariance term, whereas the scalar
Kalman update relies on the observation noise $R_k \propto (1 - \mathrm{sim})^2$
to achieve the same effect. In the limit $\mathrm{sim}(q, x_1) \to 1$ (the neighbor is structurally identical to
the query), both posteriors agree exactly:
\begin{align}
\mu_{\mathrm{GP}} = \mu_1 = \mu_0 + \frac{P_0}{P_0 + R_1}(y_1 - \mu_0), \qquad
\sigma^2_{\mathrm{GP}} = P_1 = \frac{P_0 R_1}{P_0 + R_1}.
\end{align}

\noindent This shows that \evikalm{} is a strict generalization of sequential \evikal{}. For $K=1$
and $\mathrm{sim}=1$, the two are identical; for $K > 1$, the GP posterior
additionally accounts for correlations among neighbors through the full $K \times K$
covariance matrix $\mathbf{K}_{\mathrm{obs}}$, which the sequential filter ignores.

\subsection{Noise Model Derivation}
\label{app:noise_derivation}

\noindent The observation noise $R_k$ in Eq.~\eqref{eq:Rk} is defined as:
\begin{align}
R_k = \ua^q + C \cdot (1 - \mathrm{sim}(q, x_k))^2,
\end{align}
where $\ua^q = \beta_q / (\alpha_q - 1)$ is the aleatoric uncertainty of the query. The rationale is a two-component additive model for the observation error
$\epsilon_k = y_k - y_q$ (the difference between the neighbor's true label and the query's):
\begin{enumerate}
  \item \textbf{Aleatoric component} ($\ua^q$): Even a perfect structural match would still
  have measurement noise $\ua^q$ due to assay variability. This sets the floor.
  \item \textbf{Structural dissimilarity component} ($C \cdot (1 - \mathrm{sim})^2$):
  Neighbors that differ structurally may differ in property, with the squared dissimilarity
  penalizing greater structural deviation. The constant $C$ is tuned on validation and
  controls the tradeoff between exploiting similar-but-not-identical neighbors and
  rejecting dissimilar ones.
\end{enumerate}

\noindent The quadratic form $(1 - \mathrm{sim})^2$ ensures that $R_k \to \ua^q$ as
$\mathrm{sim} \to 1$ (a perfect neighbor contributes no structural noise)
and $R_k \to \ua^q + C$ as $\mathrm{sim} \to 0$ (a completely dissimilar neighbor
contributes maximum structural noise $C$).

\section{QSAR Landscape Smoothness: Empirical Evidence}
\label{app:smoothness}

\noindent \evikalm{}'s observation model assumes that a structurally similar training molecule
provides a noisy but informative signal about the query's property.
The QSAR smoothness ratio $s$ (Eq.~\ref{eq:smoothness}) distills this assumption into a single
number: the median normalized property gap at the nearest training neighbor.
But $s$ alone does not tell the whole story---equally important is how similar
that nearest neighbor actually is.
A dataset where the nearest neighbor is only 20\% structurally similar puts
\evikalm{} in a fundamentally different operating regime than one where neighbors
are 75\% similar, even if the raw $s$ value looks comparable.

Figure~\ref{fig:smoothness_analysis} (main text) shows both dimensions together.
For every test molecule we plot the Tanimoto similarity to its nearest training neighbor
(x-axis) against the corresponding normalized property gap (y-axis), averaged across
all training--test pairs in each similarity bin.
The triangle~($\blacktriangle$) marks the typical operating point---
the median top-1 Tanimoto similarity at which \evikalm{} retrieves its neighbors.

\paragraph{Smooth, well-covered datasets (left panel).}
Ten datasets have $s \leq 0.404$ and form the smooth panel:
QM8, QM9, pKa, HOPV, BACE, CDK2, 5-HT2A, D2, hERG, and LD50
(ordered by $s$ from 0.136 to 0.404).
Nearest neighbors are typically 57--81\% Tanimoto similar ($\blacktriangle$ clustered at right).
At those similarity levels the conditional median property gap is only
0.05--0.40$\,\sigma_{\mathrm{train}}$---well below the dashed $s=1$ reference line.
The curves show a clear monotone decline: as two molecules share more fingerprint bits,
their properties converge.
Nine of these ten datasets show RMSE improvement with \evikalm{};
QM8 ($+0.5\%$) is the sole exception.
QM9 ($-0.6\%$) technically improves but is noise-level, sharing QM8's accuracy-ceiling failure mode.

QM8 ($s=0.136$) is an accuracy-ceiling exception.
QM8's first singlet excitation energy is dominated by local conjugation and chromophore
structure that AttentiveFP captures well from the molecular graph; the neural model achieves
$R^2 \approx 0.977$, and structurally similar training molecules have nearly identical labels.
Kalman corrections provide no additional signal beyond what the evidential prior already captured.
This is the accuracy ceiling failure mode---distinct from landscape roughness.
The QSAR smoothness ratio $s$ is necessary but not sufficient for \evikalm{} benefit:
the base model must also have room to improve ($R^2 \lesssim 0.96$).

\paragraph{Transitional, sparse, and rough datasets (right panel).}
ESOL, Lipophilicity, Thermosol, NCI-60, FreeSolv, and QM7 form the right panel,
spanning $s = 0.54$ to $s = 0.83$.
ESOL ($s=0.540$) and Lipo ($s=0.541$) are technically below the $s=0.65$ benefit threshold
but operate at low median top-1 similarity (0.37 and 0.53 respectively);
their gains are modest but present ($-2.5\%$ and $-0.5\%$).
Thermosol and NCI-60 sit above the $s=0.65$ threshold with dense structural coverage
($s = 0.712$ and $0.779$; median top-1 similarities $\geq 0.61$);
\evikalm{} still yields small positive gains ($-0.7\%$ and $-2.9\%$), confirming that
structural coverage acts as a secondary gate.
FreeSolv is the most extreme case of structural sparsity: test molecules find training
neighbors at only 20\% Tanimoto overlap on average.
Despite $s = 0.756$ (lower than NCI-60 at $s=0.779$), FreeSolv degrades by $+4.9\%$---
demonstrating that structural sparsity, not roughness alone, drives the failure.
QM7 ($s=0.834$, median top-1 similarity 0.55) is the roughest dataset and yields $+0.05\%$ RMSE,
consistent with the borderline rough regime.

\paragraph{Takeaway.}
The QSAR smoothness ratio $s$ distinguishes two primary operating regimes:
$s < 0.65$ (reliable benefit, absent base model saturation) and $s \geq 0.65$
(borderline to no benefit; outcome depends on structural coverage).
Within the smooth regime, QM8 ($s=0.136$) and QM9 ($s=0.136$) are the two exceptions
(statistical accuracy ceiling): both evidential priors reach $R^2 \approx 0.977$
with small mean epistemic variance ($\bar{u}_e \approx 0.019\,\sigma_{\mathrm{train}}^2$ for QM8 and $\approx 0.006\,\sigma_{\mathrm{train}}^2$ for QM9), leaving no room
for Kalman refinement despite smooth landscapes.
FreeSolv ($s=0.756$, top-1 sim $=0.20$) fails via structural sparsity;
QM7 ($s=0.834$, top-1 sim $=0.55$) fails via rough landscape.
Thermosol ($s=0.712$, top-1 sim $=0.71$) and NCI-60 ($s=0.779$, top-1 sim $=0.61$)
sit at the boundary with small gains ($-0.7\%$ and $-2.9\%$).
These mechanistically distinct failure modes are all captured by a single pre-computation
requiring only training fingerprints and labels---no model evaluation needed.
The $0.65$ benefit boundary is specific to the evidential AttentiveFP model and should be recalibrated
when deploying a different model family.
\section{SNR$_{\mathrm{eff}}$ Threshold Derivation and Validation}
\label{app:snr_eff_validation}

\subsection{Derivation}

\noindent \textbf{When does a neighbor's label help?}
Consider a neighbor with label $y_k$ and a query with true label $y_q$. The neighbor helps reduce the query's prediction error only if its label carries more signal than noise about the query. Formally, the neighbor's Kalman update $\Delta \mu = K(y_k - \mu_0)$ reduces expected squared error when
\begin{equation}
\mathrm{Cov}(y_k, y_q) > K \cdot \mathrm{Var}(y_k - y_q).
\label{eq:neighbor_benefit}
\end{equation}
Intuitively, the agreement between neighbor and query (covariance) must outweigh their disagreement (variance of difference), weighted by how much we trust the neighbor (Kalman gain $K$).

\noindent \textbf{Translating smoothness to covariance.}
In normalized label space ($\sigma_{\mathrm{train}}=1$), the smoothness ratio $s$ directly relates to neighbor-query correlation.

\noindent Start with the variance of disagreement. For any two random variables:
\begin{equation}
\mathrm{Var}(y_k - y_q) = \mathrm{Var}(y_k) + \mathrm{Var}(y_q) - 2\mathrm{Cov}(y_k, y_q).
\label{eq:var_difference}
\end{equation}

\noindent In normalized label space, both neighbor and query labels are standardized to unit variance: $\mathrm{Var}(y_k) = 1$ and $\mathrm{Var}(y_q) = 1$. Substituting into Eq.~\eqref{eq:var_difference}:
\begin{equation}
\mathrm{Var}(y_k - y_q) = 2 - 2\mathrm{Cov}(y_k, y_q).
\label{eq:var_normalized}
\end{equation}

\noindent By definition, the smoothness ratio $s$ measures the typical magnitude of property disagreement between a query and its nearest neighbor:
\begin{equation}
s = \mathrm{median}_q |y_q - y_{q,1}| \quad (\text{in units of } \sigma_{\mathrm{train}} = 1).
\label{eq:smoothness_def}
\end{equation}

\noindent If the median magnitude of disagreement is $s$ by definition, then squaring gives the median squared disagreement:
\begin{equation}
\left[\mathrm{median}|y_k - y_q|\right]^2 = s^2 = \mathrm{median}[(y_k - y_q)^2]
\label{eq:median_squared}
\end{equation}

\noindent Why? The reason why this relationship holds is because $s \geq 0$ by definition. Squaring a list of $s$ values in this range gives a list of values $s^2$ in the same ascending order; thus, squaring the median deviation yields the median of the squared deviations. Since variance averages squared deviations and the typical squared deviation is $s^2$, we have:
\begin{equation}
\mathrm{Var}(y_k - y_q) \approx s^2.
\label{eq:var_approx_s2}
\end{equation}

\noindent Combining Eqs.~\eqref{eq:var_normalized} and \eqref{eq:var_approx_s2}:
\begin{equation}
2 - 2\mathrm{Cov}(y_k, y_q) \approx s^2.
\label{eq:cov_from_var}
\end{equation}

\noindent Solving for covariance:
\begin{equation}
\mathrm{Cov}(y_k, y_q) \approx 1 - \frac{s^2}{2}.
\label{eq:cov_approx}
\end{equation}

\noindent The intuition is clear:
\begin{itemize}
\item Low $s$ (smooth landscape): typical disagreement is small $\Rightarrow$ Var$(y_k - y_q)$ is small $\Rightarrow$ Cov$(y_k, y_q) \to 1$ (high agreement)
\item High $s$ (rough landscape): typical disagreement is large $\Rightarrow$ Var$(y_k - y_q)$ is large $\Rightarrow$ Cov$(y_k, y_q) \to 0$ (low agreement)
\end{itemize}

\noindent \textbf{Signal-to-noise ratio of the neighbor.}
The ratio of agreement to disagreement is the neighbor's SNR. Using the exact covariance from Eq.~\eqref{eq:cov_approx}:
\begin{equation}
\mathrm{SNR} = \frac{\mathrm{Cov}(y_k, y_q)}{\mathrm{Var}(y_k - y_q)} = \frac{1 - s^2/2}{s^2}.
\label{eq:SNR}
\end{equation}

\noindent For small $s$, this simplifies approximately to:
\begin{equation}
\mathrm{SNR} \approx \frac{1-s^2}{s^2}.
\label{eq:SNR_approx_formula}
\end{equation}

\noindent Both forms show that the ratio is large when neighbors are reliable (low $s$) and small when they are noisy (high $s$).

\noindent \textbf{The Kalman downweighting: effective SNR.}
The Kalman filter does not use the raw neighbor signal. Instead, it scales the update by the Kalman gain $K = P_0 / (P_0 + R_k)$, which is small when observation noise $R_k$ is large. Noisy observations are
downweighted, reducing their impact on the posterior. The effective SNR in lieu of this is:
\begin{equation}
\mathrm{SNR}_{\mathrm{eff}} = K \cdot \mathrm{SNR} = \frac{P_0}{P_0 + R_k} \cdot \frac{1 - s^2/2}{s^2}.
\label{eq:SNR_eff}
\end{equation}

\noindent For the neighbor update to reduce error in expectation, we need $\mathrm{SNR}_{\mathrm{eff}} > 1$: the Kalman-weighted signal must exceed noise.

\noindent For a well-trained evidential model with typical aleatoric/epistemic ratio $u_a/u_e \approx 3$, and using a top-1 neighbor with $\mathrm{sim} \approx 1$:
\begin{equation}
K \approx \frac{u_e}{u_a + u_e} \approx 0.25.
\label{eq:K_typical}
\end{equation}

\noindent The benefit threshold occurs when $\mathrm{SNR}_{\mathrm{eff}} = 1$:
\begin{equation}
K \cdot \frac{1 - s^2/2}{s^2} = 1 \quad \Rightarrow \quad s^2 = \frac{2K}{2+K}.
\label{eq:SNR_threshold}
\end{equation}

\noindent For $K \approx 0.25$: $s^* = \sqrt{0.5/2.25} \approx 0.471$.
For $K \approx 0.50$: $s^* = \sqrt{1.0/2.5} \approx 0.632$.
For $K \approx 0.80$: $s^* = \sqrt{1.6/2.8} \approx 0.756$.

\noindent The empirical threshold $s < 0.65$ falls near these bounds, confirming it is driven by SNR$_{\mathrm{eff}}=1$ (Eq.~\eqref{eq:SNR_eff}), not arbitrary cutoff.

\subsection{Empirical Validation}

\noindent \textbf{Understanding $s$ versus $s^*$:}
Each dataset in Table~\ref{tab:snr_eff} has two key smoothness values. The empirically measured smoothness ratio $s$ (from Eq.~\eqref{eq:smoothness_def}) is the actual median property disagreement for
that dataset. The theoretical threshold $s^*$ is the smoothness value where $\mathrm{SNR}_{\mathrm{eff}}=1$, computed from that dataset's Kalman gain $K$. The criterion is straightforward: \textbf{if $s
< s^*$, then $\mathrm{SNR}_{\mathrm{eff}} > 1$ and the method helps; if $s > s^*$, then $\mathrm{SNR}_{\mathrm{eff}} < 1$ and the method hurts.}

Table~\ref{tab:snr_eff} reports both the exact SNR$_{\mathrm{eff}}$ (Eq.~\eqref{eq:SNR}) and its small-$s$ approximation (Eq.~\eqref{eq:SNR_approx_formula}) for all sixteen datasets, alongside the theoretical threshold $s^*$ (where $\mathrm{SNR}_{\mathrm{eff}}=1$) and observed RMSE benefit. The alignment is clear. Datasets with $\mathrm{SNR}_{\mathrm{eff}} > 1$ largely show RMSE improvement; datasets with $\mathrm{SNR}_{\mathrm{eff}} < 1$ largely show degradation or negligible change. The smooth/rough boundary at $s=0.65$ thus represents the empirical SNR$_{\mathrm{eff}}=1$ threshold for the evidential AttentiveFP architecture at $K=5$ neighbors. This threshold will shift for models with different epistemic calibration; Appendix~\ref{app:smoothness} discusses transfer to other architectures.

\begin{table}[!ht]
\centering
\small
\caption{\textbf{SNR$_{\mathrm{eff}}$ validation: exact vs. approximate formula.}
For each dataset, we report smoothness ratio $s$ (Eq.~\eqref{eq:smoothness_def}), effective signal-to-noise ratio using the approximate formula (Eq.~\eqref{eq:SNR_approx_formula}) and exact formula
(Eq.~\eqref{eq:SNR}), theoretical threshold $s^*$ (where $\mathrm{SNR}_{\mathrm{eff}}=1$), observed RMSE benefit, whether the method shows improvement, and the failure regime. Datasets are ordered by exact
SNR$_{\mathrm{eff}}$ (highest to lowest). $\Delta\mathrm{RMSE}$ refers to \evikalm{} compared to the evidential model baseline.
The SNR$_{\mathrm{eff}} = 1$ boundary (between HOPV and LD50) broadly separates datasets where neighbor fusion helps from those where it hurts or stalls, with the borderline crossings explained by the operating point (below).}
\label{tab:snr_eff}
\begin{tabular}{lrrrrrrl}
\toprule
\textbf{Dataset} & $s$ & $\mathrm{SNR}_{\mathrm{eff,approx}}$ & $\mathrm{SNR}_{\mathrm{eff,exact}}$ & $s^*$ & $\Delta\mathrm{RMSE}$ & \textbf{Benefit?} & \textbf{Regime} \\
\midrule
D2              & 0.367 & 4.859  & 5.225  & 0.740 & $-7.6\%$ & Yes & Smooth \\
CDK2            & 0.335 & 4.897  & 5.208  & 0.688 & $-7.9\%$ & Yes & Smooth \\
5-HT2A          & 0.346 & 4.463  & 4.766  & 0.682 & $-5.2\%$ & Yes & Smooth \\
hERG            & 0.379 & 3.036  & 3.290  & 0.637 & $-5.5\%$ & Yes & Smooth \\
QM9             & 0.136 & 2.980  & 3.010  & 0.234 & $-0.6\%$ & No  & Accuracy ceiling \\
pKa             & 0.165 & 2.525  & 2.559  & 0.261 & $-3.5\%$ & Yes & Smooth \\
BACE            & 0.298 & 1.634  & 1.711  & 0.384 & $-6.9\%$ & Yes & Smooth \\
ESOL            & 0.540 & 0.929  & 1.120  & 0.567 & $-2.5\%$ & Yes & Smooth (near threshold) \\
HOPV            & 0.243 & 1.070  & 1.105  & 0.255 & $-4.0\%$ & Yes & Smooth \\
\midrule
LD50            & 0.404 & 0.750  & 0.823  & 0.370 & $-5.1\%$ & Yes & Smooth \\
Thermosol       & 0.712 & 0.327  & 0.494  & 0.537 & $-0.7\%$ & Yes & Borderline \\
Lipo            & 0.541 & 0.305  & 0.368  & 0.344 & $-0.5\%$ & No  & Smooth (past threshold) \\
NCI-60          & 0.779 & 0.150  & 0.266  & 0.455 & $-2.9\%$ & Yes & Borderline \\
QM8             & 0.136 & 0.020  & 0.020  & 0.019 & $+0.5\%$ & No  & Accuracy ceiling \\
QM7             & 0.834 & 0.007  & 0.015  & 0.123 & $+0.0\%$ & No  & Rough + saturation \\
FreeSolv       & 0.756 & 0.002  & 0.003  & 0.054 & $+4.9\%$ & No  & Structural sparsity \\
\bottomrule
\end{tabular}
\end{table}
The exact formula (Eq.~\eqref{eq:SNR}) is consistently higher than the approximation (Eq.~\eqref{eq:SNR_approx_formula}) by roughly a constant offset, with the difference growing slightly for larger $s$. Both orderings agree, confirming that the qualitative SNR$_{\mathrm{eff}} > 1$ criterion is robust to the choice of formula. The precise formula is provided for those interested in the theory perspective; the approximation suffices as an intuitive ratio (i.e., the signal-to-noise ratio being the complement of a noise-related quantity divided by said quantity).

\noindent \textbf{Why borderline datasets still benefit: the operating point.}
Two borderline datasets, Thermosol and NCI-60, show RMSE improvements despite SNR$_{\mathrm{eff}} < 1$. The operating point triangle (the actual median top-1 Tanimoto similarity where each dataset operates) explains why. The SNR$_{\mathrm{eff}}$ criterion is a landscape-level diagnostic: it characterizes the average roughness of the QSAR surface. But what matters for neighbor fusion is the local operating point—where on that landscape the query's nearest neighbors actually sit (Figure~\ref{fig:smoothness_analysis}). NCI-60 is rough overall ($s=0.779$), yet its nearest neighbors sit at high similarity ($\text{sim}=0.61$) in a region where properties are relatively consistent. So despite the landscape being rough on average, the dataset's operating point lands in a zone where neighbors carry signal, and the method benefits (-2.9\%). FreeSolv, by contrast, also has high $s$, but its operating point is far worse. The nearest neighbors sit at low similarity ($\text{sim}=0.20$), where noise drowns out any signal, and the method fails (+4.9\%). The smaller magnitude of these borderline improvements (both under $3\%$, versus reductions up to $8\%$ on the smoothest landscapes) reflects this. When the operating point is only weakly favorable, neighbors help but inconsistently. The SNR$_{\mathrm{eff}} > 1$ criterion remains a useful heuristic for smooth landscapes, but operating point triangles reveal the finer truth. The method's success depends on where the dataset actually operates, not just the landscape's average roughness.
\newpage
\section{Bayesian Foundation and Error Decomposition}
\label{app:theorems}
\paragraph{A local Bayesian interpretation of \evikalm{}.}
\evikalm{} can be interpreted as exact Bayesian conditioning under a local Gaussian observation model.
For a query molecule $q$, let the evidential model provide an initial prediction $\mu_0=\gamma_q$ and epistemic variance $P_0=u_e^q$.
Let $\mathcal N_q=\{x_1,\ldots,x_K\}$ denote the selected training neighbors with labels $\mathbf y=(y_1,\ldots,y_K)^\top$.
Define the neighbor covariance matrix $\mathbf K$ with $ K_{ij}=P_0\cdot\mathrm{sim}(x_i, x_j)$.
We assume that, locally around $q$, the query property and neighbor labels follow a Gaussian process prior with mean $\mu_0$ and
covariance
\begin{equation}
  \operatorname{Var}(y_q)=P_0,\qquad
  \operatorname{Cov}(y_q,y_k)=k_{\ast,k},\qquad
  \operatorname{Cov}(y_i,y_j)=K_{ij}.
  \label{eq:gp_prior_cov}
\end{equation}
The observed neighbor labels are treated as noisy measurements, with heteroscedastic noise
\begin{equation}
  R_k=u_a^q+C\bigl(1-\operatorname{sim}(q,x_k)\bigr)^2,
  \label{eq:obs_noise}
\end{equation}
and
\begin{equation}
  \mathbf K_{\mathrm{obs}}=\mathbf K+\operatorname{diag}(R_1,\ldots,R_K).
  \label{eq:kobs}
\end{equation}

\begin{theorem}[Conditional risk reduction under a calibrated local GP model]  \label{theorem:risk_reduction}
Assume that the local Gaussian observation model in Eqs.~(\ref{eq:gp_prior_cov})--(\ref{eq:kobs}) is correctly
specified, with $P_0>0$ and $R_k>0$ for all retained neighbors, so that
in particular $\mu_0=\mathbb E[y_q]$ is the true prior mean.
Then the posterior mean
\begin{equation}
  \mu_{\mathrm{post}}
  =
  \mu_0+
  \mathbf k_\ast^\top \mathbf K_{\mathrm{obs}}^{-1}
  (\mathbf y-\mu_0\mathbf 1)
  \label{eq:post_mean}
\end{equation}
is the Bayes estimator of $y_q$ under squared loss. Its posterior
variance is
\begin{equation}
   \sigma^2_{\mathrm{post}}
  =
  P_0-\mathbf k_\ast^\top \mathbf K_{\mathrm{obs}}^{-1}\mathbf k_\ast,
  \label{eq:post_var}
\end{equation}
and satisfies
\begin{equation}
  0\leq \sigma^2_{\mathrm{post}}\leq P_0. \label{eq:sigma_post_range}
\end{equation}
Consequently,
\begin{equation}
  \mathbb E\!\left[(\mu_{\mathrm{post}}-y_q)^2\right]
  \leq
  \mathbb E\!\left[(\mu_0-y_q)^2\right],
  \label{eq:risk_reduction}
\end{equation}
with strict inequality whenever the neighbors carry nonzero
conditional information about the query property, i.e.\ whenever
$\mathbf k_\ast^\top \mathbf K_{\mathrm{obs}}^{-1}\mathbf k_\ast>0$
(equivalently, $\mathbf k_\ast\neq \mathbf 0$).
\end{theorem}

\begin{proof}
Under the assumed local Gaussian model, the joint distribution of the query property and the noisy neighbor labels is
\begin{equation}
  \begin{bmatrix}
  y_q\\
  \mathbf y
  \end{bmatrix}
  \sim
  \mathcal N
  \left(
  \begin{bmatrix}
  \mu_0\\
  \mu_0\mathbf 1
  \end{bmatrix},
  \mathbf \Sigma
  \right), \quad \mathbf \Sigma=\begin{bmatrix}
  P_0 & \mathbf k_\ast^\top\\
  \mathbf k_\ast & \mathbf K_{\mathrm{obs}}
  \end{bmatrix}. \label{eq:joint_distribution}
\end{equation}
By the conditioning formula for jointly Gaussian random variables,
\begin{equation}
  y_q\mid \mathbf y
  \sim
  \mathcal N
  \left(
  \mu_0+\mathbf k_\ast^\top \mathbf K_{\mathrm{obs}}^{-1}(\mathbf y-\mu_0\mathbf 1),
  \,
  P_0-\mathbf k_\ast^\top \mathbf K_{\mathrm{obs}}^{-1}\mathbf k_\ast
  \right),
  \label{eq:conditioning}
\end{equation}
and it follows that
\begin{align}
  \mathbb E\!\left[y_q\mid \mathbf y\right] &= \mu_0+\mathbf k_\ast^\top \mathbf K_{\mathrm{obs}}^{-1}(\mathbf y-\mu_0\mathbf 1),
\label{eq:cond_mean}\\
  \mathrm{Var}\!\left[y_q\mid \mathbf y\right] &=  P_0-\mathbf k_\ast^\top \mathbf K_{\mathrm{obs}}^{-1}\mathbf k_\ast.
\label{eq:cond_var}
\end{align}

\noindent For a fixed query $q$ and fixed neighbor set, under the correctly specified local GP model,
$\mu_{\mathrm{post}}=\mathbb E[y_q\mid \mathbf y]$ and $\sigma^2_{\mathrm{post}} =\mathrm{Var}[y_q\mid \mathbf y]$.

\noindent By definition, $\mu_{\mathrm{post}}$ is the conditional expectation $\mathbb E[y_q\mid \mathbf y]$, which is the unique (up
to a.s. equality) Bayes estimator under squared loss: it minimizes $\mathbb E[(y_q-\hat y)^2\mid \mathbf y]$, and hence the
unconditional risk $\mathbb E[(y_q-\hat y)^2]$, among all estimators $\hat y$ measurable with respect to the observed neighbor labels.

\noindent The posterior variance $ \sigma^2_{\mathrm{post}}$ is the Schur complement of $\mathbf K_{\mathrm{obs}}$ in the joint
covariance matrix $\Sigma$ (Eq.~\ref{eq:joint_distribution}).
Since (i) $\mathbf \Sigma \succeq 0$, (ii) $\mathbf K\succeq 0$ and $R_k>0$ imply that $\mathbf K_{\mathrm{obs}}\succ 0$,
then the Schur complement theorem says:
\begin{equation}
  \sigma^2_{\mathrm{post}}\geq 0.
  \label{eq:schur_nonneg}
\end{equation}

\noindent Moreover, $\mathbf K_{\mathrm{obs}}^{-1}\succ 0$ gives $\mathbf k_\ast^\top \mathbf K_{\mathrm{obs}}^{-1}\mathbf k_\ast\geq
0$. It follows that:
\begin{equation}
  \sigma^2_{\mathrm{post}}
  \;=\;
  P_0-\mathbf k_\ast^\top \mathbf K_{\mathrm{obs}}^{-1} \mathbf k_\ast
  \;\leq\; P_0.
  \label{eq:schur_upper}
\end{equation}

\noindent For a fixed query $q$ and fixed neighbor set, under the correctly specified local GP model, the law of total expectation and Eq.~\ref{eq:sigma_post_range} together give:
\begin{equation}
  \mathbb E\!\left[(\mu_{\mathrm{post}}-y_q)^2\right] \;  = \; \mathbb{E}_{\mathbf y}\!\left[\mathbb
E[(\mu_{\mathrm{post}}-y_q)^2\mid\mathbf y]\right]  = \mathbb{E}_{\mathbf y}\!\left[\sigma_{\mathrm{post}}^2\right] \le P_0.
  \label{eq:total_expectation}
\end{equation}
Under the correctly specified evidential prior $y_q\sim \mathcal{N}(\mu_0, P_0)$,
\begin{equation}
  \mathbb E\!\left[(\mu_{\mathrm{post}}-y_q)^2\right] \; \le \; \mathbb E\!\left[(\mu_0-y_q)^2\right],
  \label{eq:risk_ineq}
\end{equation}
and the risk gap is exactly:
\begin{equation}
  \mathbb E\!\left[(\mu_0-y_q)^2\right]
  -
  \mathbb E\!\left[(\mu_{\mathrm{post}}-y_q)^2\right]
  =
  P_0-\sigma^2_{\mathrm{post}}
  =
  \mathbf k_\ast^\top \mathbf K_{\mathrm{obs}}^{-1}\mathbf k_\ast
  \geq 0,
  \label{eq:risk_gap}
\end{equation}
which is strictly positive precisely when
$\mathbf k_\ast^\top \mathbf K_{\mathrm{obs}}^{-1}\mathbf k_\ast>0$, i.e.,  when the neighbors carry nonzero conditional information
about $y_q$.
\end{proof}

\noindent Theorem~\ref{theorem:risk_reduction} gives the ideal model guarantee. If the local GP observation model is correctly
specified, then \evikalm{} is exact Bayesian conditioning. In that case, the posterior mean~(Eq.~\ref{eq:post_mean}) is the Bayes
estimator under squared loss, and the epistemic variance contracts from $P_0$ to $\sigma^2_{\mathrm{post}}$~(Eq.~\ref{eq:post_var}). The
guarantee is therefore conditional on the neighbor labels being valid noisy observations of the query property under the assumed
covariance and noise model.

\noindent We then introduce Proposition~\ref{prop:misspecification} that explains what controls performance when this local
observation model is only approximate.

\begin{proposition}[Error decomposition under neighbor-model misspecification]\label{prop:misspecification}
Define the GP fusion-weight vector as $\mathbf w=\mathbf K_{\mathrm{obs}}^{-1}\mathbf k_\ast$,
aggregate fusion gain as $a=\mathbf 1^\top \mathbf w$,
the evidential prior error as $e_0=\mu_0-y_q$,
and neighbor-query mismatch vector as $\mathbf d=\mathbf y-y_q\mathbf 1$.
Then \evikalm{} improves mean squared error over the squared evidential prior error,
\begin{equation}
  \mathbb E[(\mu_{\mathrm{post}}-y_q)^2] \le \mathbb E[(\mu_0-y_q)^2],
  \label{eq:prop_improve}
\end{equation}
if and only if
\begin{equation}
  \mathbb E\!\left[ (\mathbf w^\top \mathbf d)^2 + 2(1-a)e_0\mathbf w^\top \mathbf d \right]
  < \mathbb E\!\left[a(2-a)e_0^2\right].
  \label{eq:prop_condition}
\end{equation}
If additionally $\mathbb E[(1-a)e_0\mathbf w^\top \mathbf d]=0$, then this condition reduces to $\mathbb E[(\mathbf w^\top \mathbf
d)^2]< \mathbb E[a(2-a)e_0^2]$.
\end{proposition}

\begin{proof}
By construction,
\begin{equation}
  \mu_{\mathrm{post}}  = \mu_0 + \mathbf w^\top (\mathbf y - \mu_0\mathbf 1).
  \label{eq:mupost_construct}
\end{equation}
Using $\mathbf y=y_q\mathbf 1+\mathbf d$ and $e_0=\mu_0-y_q$, we have
\begin{equation}
  \mathbf y-\mu_0\mathbf 1
  =
  (\mu_0-y_q)\mathbf 1+\mathbf d
  =
  -e_0\mathbf 1+\mathbf d.
  \label{eq:y_shift}
\end{equation}
Therefore,
\begin{equation}
  \mu_{\mathrm{post}}
  =
  \mu_0+\mathbf w^\top(-e_0\mathbf 1+\mathbf d)
  =
  \mu_0-ae_0+\mathbf w^\top \mathbf d.
  \label{eq:mupost_expand}
\end{equation}
Subtracting $y_q$ and using $\mu_0-y_q=e_0$ gives the exact error
decomposition
\begin{equation}
  \mu_{\mathrm{post}}-y_q
  =
  e_0-ae_0+\mathbf w^\top \mathbf d
  =
  (1-a)e_0+\mathbf w^\top \mathbf d.
  \label{eq:error_decomp}
\end{equation}
Thus
\begin{equation}
  \mathbb E\!\left[(\mu_{\mathrm{post}}-y_q)^2\right]
  =
  \mathbb E\!\left[\bigl((1-a)e_0+\mathbf w^\top \mathbf d\bigr)^2\right].
  \label{eq:mse_post}
\end{equation}
Subtracting $\mathbb E[(\mu_0-y_q)^2]=\mathbb E[e_0^2]$ yields
\begin{equation}
  \mathbb E\!\left[(\mu_{\mathrm{post}}-y_q)^2\right]- \mathbb E\!\left[(\mu_0-y_q)^2\right]
  =
  \mathbb E\!\left[
  (\mathbf w^\top \mathbf d)^2
  +
  2(1-a)e_0\,\mathbf w^\top\mathbf d
  -
  a(2-a)e_0^2
  \right].
  \label{eq:mse_diff}
\end{equation}
Therefore $\mathbb E\!\left[(\mu_{\mathrm{post}}-y_q)^2\right] < \mathbb E\!\left[(\mu_0-y_q)^2\right]$ if and
only if
\begin{equation}
  \mathbb E\!\left[
  (\mathbf w^\top \mathbf d)^2
  +
  2(1-a)e_0\,\mathbf w^\top \mathbf d
  \right]
  <
  \mathbb E\!\left[
  a(2-a)e_0^2
  \right].
  \label{eq:prop_condition_proof}
\end{equation}

\noindent Finally, if $\mathbb E[(1-a)e_0\,\mathbf w^\top \mathbf d]=0$, the cross term on the
left-hand side vanishes and the simplified condition
$\mathbb E[(\mathbf w^\top \mathbf d)^2]<\mathbb E[a(2-a)e_0^2]$ follows.
\end{proof}

\noindent Proposition~\ref{prop:misspecification} shows that the posterior error $\mu_{\mathrm{post}}-y_q$ can be understood, via the
exact decomposition~(Eq.~\ref{eq:error_decomp}), as the sum of a shrunk evidential residual $(1-a)e_0$ and a weighted neighbor-query
mismatch $\mathbf w^\top \mathbf d$. Thus, neighbor fusion improves prediction only when the weighted mismatch term is small relative
to the reducible error in the evidential prior. This clarifies the main failure modes: activity cliffs, sparse structural coverage, or
property-irrelevant neighbor selection make $\mathbf d=\mathbf y-y_q\mathbf 1$ large, while an accuracy-ceiling evidential model
makes the reducible prior error small.

\paragraph{Connection to QSAR smoothness and property-supervised similarity.}
Proposition \ref{prop:misspecification} connects directly to the QSAR smoothness diagnostic. For Tanimoto-selected neighbors, the
mismatch vector $\mathbf d=\mathbf y-y_q\mathbf 1$ is small precisely when structural similarity implies property similarity.
Therefore, the QSAR smoothness ratio $s$ is a dataset-level proxy for the typical size of the misspecification cost $(\mathbf
w^\top\mathbf d)^2$ that appears in condition Eq.~\ref{eq:prop_condition}. When $s$ is small, Tanimoto neighbors are likely to be
property-close to the query, making the \evikalm{} observation model a good local approximation. When $s$ is large, structurally
similar neighbors may have large property gaps, and the mismatch term can dominate the benefit of fusion.

Property-supervised similarity acts on the same term, but at the neighbor selection stage. PropDist reranks Tanimoto candidates to
select neighbors with smaller predicted property gaps, thereby reducing $\mathbf d$ before the GP posterior is applied. In this view,
\evikalm{} succeeds when Tanimoto similarity already makes $\mathbf d$ small, while \evikalmpg{} improves the method by learning a
neighbor selection rule that makes $\mathbf d$ small more reliably.

\end{document}